\documentclass{article}

\usepackage{times}

\usepackage[scaled=.90]{helvet}
\usepackage{courier}

\usepackage{graphicx} % more modern
\usepackage{subfigure}

\usepackage{natbib}

\usepackage{enumitem}

\usepackage{algorithm}
\usepackage{algpseudocode}

\usepackage{xspace}

\usepackage[dvipsnames]{xcolor}
\usepackage[colorlinks=true, linkcolor=blue, citecolor=blue]{hyperref}
\usepackage{microtype}

\usepackage{amsthm}
\usepackage{mathtools}
\usepackage{amsmath}
\usepackage{bbm}
\usepackage{amsfonts}
\usepackage{amssymb}

\usepackage{MnSymbol} %Actually conflicts with amssymb and others

\usepackage{xpatch}

\theoremstyle{definition}  %Sets style of subsequent newtheorems to 'definition'

\newtheorem{lemma}{Lemma}

\newtheorem{corollary}{Corollary}

\newtheorem{assumption}{Assumption}

\theoremstyle{plain}
\newtheorem{proposition}{Proposition}
\newtheorem{remark}{Remark}

\newtheorem{theorem}{Theorem}
\newtheorem{definition}{Definition}
\newtheorem*{definition*}{Definition}

\xpatchcmd{\proof}{\itshape}{\normalfont\proofnameformat}{}{}
\newcommand{\proofnameformat}{\bfseries}

\usepackage{prettyref}
\newcommand{\pref}[1]{\prettyref{#1}}
\newcommand{\pfref}[1]{Proof of \prettyref{#1}}
\newcommand{\savehyperref}[2]{\texorpdfstring{\hyperref[#1]{#2}}{#2}}
\newrefformat{eq}{\savehyperref{#1}{\textup{(\ref*{#1})}}}
\newrefformat{eqn}{\savehyperref{#1}{Equation~\ref*{#1}}}
\newrefformat{lem}{\savehyperref{#1}{Lemma~\ref*{#1}}}
\newrefformat{def}{\savehyperref{#1}{Definition~\ref*{#1}}}
\newrefformat{thm}{\savehyperref{#1}{Theorem~\ref*{#1}}}
\newrefformat{corr}{\savehyperref{#1}{Corollary~\ref*{#1}}}
\newrefformat{sec}{\savehyperref{#1}{Section~\ref*{#1}}}
\newrefformat{subsec}{\savehyperref{#1}{Section~\ref*{#1}}}
\newrefformat{app}{\savehyperref{#1}{Appendix~\ref*{#1}}}
\newrefformat{ass}{\savehyperref{#1}{Assumption~\ref*{#1}}}
\newrefformat{ex}{\savehyperref{#1}{Example~\ref*{#1}}}
\newrefformat{tab}{\savehyperref{#1}{Table~\ref*{#1}}}
\newrefformat{fig}{\savehyperref{#1}{Figure~\ref*{#1}}}
\newrefformat{alg}{\savehyperref{#1}{Algorithm~\ref*{#1}}}
\newrefformat{rem}{\savehyperref{#1}{Remark~\ref*{#1}}}
\newrefformat{conj}{\savehyperref{#1}{Conjecture~\ref*{#1}}}
\newrefformat{prop}{\savehyperref{#1}{Proposition~\ref*{#1}}}
\newrefformat{proto}{\savehyperref{#1}{Protocol~\ref*{#1}}}
\newrefformat{prob}{\savehyperref{#1}{Problem~\ref*{#1}}}
\newrefformat{claim}{\savehyperref{#1}{Claim~\ref*{#1}}} 

\DeclarePairedDelimiter{\abs}{\lvert}{\rvert} %
\DeclarePairedDelimiter{\brk}{[}{]}
\DeclarePairedDelimiter{\crl}{\{}{\}}
\DeclarePairedDelimiter{\prn}{(}{)}
\DeclarePairedDelimiter{\nrm}{\|}{\|}
\DeclarePairedDelimiter{\tri}{\langle}{\rangle}

\let\Pr\undefined

\DeclareMathOperator{\En}{\mathbb{E}}
\DeclareMathOperator{\Pr}{Pr}

\DeclareMathOperator*{\argmin}{arg\,min} % * Places subscript directly under operator
\DeclareMathOperator*{\argmax}{arg\,max}

\newcommand{\ls}{\ell}
\newcommand{\ind}{\mathbf{1}}%{\mathbbm{1}}    %Indicator

\newcommand{\eps}{\epsilon}
\newcommand{\veps}{\varepsilon}

\newcommand{\ldef}{\vcentcolon=}

\newcommand{\wt}[1]{\widetilde{#1}}

\def\ddefloop#1{\ifx\ddefloop#1\else\ddef{#1}\expandafter\ddefloop\fi}
\def\ddef#1{\expandafter\def\csname bb#1\endcsname{\ensuremath{\mathbb{#1}}}}
\ddefloop ABCDEFGHIJKLMNOPQRSTUVWXYZ\ddefloop
\def\ddefloop#1{\ifx\ddefloop#1\else\ddef{#1}\expandafter\ddefloop\fi}
\def\ddef#1{\expandafter\def\csname b#1\endcsname{\ensuremath{\mathbf{#1}}}}
\ddefloop ABCDEFGHIJKLMNOPQRSTUVWXYZ\ddefloop
\def\ddef#1{\expandafter\def\csname c#1\endcsname{\ensuremath{\mathcal{#1}}}}
\ddefloop ABCDEFGHIJKLMNOPQRSTUVWXYZ\ddefloop
\def\ddef#1{\expandafter\def\csname h#1\endcsname{\ensuremath{\widehat{#1}}}}
\ddefloop ABCDEFGHIJKLMNOPQRSTUVWXYZ\ddefloop
\def\ddef#1{\expandafter\def\csname hc#1\endcsname{\ensuremath{\widehat{\mathcal{#1}}}}}
\ddefloop ABCDEFGHIJKLMNOPQRSTUVWXYZ\ddefloop
\def\ddef#1{\expandafter\def\csname t#1\endcsname{\ensuremath{\widetilde{#1}}}}
\ddefloop ABCDEFGHIJKLMNOPQRSTUVWXYZ\ddefloop
\def\ddef#1{\expandafter\def\csname tc#1\endcsname{\ensuremath{\widetilde{\mathcal{#1}}}}}
\ddefloop ABCDEFGHIJKLMNOPQRSTUVWXYZ\ddefloop

\newcommand{\Unif}[1]{\text{\upshape\mdseries Unif}\Parens{#1}}
\newcommand{\Low}{\text{\scshape\mdseries Low}\xspace}
\newcommand{\High}{\text{\scshape\mdseries High}\xspace}

\newcommand{\Wid}{W}%\text{\scshape\mdseries Wid}\xspace}

\newcommand{\consL}{L}
\newcommand{\one}{\mathbf{1}}
\newcommand{\reg}{{\text{\rm Reg}}}
\newcommand{\tLow}{z_{\Low}}%\,\,\widetilde{\!\!\text{\scshape\mdseries Low}\!\!}\,\,\xspace}
\newcommand{\tHigh}{z_{\High}}%\,\,\,\widetilde{\!\!\!\text{\scshape\mdseries High}\!\!\!}\,\,\,\xspace}
\newcommand{\R}{\mathbb{R}}
\newcommand{\vv}{\mathbf{v}}

\newcommand{\whi}{w_\text{\mdseries H}}
\newcommand{\wlo}{w_\text{\mdseries L}}

\newcommand{\fhi}{f_\text{\mdseries H}}
\newcommand{\flo}{f_\text{\mdseries L}}
\newcommand{\zhi}{z_\text{\mdseries H}}
\newcommand{\zlo}{z_\text{\mdseries L}}

\newcommand{\Rmin}{R_\text{\mdseries min}}
\newcommand{\fmin}{f_\text{\mdseries min}}
\newcommand{\Alg}[1]{\pref{alg:#1}}
\newcommand{\Eq}[1]{Eq.~\eqref{eq:#1}}

\newcommand{\set}[1]{\{#1\}}

\newcommand{\bigSet}[1]{\bigl\{#1\bigr\}}

\newcommand{\BigBraces}[1]{\Bigl\{#1\Bigr\}}

\newcommand{\bigBracks}[1]{\bigl[#1\bigr]}
\newcommand{\BigBracks}[1]{\Bigl[#1\Bigr]}

\newcommand{\Bracks}[1]{\left[#1\right]}

\newcommand{\Parens}[1]{\left(#1\right)}
\newcommand{\bigParens}[1]{\bigl(#1\bigr)}

\newcommand{\given}{\mathbin{\vert}}
\newcommand{\bigGiven}{\mathbin{\bigm\vert}}
\newcommand{\BigGiven}{\mathbin{\Bigm\vert}}

\newcommand{\Abs}[1]{\left\lvert#1\right\rvert}
\newcommand{\bigAbs}[1]{\bigl\lvert#1\bigr\rvert}

\newcommand{\expfour}{\textsf{Exp4}\xspace}
\newcommand{\ucb}{\textsf{UCB}\xspace}
\newcommand{\linucb}{\textsf{LinUCB}\xspace}
\newcommand{\egreedy}{\textsf{$\eps$-Greedy}\xspace}
\newcommand{\minimonster}{\textsf{ILOVETOCONBANDITS}\xspace}
\newcommand{\iltcb}{\textsf{ILTCB}\xspace}
\newcommand{\bootstrap}{\textsf{Bootstrap}\xspace}
\newcommand{\rucb}{\textsf{RegCB}\xspace}

\newcommand{\concG}{{C_\delta}}
\newcommand{\concF}{{C_\delta'}}
\newcommand{\surprise}{surprise bound\xspace}
\newcommand{\Surprise}{Surprise bound\xspace}
\newcommand{\implicit}{implicit exploration coefficient\xspace}
\newcommand{\Implicit}{Implicit exploration coefficient\xspace}
\newcommand{\iec}{IEC\xspace}

\usepackage[accepted]{icml2018_arxiv}

\icmltitlerunning{Practical Contextual Bandits with Regression Oracles}

\begin{document}

\twocolumn[
\icmltitle{Practical Contextual Bandits with Regression Oracles}

\icmlsetsymbol{equal}{*}

\begin{icmlauthorlist}
\icmlauthor{Dylan J. Foster}{co}
\icmlauthor{Alekh Agarwal}{microsoft}
\icmlauthor{Miroslav Dud\'ik}{microsoft}
\icmlauthor{Haipeng Luo}{usc}
\icmlauthor{Robert E. Schapire}{microsoft}
\end{icmlauthorlist}

\icmlaffiliation{co}{Cornell University. Work performed while the author was an intern at Microsoft Research.}
\icmlaffiliation{microsoft}{Microsoft Research}
\icmlaffiliation{usc}{University of Southern California}

\icmlcorrespondingauthor{Dylan J. Foster}{djf244@cornell.edu}

\vskip 0.3in
]

\printAffiliationsAndNotice{\icmlEqualContribution} 

\begin{abstract}

A major challenge in contextual bandits is to design general-purpose algorithms that are both practically useful and theoretically well-founded. We present a new
technique that has the empirical and computational advantages of
realizability-based approaches combined with the flexibility of agnostic
methods. Our algorithms leverage the availability of a regression
oracle for the value-function class, a more realistic and reasonable
oracle than the classification oracles over policies typically
assumed by agnostic methods. Our approach generalizes both \ucb{} and \linucb{} to far more
expressive possible model classes and achieves low regret under certain
distributional assumptions. In an extensive empirical
evaluation, compared to both realizability-based and agnostic baselines, we
find that our approach typically gives comparable or superior results.

\end{abstract}

\section{Introduction}
We study the design of practically useful, theoretically well-founded,
general-purpose algorithms for the contextual bandits problem.  In
this setting, the learner repeatedly receives \emph{context}, then
selects an \emph{action}, resulting in a received \emph{reward}.  The
aim is to learn a \emph{policy}, a rule for choosing actions based on
context, so as to maximize the long-term cumulative reward.  For instance,
a news portal must repeatedly choose articles to present to each user
to maximize clicks.  Here, the context is information about the user,
the actions represent the choice of articles, and the reward indicates
if there was a click.  We refer the reader to a recent ICML 2017 tutorial
(\url{http://hunch.net/~rwil/}) for further examples and motivation.

Approaches to contextual bandit learning can broadly be put into two
groups.  Some methods~\cite{langford2008epoch,agarwal2014taming} are
\emph{agnostic} in the sense that they are provably effective for
\emph{any} given policy class and data distribution.  In contrast,
\emph{realizability-based} approaches such as \linucb{} and
variants~\cite{chu2011contextual,li2017provable,filippi2010parametric}
or Thompson Sampling~\cite{thompson1933likelihood} assume the data is
generated from a particular parametrized family of models.
Computationally tractable realizability-based algorithms are
only known for specific model families, such as when the conditional
reward distributions come from a generalized linear model.

The two groups of approaches seem to have different advantages and disadvantages.
Empirically, in the contextual
semibandit setting, \citet{krishnamurthy2016contextual} found
that the realizability-based \linucb{} approach outperforms all
agnostic baselines using a linear policy class. However, the
agnostic approaches were able to overcome this shortcoming by using a more
powerful policy class.
Computationally, previous realizability-based approaches have been limited
by their reliance on either closed-form confidence bounds (as in
\linucb{} variants), or the ability to efficiently sample from and
frequently update the posterior (as in Thompson sampling).
Agnostic approaches, on the other hand, typically assume an oracle
for cost-sensitive classification, which is in general
computationally intractable, though often practically feasible
for many natural policy classes.

In this paper, we aim to develop techniques that combine what is best
about both of these approaches. To this end, in \pref{sec:algs}, we propose computationally efficient and
practical realizability-based algorithms for arbitrary model classes.  As is
often done in agnostic approaches, we assume the availability of an
oracle which reduces to a standard learning setting and knows how to
efficiently leverage the structure of the model class. Specifically,
we require access to a squared regression oracle over the model class
that we use for predicting rewards, given contexts.  Since regression can
often be solved efficiently, the availability of such an oracle is a rather mild assumption, far more reasonable than the kind
of cost-sensitive classification oracle more commonly assumed, which
typically must solve NP-hard problems. In fact, for this reason, even the classification oracles are typically approximated by regression oracles in practice (see, e.g., \citealp{beygelzimer2009offset}).
Our main
algorithmic components here are motivated by and adapted from a recent
work of \citet{krishnamurthy2017active} on cost-sensitive active
learning.

In \pref{sec:regret}, we prove that our algorithms are
effective in the sense of achieving low regret under certain favorable
distributional assumptions. Specifically, we show that our methods
enjoy low regret so long as certain quantities like the
\emph{disagreement
  coefficient}~\citep{hanneke2014theory,krishnamurthy2017active} are
bounded. We also present a second set of bounds in terms of certain
distributional coefficients which generalize the exploration
parameters introduced by \citet{bastani2015online} from linear to
general function classes. As a special consequence, we obtain nearly
dimension-free results for sparse linear bandits in high dimensions.

Finally, in \pref{sec:experiments}, we conduct a very extensive
empirical evaluation of our algorithms on a number of datasets and
against both realizability-based and agnostic baselines.  In this test of
practical effectiveness, we find that our approach gives comparable or
superior results in nearly all cases, and we also validate the
distributional assumptions required for low-regret guarantees
on these datasets.

\section{Preliminaries}

We consider the following contextual bandit protocol.
Contexts are drawn from an arbitrary space, $x\in\cX$, actions are from a finite set, $a\in\cA\ldef\{1, \ldots, K\}$, for some fixed $K$,
and reward vectors are from a bounded set, $r \in[0,1]^K$, with component $r(a)$ denoting the reward for action $a \in \cA$.

We consider a stochastic setting where there is a fixed and unknown
distribution $D$ over the context-reward pairs $(x,r)$.  Its marginal
distribution over $\cX$ is denoted by $D_{\cX}$.  The learning
protocol proceeds in rounds $t=1,\ldots, T$. In each round $t$, nature
samples $(x_t,r_t)$ according to $D$ and reveals $x_t$ to the
learner. The learner chooses an action $a_t \in \cA$ and observes the
reward $r_t(a_t)$.  The goal of the learner is to maximize the reward
and do well compared with any strategy that models the
expected reward \mbox{$\En[r(a)\given x,a]$} via a function
$f:\cX\times\cA\to[0,1]$. These mappings $f$ are drawn from a given class of predictors $\cF$,
such as the class of linear predictors or regression trees.

The main assumption this paper follows is that the class $\cF$ is rich enough
to contain a predictor that perfectly predicts the expected reward of any action under any context,
that is:
\begin{assumption}[Realizability]
\label{ass:realizable}
There is a predictor $f^{\star}\in\cF$ such that
\[
\En\brk*{r(a)\mid{}x,a} = f^{\star}(x, a)
\quad
\text{for all $x \in \cX$ and $a\in \cA$.}
\]
\end{assumption}

Given a predictor $f\in\cF$, the associated optimal strategy $\pi_f:\cX\to\cA$, called a \emph{policy}, always
picks the action with the highest predicted reward, i.e., $\pi_{f}(x)\ldef\argmax_{a\in\cA}f(x,a)$
(breaking ties arbitrarily).
We use the abbreviation $\pi^{\star}\ldef\pi_{f^{\star}}$ for the underlying optimal policy.
The formal goal of the learner is then to minimize the regret
\begin{equation*}
\reg_T = \sum_{t=1}^{T}r_{t}(\pi^\star(x_t))- \sum_{t=1}^{T}r_{t}(a_t),
\end{equation*}
which compares the accumulated rewards between the optimal policy and the learner's strategy.
The classic {\expfour} algorithm~\cite{auer2002nonstochastic} achieves an optimal regret bound of order
$O(\sqrt{TK\ln|\cF|})$ (for any finite \cF), but the computational complexity is unfortunately linear in $|\cF|$.

\paragraph{Regression Oracle}
To overcome the computational obstacle,
our algorithms reduce the contextual bandit problem to weighted least-squares regression.
Abstracting the computational complexity, we assume access to a \emph{weighted least-squares regression oracle} over the predictor class $\cF$,
which takes any set $H$ of weighted examples $(w, x, a, y) \in {\mathbb{R}_+ \times \cX \times \cA \times [0,1]}$ as input,
and outputs the predictor with the smallest weighted squared loss:
\[
\textsc{Oracle}(H) = \argmin_{f\in\cF}\sum_{(w, x, a, y) \in H} w(f(x, a) - y)^{2}.
\]
As mentioned, such regression tasks are very common in machine learning practice
and the availability of such oracle is thus a very mild assumption.

\section{Algorithms}
\label{sec:algs}

The high-level idea of our algorithms is the following.
As data is collected, we maintain a subset of $\cF$, which we refer to as the \emph{version space},
that only contains predictors with small squared loss on observed data.
When a new example arrives, we construct upper and lower confidence bounds on the reward of each action
based on the predictors in the version space.
Finally, with these confidence bounds, we either optimistically pick the action with the highest upper bound,
similar to {\ucb} and {\linucb}, or randomize among all actions that are potentially the best.

The challenge here is to maintain such version spaces and confidence bounds efficiently,
and we show that this can be done using a simple binary search together with a small number of regression oracle calls.

We now describe our algorithms more formally.
First, we define the upper and lower reward bounds with respect to a a subset $\cF'\subseteq\cF$ as
\begin{align*}
    \High_{\cF'}(x,a)&= \max_{f\in\cF'}f(x,a)
\\[-2pt]
    \Low_{\cF'}(x,a) &=\min_{f\in\cF'}f(x,a)
    .
\end{align*}
Our algorithms will induce the confidence bounds by instantiating these quantities using the version space as $\cF'$.
To reduce computational costs, our algorithms update according to a doubling epoch schedule.
Epoch $m$ will begin at time $\tau_{m} = 2^{m-1}$, and $M = O(\log T)$ is the total number of epochs.
At each epoch $m$ our algorithms (implicitly) construct a version space $\cF_{m}\subseteq{}\cF$,
and then select an action based on the reward ranges defined by $\High_{\cF_m}\!(x,a)$ and $\Low_{\cF_m}\!(x,a)$
for each time $t$ that falls into epoch $m$.
Specifically, we consider two algorithm variants:
the first one uniformly at random picks from actions that are plausible to be the best, that is,
\[
a_t \sim \Unif{\BigBraces{a \BigGiven \High_{\cF_m}(x_t,a) \geq{}\max_{a'\in\cA}\Low_{\cF_m}(x_t,a')}},
\]
where $\Unif{S}$ denotes the uniform distribution over a set~$S$;
the second one simply behaves optimistically and picks the action with the highest upper reward bound, that is,
\[
a_t = \argmax_{a\in\cA} \High_{\cF_m}(x_t,a)
\]
(ties are broken arbitrarily).
For technical reasons,
the optimistic variant also spends the first few epochs doing pure exploration
as a warm start for the algorithm.

To construct these version spaces,
we further introduce the following least-squares notation for any $m \geq 2$:
\begin{itemize}
\item $\hR_m(f) = \frac{1}{\tau_m-1}\sum_{s<\tau_m}\bigParens{f(x_s, a_s) - r_s(a_s)}^{2}$,
\item $\hcF_m(\beta)=\bigSet{f\in\cF \bigGiven \hR_m(f) - \min_{f\in\cF}\hR_m(f) \leq \beta}$,
\end{itemize}
and also let $\hcF_1(\beta) = \cF$ for any $\beta$.
With this notation $\cF_m$ is simply set to $\hcF(\beta_m)$ for some tolerance parameter $\beta_m$.

\paragraph{Product Classes}
Sometimes it is desirable to have a product predictor class, that is,  $\cF=\cG^{\cA}$,
where $\cG:\cX\to\brk*{0,1}$ is a ``base class'' and each $f\in\cF$, described by a $K$-tuple $(g_a)_{a\in\cA}$ where $g_a\in\cG$,
predicts according to $f(x,a)=g_a(x)$.
Similar to the general case, we introduce the following notation for $m \geq 2$:
\begin{itemize}[leftmargin=*]
\item $\hcR_m(g, a) = \frac{1}{\tau_m-1}\sum_{s<\tau_m}(g(x_s) - r_s(a_s))^{2}\ind\crl*{a_s =a}$,
\item $\hcG_m(\beta, a)=\crl*{g\in\cG\mid{}\hcR_m(g, a) - \min_{g\in\cG}\hcR_m(g, a) \leq \beta}$,
\end{itemize}
and let $\hcG_1(\beta, a) = \cG$ for any $\beta$.
In this case we construct $\cF_m$ as $\prod_{a\in\cA}\hcG_{m}(\beta_m, a)$ for some tolerance parameter $\beta_m$.

Our two procedures are described in detail in \pref{alg:regression_ucb_elim} and \pref{alg:regression_ucb_optimistic}.

\begin{algorithm}[t]
\caption{\textsc{RegCB.Elimination}}
\label{alg:regression_ucb_elim}
\begin{algorithmic}[1]
\State{\textbf{Input}}: square-loss tolerance $\beta_m$
\For{epoch $m=1,\ldots,M$}
\State{
$\cF_{m} \gets \begin{cases}
\prod_{a\in\cA}\hcG_{m}(\beta_m, a)  &(\textsc{Option I})\\
\hcF_m(\beta_m)  &(\textsc{Option II})\\
\end{cases}$
}
\For{time $t=\tau_{m},\ldots,\tau_{m+1}-1$}
\State{Receive $x_{t}$.}
\State $A_{t}\gets{} \{a: \High_{\cF_m}(x_t, a) \geq$
\Statex \hspace{10em} $\max_{a'\in\cA}\Low_{\cF_m}(x_t, a')\}$.
\State{Sample $a_t\sim{} \Unif{A_t}$ and receive $r_t(a_t)$.}
\EndFor
\EndFor
\end{algorithmic}
\end{algorithm}

\begin{algorithm}[t]
\caption{\textsc{RegCB.Optimistic}}
\label{alg:regression_ucb_optimistic}
\begin{algorithmic}[1]
\State{\textbf{Input}:~square-loss tolerance $\beta_m$}
\Statex{~\hphantom{\textbf{Input}:}number of warm-start epochs $M_0$}
\For{time $t=1,\ldots,\tau_{M_0}-1$}
\State{Receive $x_{t}$, play $a_t \sim \Unif{\cA}$, and receive $r_t(a_t)$.}
\EndFor
\For{epoch $m=M_0,\ldots,M$}
\State{$\cF_{m} \gets \hcF_m(\beta_m)$.}
\For{time $t=\tau_{m},\ldots,\tau_{m+1}-1$}
\State{Receive $x_{t}$.}
\State{Select $a_t=\argmax_{a\in\cA}\High_{\cF_m}(x_t, a)$.}
\State{Receive $r_t(a_t)$.}
\EndFor
\EndFor
\end{algorithmic}
\end{algorithm}

\subsection{Efficient Reward-Range Computation}
\label{subsec:computation}

Both Algorithms~\ref{alg:regression_ucb_elim} and~\ref{alg:regression_ucb_optimistic} hinge on the computation of the reward bounds $\High_{\cF_m}$ and $\Low_{\cF_m}$.
It turns out that this can be carried out efficiently via a small number of calls to the regression oracle.

Specifically, to calculate the confidence bounds for a given $x$, $a$, we augment the data set $H_m$ with a single example $(x,a,r)$ with a weight $w$, and set
its reward $r$ beyond the reward range. For the upper confidence bound, we use $r=2$; for the lower confidence bound $r=-1$. By increasing the weight $w$,
we force the regression oracle to perform better on this single example. For the upper confidence bound it means to predict higher rewards
as $w$ increases, while getting worse performance on the remaining examples. The binary search over $w$ then identifies, up to a given
precision, the weight~$w$ and the corresponding predicted value at~$x$ and~$a$, at which the performance on the previous examples suffers by exactly the desired tolerance $\beta$.
See \Alg{binary_search_01} for details, including the choice of the initial weight.

In \pref{app:binary_search} we show that this strategy indeed works as intended
and in $O(\log(1/\alpha))$ iterations computes the confidence bounds up to a precision of $\alpha$. The guarantee is formalized in the following theorem:

\begin{theorem}
\label{thm:binary_search}
Let $H_m = \{(x_s, a_s, r_s(a_s))\}_{s=1}^{\tau_m-1}$.
If the function class $\cF$ is convex and closed under pointwise convergence, then the calls
\begin{align*}
\tHigh &\gets{} \textsc{BinSearch}(\textsc{High},(x,a), H_m, \beta, \alpha)\\
\tLow &\gets{} \textsc{BinSearch}(\textsc{Low},(x,a), H_m, \beta, \alpha)
\end{align*}
terminate after $O(\log(1/\alpha))$ oracle invocations and the returned values satisfy
\begin{align*}
\abs*{\High_{\hcF_m(\beta)}(x,a) - \tHigh} &\leq{}\alpha \\
\abs*{\Low_{\hcF_m(\beta)}(x,a) - \tLow}   &\leq{}\alpha.
\end{align*}
\end{theorem}

\begin{algorithm}[t]
\caption{\textsc{BinSearch}}
\label{alg:binary_search_01}
\begin{algorithmic}[1]
\State{\textbf{Input}:~bound type${}\in\set{\textsc{Low},\textsc{High}}$, target pair $(x,a)$}
\Statex{~\hphantom{\textbf{Input}:}history $H$, radius $\beta>0$, precision $\alpha>0$}
\State{Based on bound type: $r=2$ if \textsc{High} and $r=-1$ if \textsc{Low}}
\State{Let $R(f)\coloneqq\sum_{(x', a', r')\in{}H}(f(x',a')-r')^{2}$.}
\State{Let $\wt{R}(f, w)\coloneqq R(f) + \frac{w}{2}(f(x,a)-r)^{2}$}
\State{$\wlo\gets 0,\;\whi\gets\beta/\alpha$}
\smallskip
\Statex{\textbf{\texttt{// Invoke oracle twice}}}
\State{$\flo\gets\argmin_{f\in\cF}\wt{R}(f, \wlo),\;\zlo\gets\flo(x,a)$}
\State{$\fhi\gets\argmin_{f\in\cF}\wt{R}(f, \whi),\;\zhi\gets\fhi(x,a)$}
\State{$\Rmin\gets R(\flo)$}
\State{$\Delta\gets\alpha\beta/(r-\zlo)^3$}
\While{$\abs{\zhi-\zlo} > \alpha$ and $\abs{\whi-\wlo} > \Delta$}
\State{${w}\gets(\whi+\wlo)/2$}
\smallskip
\Statex{\hspace{\algorithmicindent}\textbf{\texttt{// Invoke oracle.}}}
\State{${f}\gets\argmin_{\tilde{f}\in\cF}\wt{R}(\tilde{f}, w),\;{z}\gets{f}(x,a)$}
\If{$R({f})\ge\Rmin+\beta$}
\State{$\whi\gets {w},\;\zhi\gets {z}$}
\Else
\State{$\wlo\gets {w},\;\zlo\gets {z}$}
\EndIf
\EndWhile
\State{\textbf{return} $\zhi$.}
\end{algorithmic}
\end{algorithm}
Compared to the procedure used by \citet{krishnamurthy2017active},
\pref{alg:binary_search_01} is much simpler and achieves an exponential improvement in terms of oracle
calls, namely, $O(\log(1/\alpha))$ as opposed to $O(1/\alpha)$, when $\cF$ is convex.
Compared to oracles used in cost-sensitive classification, convexity is not a strong assumption for regression oracles.
Nonetheless, when $\cF$ is not convex, the reward ranges can be computed with $O(1/\alpha)$ oracle calls using the techniques of \citet{krishnamurthy2017active}.

\section{Regret Guarantees}
\label{sec:regret}
% !TEX root = paper.tex

In this section we provide regret guarantees for \rucb{} (\pref{alg:regression_ucb_elim} and \pref{alg:regression_ucb_optimistic}). Note that \rucb{} is not minimax optimal: while it can obtain $O\bigParens{\sqrt{KT\log\abs{\cF}}}$ regret or even logarithmic regret under certain distributional assumptions, which we describe shortly, for some instances it can make as many as $\abs*{\cF}$ mistakes, which is suboptimal:
\begin{proposition}[Bad instance for confidence-based strategies]
\label{prop:lb_optimistic}
For every $\eps\in(0, 1]$ and $N\in\bbN$ there exists a class of
  reward predictors with $|\cF| = N+1$ and a distribution for which both
  Algorithms~\ref{alg:regression_ucb_elim}
  and~\ref{alg:regression_ucb_optimistic} have regret across $T$
  rounds lower bounded by $(1-\eps)\cdot\min\crl*{N, \wt{\Omega}(T)}$.
\end{proposition}

\pref{prop:lb_optimistic} is proved in \pref{app:lb_optimistic}. The proof
is build on a well known, albeit rather pathological instance.
In contrast, our
strong empirical results in the following section show that such
instances are not encountered in practice. In order to understand the
typical behavior of such algorithms, prior works have considered
structural assumptions such as finite eluder
dimension~\cite{russo2013eluder} or disagreement
coefficients~\cite{hanneke2014theory,krishnamurthy2017active}. In the
next two subsections, we use similar ideas to analyze the
regret incurred by our algorithm. For simplicity, we assume that
$\High_{\cF_m}$ and $\Low_{\cF_m}$ are computed exactly.

\subsection{Disagreement-based Analysis}
\label{sec:disagreement}

Disagreement coefficients come from the active learning
literature~\cite{hanneke2014theory}, and roughly assume that given a
set of functions which fit the historical data well, the probability that
these functions make differing predictions on a new example is
small. This rules out the bad case of
\pref{prop:lb_optimistic}, where a near-optimal predictor
significantly disagrees from the others on each context. Our development in this subsection largely
follows~\citet{krishnamurthy2017active}, with
appropriate modifications to translate from active learning to
contextual bandits. We start by recalling some formal definitions, leading up to the definition of the
disagreement coefficient.

\begin{definition}
For any $\veps > 0$, the \emph{policy-regret ball} of radius $\veps$ for $\cF$ is defined as
\begin{align*}
\cF(\veps)=\BigBraces{f\in\cF:\: \En[r(\pi_f(x))] \geq
 \En[r(\pi^\star(x))] & - \veps
}.
\end{align*}
\end{definition}

\begin{definition}[Reward width] For any predictor class $\cF$, context $x$, and action $a$,
the reward width is defined as
\[
\Wid_{\cF}(x,a) =  \High_{\cF}(x,a) - \Low_{\cF}(x,a).
\]
\end{definition}

\begin{definition}[Disagreement Region]
For any predictor class $\cF$, the disagreement region $\mathrm{Dis}(\cF)$ is defined as\footnote{When the maximizing action $\argmax_{a \in \cA} f(x,a)$ is not unique, the ``$\neq$'' in the disagreement set definition checks that the two argmax sets are identical.}
\begin{align*}
&\mathrm{Dis}(\cF)
\\&= \BigBraces{x\;\BigGiven\;
    \exists{}f,f'\in\cF: \argmax_{a\in \cA}f(x, a)\neq{}\argmax_{a\in \cA}f'(x,a)}.
\end{align*}
\end{definition}

\begin{definition}[Disagreement set]
\label{def:confused_set}
 For a predictor class~$\cF$ and a context~$x$, the \emph{disagreement set} at $x$ is defined as
\[
A_{\cF}(x) = \bigcup_{f\in\cF}\argmax_{a\in \cA}f(x, a).
\]
\end{definition}

With these preliminaries, the disagreement coefficient is defined as follows.

\begin{definition}[Disagreement Coefficient]
\label{def:disagreement}
The disagreement coefficient for $\cF$
(with respect to $D_\cX$)
is defined as
\begin{align*}
\theta_0\ldef\sup_{\delta>0, \veps>0}\;
\frac{\delta}{\veps}\Pr_{D_{\cX}}
  &\BigBracks{
      x\in\mathrm{Dis}(\cF(\veps))\textnormal{ and}
\\[-6pt]
  &\quad
      \exists{}a\in{}A_{\cF(\veps)}(x) : \Wid_{\cF(\veps)}(x, a) > \delta
  }
.
\end{align*}
\end{definition}

Informally, the disagreement
coefficient is small if on most contexts either all functions in
$\cF(\veps)$ choose the same action according to their greedy
policies or all actions chosen by those policies have a low range of predicted rewards.

The following theorem provides regret bounds in terms of the disagreement coefficient. In this theorem and subsequent theorems we use $\tO$ to suppress polynomial dependence on $\log{}T$, $\log{}K$, and $\log(1/\delta)$, where $\delta$ is the failure probability.
Moreover, all results can be improved to bounds that are logarithmic (in $T$) under the standard Massart noise condition
(see the appendix for the definition and the complete theorem statements under this condition).
\begin{theorem}
\label{thm:disagreement}
With $\beta_m = \frac{(M-m+1)\concG}{\tau_m - 1}$ and $\concG= 16\log\left(\frac{2|\cG|KT^2}{\delta}\right)$,
\pref{alg:regression_ucb_elim} with \textrm{Option I} ensures that with probability at least $1- \delta$,
\[
\reg_T = \tO\left(T^{\frac{3}{4}}\prn*{\log\Abs{\cG}}^{\frac{1}{4}}\sqrt{\theta_0 K}\right).
\]
\end{theorem}
We state the theorem above for finite classes for simplicity. See \pref{thm:disagreement_with_massart} in \pref{app:disagreement} for the full version of this theorem, which applies to infinite classes and additionally obtains faster rates under the Massart noise condition.

\paragraph{Discussion} \pref{thm:disagreement} critically uses the
product class structure, specifically, the fact that the set $A_{t}$ computed by the algorithm coincides with the disagreement set $A_{\cF_m}(x_t)$ for $t\in\crl*{\tau_m,\ldots,\tau_{m+1}-1}$. This is true for product classes, but not necessarily for general (non-product) predictor classes.
Computing the disagreement set
efficiently for non-product classes is a challenge for future work.

While bounding the disagreement coefficients \textit{a priori} often requires
strong assumptions on the model class and the distribution, the size
of disagreement set can be easily checked empirically under the
product class assumption, and we include this diagnostic in our experimental
results.

Finally, it is not obvious how to use the disagreement coefficient to
analyze \pref{alg:regression_ucb_optimistic}. Our analysis
crucially requires that any plausibly optimal action $a$
be chosen with a reasonable probability, something which the
optimistic algorithm fails to ensure.

\subsection{Moment-based Analysis}
\label{sec:moment}

The disagreement-based analysis of \pref{thm:disagreement} is not
entirely satisfying because, even for simple linear predictors such as
in \linucb{} \citep{chu2011contextual} it is known that fairly strong assumptions
on the context distribution $D_{\cX}$ such as log-concavity
are required to bound the disagreement coefficient $\theta_0$
\citep{hanneke2014theory}. In order to capture and extend the linear
setting with distributional assumptions on the contexts, prior work
has used the notion of eluder dimension~\cite{russo2013eluder}. It
remains challenging, however, to show examples with a small eluder
dimension beyond linearly parameterized functions. In addition, taking
the worst-case over all histories, as in the definition of eluder dimension,
is overly pessimistic in the stochastic contextual-bandit setting.

To address the shortcomings of both the disagreement-based
analysis as well as eluder dimension for i.i.d.\ settings, we next
define a couple of distributional properties which we then use to
analyze the regret of our both algorithms.

\begin{definition}[\Surprise]
\label{def:infinity_l2}
The \surprise $\consL_1 > 0$ is the smallest constant such that for all $f\in\cF$, $x\in\cX$, and $a\in\cA$,
\begin{align*}
&
  \bigParens{f(x,a)-f^\star(x,a)}^2
\\
&\quad{}
\le
  \consL_1\En_{x'\sim D_\cX} \En_{a'\sim\Unif{\cA}}\BigBracks{\bigParens{f(x',a')-f^\star(x',a')}^2}
\enspace.
\end{align*}
\end{definition}
\comment{
        Let $\cD' = \cD_{\cX}\tens{}\textrm{Uniform}(\brk*{K})$.
        There exists constant $\consL_1$ such that for all $h\in\cF-f^{\star}$,
        \begin{equation}
        \nrm*{h}_{L_{\infty}(\cD')} \leq{} \consL_1\cdot\nrm*{h}_{L_2(\cD')}.
        \end{equation}
        \begin{remark}
        The $L_{\infty}$ norm can be relaxed to the $\psi_{2}$ Orlicz norm at the cost of an extra concentration argument.
        \end{remark}
}
The \surprise is small if functions with a small expected squared
error to $f^\star$ (under a uniform choice of actions) do not
encounter a much larger squared error on any single context-action
pair. 

The second quantity which we call the {\it \implicit} (\iec for short) relates
the expected regression error under actions chosen by the optimal
policy to the worst-case error on any other context-action pair.
Specifically, for any $\lambda\in [0, 1]$, first define $U_{\lambda}(a)$ to be the set of contexts where $a$ is the best action by a margin of $\lambda$:
\[
U_{\lambda}(a)\ldef\BigBraces{x\;\BigGiven\; f^{\star}(x,a)\ge f^{\star}(x,a') + \lambda\text{ for all $a'\ne a$}}.
\]

\begin{definition}[\Implicit---\iec]
\label{def:l2norm_ua}
For any $\lambda \in [0,1]$, the \implicit $\consL_{2, \lambda} > 0$
is the smallest constant such that for all $f\in\cF$, $x\in\cX$, and $a\in\cA$,
\begin{align}
&\notag
                 \bigParens{f(x,a)-f^\star(x,a)}^2
\\[-2pt]
&\label{eq:iec}
\le
     \consL_{2, \lambda}\En_{x'\sim D_\cX}\En_{a'\sim{}\Unif{\cA}}\Bigl[\ind\bigSet{x'\in U_{\lambda}(a')}
\\[-2pt]
&\hspace{1.6in}\notag
     {}\cdot\bigParens{f(x',a')-f^\star(x',a')}^2\Bigr].
\end{align}
\end{definition}
We next make a couple remarks about these definitions and their impact on the performance of \pref{alg:regression_ucb_elim} and \pref{alg:regression_ucb_optimistic}, and then spell them
out more precisely in \pref{thm:moment} and \pref{thm:moment_optimistic}.

\begin{itemize}
  \item By definition, $\consL_{2,\lambda}$ is non-decreasing in
$\lambda$. For \pref{alg:regression_ucb_elim} we can simply use
$\lambda=0$, for which it is sufficient to replace right-hand side of \pref{eq:iec} with
\[\frac{\consL_{2, 0}}{K}\En_{x\sim
  D_\cX}\brk{\prn{f(x,\pi^{\star}(x))-f^\star(x,\pi^{\star}(x))}^2}.\] The
analysis of \pref{alg:regression_ucb_optimistic} requires $\lambda >
0$, and this $\lambda$ must be used to tune the algorithm's warm-start period.
\item We always have $L_{1}\leq{}L_{2,0}$, but $L_{1}$ may be much smaller. Only \pref{alg:regression_ucb_optimistic} has a regret bound depending on $L_{1}$ directly, whereas the regret of \pref{alg:regression_ucb_elim} is independent of this constant.
\end{itemize}

With this in mind, we proceed to state the regret bound for \pref{alg:regression_ucb_elim} with a general predictor class $\cF$:
\begin{theorem}
\label{thm:moment}
With $\beta_m = \frac{(M-m+1)\concF}{\tau_m - 1}$ where $\concF= 16\log\left(\frac{2|\cF|T^2}{\delta}\right)$,
\pref{alg:regression_ucb_elim} with \textrm{Option II} ensures that with probability at least $1- \delta$,
\[
\reg_T = \tO\left(\sqrt{T\consL_{2,0}\log\Abs{\cF}}\right).
\]
\end{theorem}

We now move on to describe the performance guarantee for
\pref{alg:regression_ucb_optimistic}. Because this optimistic strategy does not explore as
readily as the elimination-based strategy
\pref{alg:regression_ucb_elim}, the analysis requires both that (i) the
\iec $L_{2,\lambda}$ be invoked for some $\lambda>0$
and (ii) that the algorithm use a warm-start period whose size grows
as $1/\lambda^2$.

\begin{theorem}
\label{thm:moment_optimistic}
With $\beta_m = \frac{(M-m+1)\concF}{\tau_m - 1}$ where $\concF= 16\log\left(\frac{2|\cF|T^2}{\delta}\right)$ and
$M_0 = 2 + \left\lfloor \log_2\left(1 + \frac{(2M+3)\consL_1\concF}{\lambda^2}\right) \right\rfloor$ for any $\lambda \in (0,1)$,
\pref{alg:regression_ucb_optimistic} ensures that with probability at least $1- \delta$,
\[
\reg_T = \tO\left(\frac{\consL_1\log\Abs{\cF}}{\lambda^2} +  \sqrt{T\consL_{2,\lambda}\log\Abs{\cF}}\right).
\]
\end{theorem}

Because \pref{alg:regression_ucb_optimistic} requires warm start,
the regret bounds of \pref{thm:moment_optimistic} for
\pref{alg:regression_ucb_optimistic} are always worse than those of
\pref{thm:moment} for \pref{alg:regression_ucb_elim}. \pref{app:moment_proofs} contains full versions of these theorems, \pref{thm:moment_with_massart} and \pref{thm:moment_optimistic_with_massart}, which---as in the disagreement case---obtain faster rates under the Massart noise condition and apply to infinite classes.

We now bound the
regret of both algorithms for some special cases.
\paragraph{Linear classes}
Consider the linear setting, as for instance in \linucb, with a fixed feature map $\phi:\cX\times\cA\to\bbR^{d}$
and $\cF=\crl*{(x,a)\mapsto{}w^\top\phi(x,a)\given w\in\cW}$ for some $\cW\subseteq{}\bbR^{d}$.
\begin{proposition}~
\label{prop:linear_ex}
\begin{itemize}
\item If $\nrm*{\phi(x,a)}_2\leq{}1$ and $\nrm*{w}_2\leq{}1$ then $L_{2,\lambda}$ is bounded by
  \[
  \frac{K}{\lambda_{\min}\prn*{\sum_{a\in\cA}\En_{x}\bigBracks{\ind\crl*{x\in{}U_{\lambda}(a)}\,\phi(x,a)\phi(x,a)^{\top}}}},
  \]
  where $\lambda_{\min}(\cdot)$ is the smallest eigenvalue of a matrix, and $L_1$ is bounded by
  \[
  \frac{K}{\lambda_{\min}\prn*{\sum_{a\in\cA}\En_{x}\brk*{\phi(x,a)\phi(x,a)^{\top}}}}.
  \]
\item In the sparse high-dimensional setting with
  $\nrm*{\phi(x,a)}_{\infty}\leq{}1$, $\nrm*{w}_{\infty}\leq{}1$, and
  $\nrm*{w}_{0}\leq{}s$, then $L_{2,\lambda}$ is bounded by

  \[
    \frac{2Ks}{\psi_{\min}\prn*{\sum_{a\in\cA}\En_{x}\bigBracks{\ind\crl*{x\in{}U_{\lambda}(a)}\,\phi(x,a)\phi(x,a)^{\top}}}},
  \]
  where $\psi_{\min}(A)\coloneqq
  \min_{w\neq{}0:\:\nrm*{w}_{0}\leq{}2s}w^\top\!\!Aw\,/\,w^\top w$ is the
  minimum restricted
  eigenvalue for $2s$-sparse predictors \citep{raskutti2010restricted}. The coefficient $L_1$ is bounded by
    \[
    \frac{2Ks}{\psi_{\min}\prn*{\sum_{a\in\cA}\En_{x}\bigBracks{\phi(x,a)\phi(x,a)^{\top}}}}.
  \]
\end{itemize}
\end{proposition}

We emphasize again that \pref{alg:regression_ucb_elim} has a better regret bound than \pref{alg:regression_ucb_optimistic} due to the warm-start phase in \pref{alg:regression_ucb_optimistic}. This is most easily seen by noting that $L_{2,\lambda}$ is non-decreasing in $\lambda$, then observing that the regret of \pref{alg:regression_ucb_elim} depends on $L_{2,0}$ while the regret of \pref{alg:regression_ucb_optimistic} requires $\lambda>0$ due to warm start (recall also that  $L_1 \leq L_{2,0}$). For the linear example above, this can be observed more directly by noting that the moment matrices
$\sum_a\En_{x}\bigBracks{\ind\crl*{x\in{}U_{\lambda}(a)}\,\phi(x,a)\phi(x,a)^{\top}}$ that appear in $L_{2,\lambda}$ in
\pref{prop:linear_ex} are lower bounded by
$\En_{x}\bigBracks{\phi(x,\pi^{\star}(x))\phi(x,\pi^{\star}(x))^{\top}}$ in the Loewner order when $\lambda=0$.

\paragraph{Sparse bandits} For the sparse high-dimensional
setting above, we can apply \pref{thm:moment} by discretizing the set of
weights and invoking a standard covering argument to obtain
$\log\Abs{\cF}={}O\prn*{s\log d}$\footnote{This is made precise via \pref{lem:disagreement_conc_F} in the appendix.}. This yields a near
dimension-independent bound on $\reg_T$ of
\[
 \tO\Parens{
    s\sqrt{KT\log{}d\,\bigm/\,
           \psi_{\min}\prn*{\En_{x}\bigBracks{\phi(x,\pi^{\star}(x))\phi(x,\pi^{\star}(x))^\top}}
    }}
.
\]

This improves upon the moment matrix conditions of \citet{bastani2015online}, although our algorithm is only efficient in the oracle model.\footnote{Because the predictor class $\cF$ is non-convex, this would require the slower binary search algorithm of~\citet{krishnamurthy2017active}.} Furthermore, \pref{alg:regression_ucb_elim} does not require a warm start based on distributional parameters unlike their algorithm (or our \pref{alg:regression_ucb_optimistic}).
Note that without the
scaling with $K$ as in our result, a $\sqrt{d}$ dependence is
unavoidable~\cite{abbasi2012online}. The result highlights the
strengths of our analysis in the best case compared with eluder
dimension, which does not adapt to sparsity structures. On the other
hand, for the standard \linucb{} setting, our result is inferior
by at least a factor of~$K$.\looseness=-1

\paragraph{Discussion}

Our moment-based analysis is influenced by the results of \citet{bastani2015online} for the (high-dimensional) linear setting. Our analysis extends to general classes and, when applied to \pref{alg:regression_ucb_elim}, it makes weaker assumptions. Similar assumptions have been used to analyze purely greedy linear contextual bandits \citep{bastani2017exploiting,kannan2018smoothed}; our assumptions are strictly weaker.\looseness=-1

\section{Experiments}
\label{sec:experiments}

We compared our new algorithms with existing oracle-based alternatives. In addition to showing that \rucb\footnote{\rucb refers collectively to both Algorithms~\ref{alg:regression_ucb_elim} and~\ref{alg:regression_ucb_optimistic}.} has strong empirical performance, our experiments also provide a more extensive empirical study of oracle-based contextual bandit algorithms than any past works~(e.g., \citealp{agarwal2014taming}, \citealp{krishnamurthy2016contextual}). Detailed descriptions of the datasets, benchmark algorithms, and oracle configurations, as well as further experimental results are included in \pref{app:experiments}.

\paragraph{Datasets}
We begin with 10 datasets with full reward information and simulate bandit feedback by withholding the rewards for actions not selected by the algorithm. First, there are two large-scale learning-to-rank datasets, Microsoft MSLR-WEB30k (\texttt{mslr}) \citep{qin2010mslr} and Yahoo! Learning to Rank Challenge V2.0 (\texttt{yahoo}) \cite{chapelle2011yahoo}, that have previously been used to evaluate contextual semibandits \citep{krishnamurthy2016contextual}. Second, we use a collection of eight classification datasets from the UCI repository \cite{lichman2013uci}, summarized in \pref{tab:uci} of \pref{app:datasets}.

The ranking datasets have natural rewards (relevances), but the rewards for the classification datasets always have multiclass structure ($1$ for the correct action and $0$ for all others). Therefore, to ensure that we evaluate at the full generality of the contextual bandit setting, we create eight ``noisy'' UCI datasets by sampling new rewards for the datasets according to a noisy reward matrix model described in \pref{app:experiments}. This yields additional 8 datasets for the total of 18.

On each dataset we consider several replicates obtained by randomly permuting examples and, on noisy UCI, also randomly generating rewards. All the methods are evaluated on the same set of replicates.

\paragraph{Algorithms}
We evaluate both \pref{alg:regression_ucb_elim} and \pref{alg:regression_ucb_optimistic} against three baselines, all based on various optimization-oracle assumptions. First, we use the standard \egreedy{} strategy \citep{langford2008epoch}. 
Second, we use the minimax-optimal \minimonster{} (\iltcb) strategy of \citet{agarwal2014taming}.\footnote{We use an implementation available at \url{https://github.com/akshaykr/oracle_cb}, which was also used by \citet{krishnamurthy2016contextual}.}

The \egreedy{} and \iltcb{} strategies both assume cost-sensitive classification oracles and come equipped with theoretical guarantees. The last baseline we consider is a bootstrapping-based exploration strategy of \citet{dimakopoulou2017estimation} (henceforth \bootstrap{}), which works in the regression-oracle model as we consider here, but without the corresponding theoretical analysis.

Note that the \linucb{} algorithm \citep{chu2011contextual, abbasi2011improved}, which is a natural baseline as well, coincides with our \pref{alg:regression_ucb_optimistic} (with a linear oracle), so we only plot the performance of \rucb with a linear oracle.

All of the algorithms update on an epoch schedule with epoch lengths of $2^{i/2}$, which is a theoretically rigorous choice for each algorithm.

\paragraph{Oracles} We consider two baseline predictor classes $\cF$: $\ls_2$-regularized linear functions (\textsf{Linear}) and gradient-boosted depth-$5$ regression trees (\textsf{GB5}). For the regularized linear class, \pref{alg:regression_ucb_optimistic} is equivalent to \linucb{} on an epoch schedule.\footnote{More precisely, it is equivalent to the well-known OFUL variant of \linucb{} \citep{abbasi2011improved}.}

When running both \rucb variants with the \textsf{GB5} oracle, we use a simple heuristic to substantially speed up the computation. At the beginning of each epoch~$m$, we find the best regression tree ensemble on the dataset so far (i.e., with respect to $\hR_m$). Throughout the epoch, we keep the structure of the ensemble fixed and in each call to $\textsc{Oracle}(H)$ we only re-optimize the predictions in leaves. This can be solved in closed form, similar to \textsf{LinUCB}, so the full binary search procedure (\pref{alg:binary_search_01}) does not need to be run.

\paragraph{Parameter Tuning}
We evaluate each algorithm for eight exponentially spaced parameter values across five repetitions. For \egreedy{} we tune the constant $\eps$, and for \iltcb{} we tune a certain smoothing parameter (see \pref{app:experiments}). For \pref{alg:regression_ucb_elim} and \pref{alg:regression_ucb_optimistic} we set $\beta_{m}=\beta$ for all $m$ and tune $\beta$. For \pref{alg:regression_ucb_optimistic} we use a warm start of $0$. We tune a confidence parameter similar to $\beta$ for \bootstrap{}.

\paragraph{Evaluation}
Each dataset is split into ``training data'', for which algorithm receives one example at a time and must predict online, and a holdout validation set. Validation is performed by simulating the algorithm's predictions on examples from the holdout set without allowing the algorithm to incorporate these examples. 
We also plot the validation reward of a ``supervised'' baseline obtained by training the oracle (either \textsf{Linear} or \textsf{GB5}) on the entire training set at once (including rewards for all actions).

For Algorithms~\ref{alg:regression_ucb_elim} and \ref{alg:regression_ucb_optimistic} we show average reward at various numbers of training examples for the best fixed parameter value in each dataset. For the baselines, we take the \emph{pointwise maximum of the average validation reward across all parameter values} for each number of examples to be as generous as possible. Thus, the curves for our methods correspond to an actual run of the algorithm, while the baselines are an upper envelope aggregating multiple parameter values.\looseness=-1

\begin{figure*}[h]
  \begin{centering}
\includegraphics[width=0.33\textwidth]{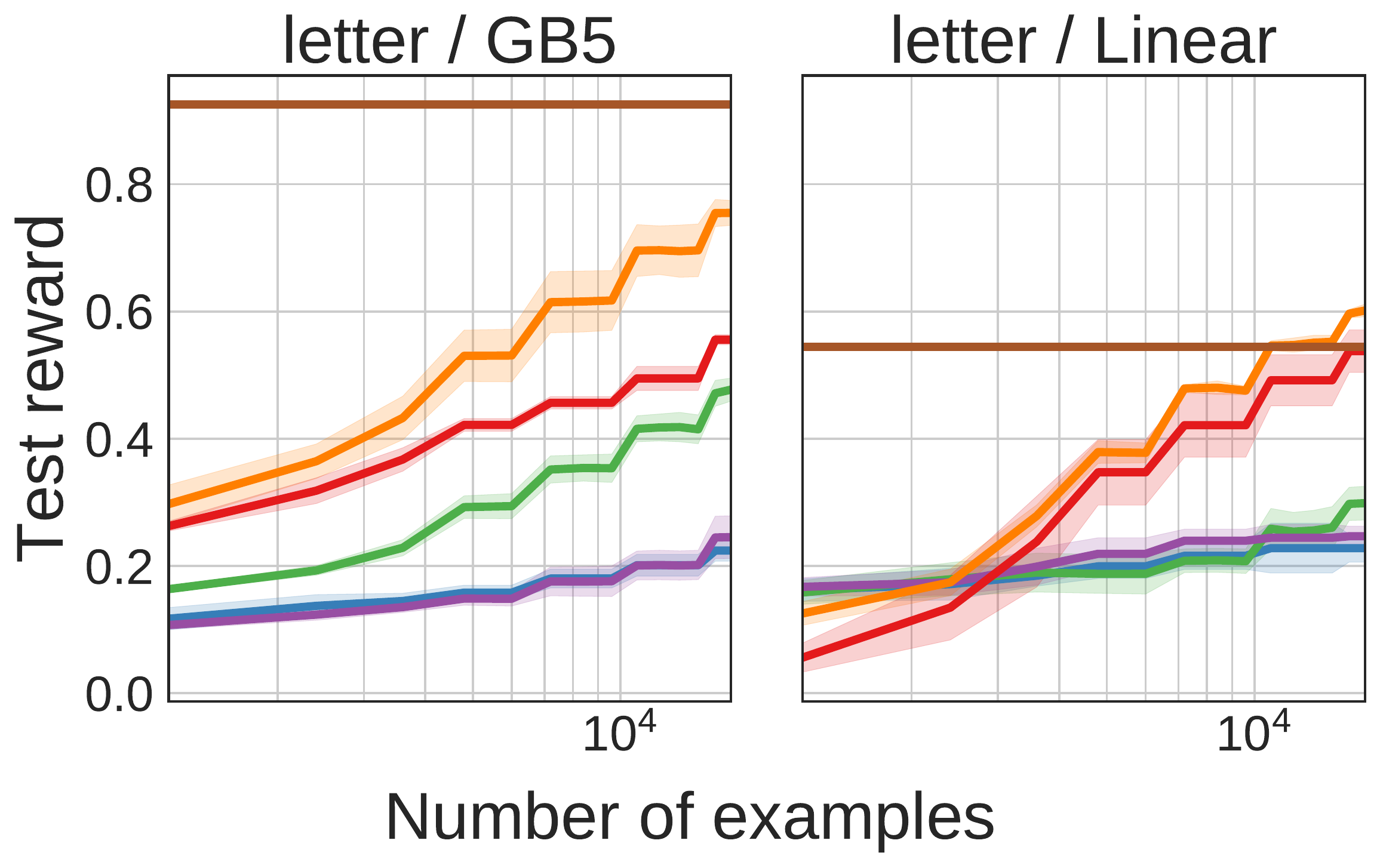}%
\includegraphics[width=0.33\textwidth]{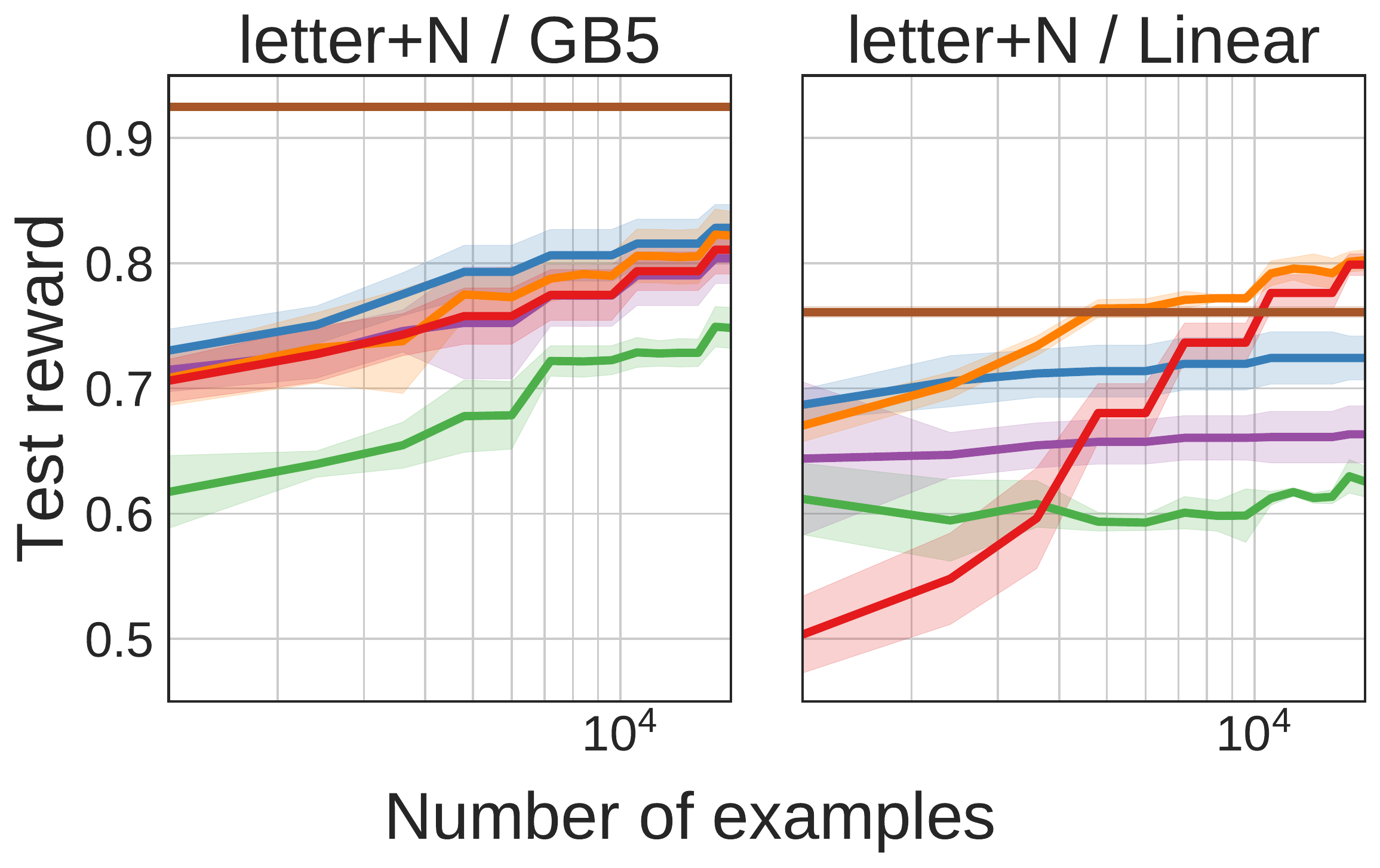}%
\includegraphics[width=0.33\textwidth]{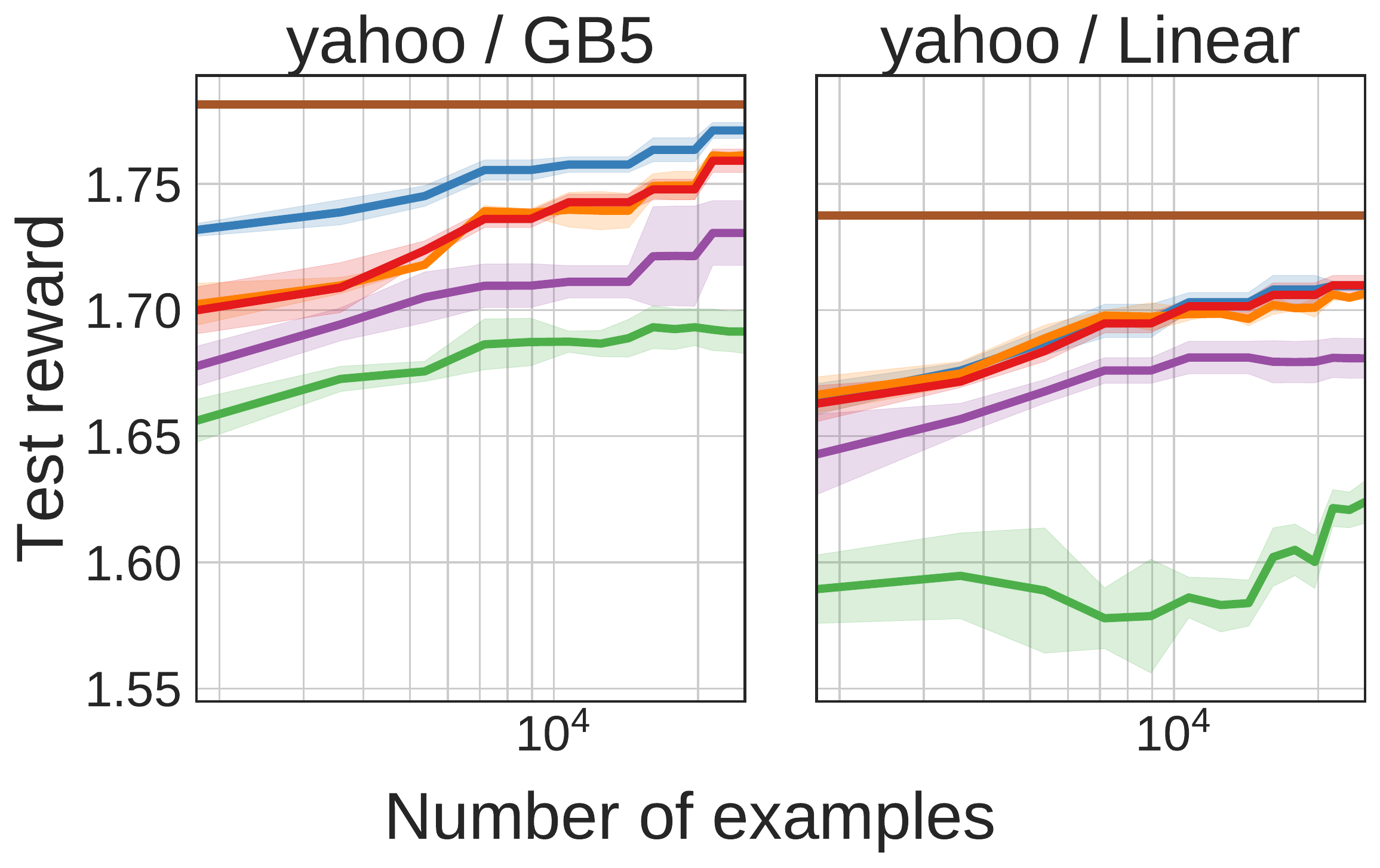}\\
\medskip
~\hfill\includegraphics[width=0.75\textwidth]{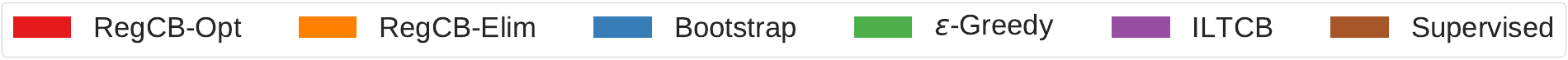}\hfill~
\caption{Validation performance for three representative datasets as a function of the number of rounds $t$ of interaction. The number of rounds $t$ is on a log scale. For each dataset, we show separately the performance with the \textsf{GB5} oracle and the \textsf{Linear} oracle.}
\label{fig:performance}
\end{centering}
\end{figure*}

\paragraph{Results: Performance}
\pref{fig:performance} shows average reward of each algorithm on a holdout validation set for three representative datasets, \texttt{letter} from UCI, \texttt{letter-noise} (the variant with simulated rewards), and \texttt{yahoo}.

\rucb (both Algorithms \ref{alg:regression_ucb_elim} and \ref{alg:regression_ucb_optimistic}) outperforms all baselines on the unmodified UCI datasets (e.g., \texttt{letter} in \pref{fig:performance}). On the noisy variants (e.g., \texttt{letter+N} in \pref{fig:performance}), the performance of the \iltcb{} and \bootstrap{} benchmarks improves significantly, with \bootstrap{} slightly edging out the rest of the algorithms. On the \texttt{yahoo} ranking dataset (\pref{fig:performance}, right), the ordering of the algorithms in performance is similar to noisy UCI datasets.

Validation performance plots for all datasets are in \pref{app:experiments}. Overall, we see that \rucb methods and \bootstrap generally dominate the field. While \bootstrap can outperform \rucb methods when using \textsf{GB5} models, the gap is typically quite small. For linear models, \rucb methods generally outperform \bootstrap. This hints that the stronger relative performance of \bootstrap under \textsf{GB5} models might be partly due to the approximation we make by only considering a fixed ensemble structure in each epoch.
We also observe that when \rucb methods outperform \bootstrap, the performance gap can often be quite large. We will see further evidence of this behavior in the next set of results.

\paragraph{Results: Aggregate Performance}

To rigorously draw conclusions about overall performance, \pref{fig:cdf} aggregates performance across all datasets. We compute ``normalized relative loss'' for each algorithm by rescaling the validation reward (computed as in \pref{fig:performance}) so that, at each round, the best performing algorithm has loss $0$ and the worst-performing has loss $1$. In each plot of \pref{fig:cdf} we consider normalized relative losses at a specific cutoff time ($1000$ examples in the left plot, and all examples in the center and right), and for each method we plot how often, i.e., on how many datasets, it achieves any given value of loss or better, as a function of the loss value. Thus, curves towards top left corner correspond to better methods, i.e., the methods that achieve lower relative loss on more datasets. The intercept at the relative loss $0$ shows the number of datasets where each algorithm is the best, and the intercept at
0.99 shows the number of datasets where the algorithm is not the worst (so the distance from top is the number of datasets where it is the worst).
Solid lines correspond to runs with the \textsf{GB5} oracle and dashed lines to the runs with the \textsf{Linear} oracle.

The aggregate performance with the \textsf{GB5} oracle across all datasets can be briefly summarized as follows: \rucb always beats \egreedy{} and \iltcb{}, but sometimes loses out to \bootstrap{}, and \bootstrap{} itself sometimes underperforms relative to the other baselines, especially on the UCI datasets. Even when \rucb is not the best, it is almost always within $20\%$ of the best. The elimination and optimistic variants of \rucb have comparable performance, with elimination performing slightly better in aggregate.

The \rucb algorithms with the \textsf{GB5} oracle also dominate the \egreedy{}, \iltcb{}, and \bootstrap{} baselines when they are equipped with \textsf{Linear} oracles (the dashed lines in \pref{fig:cdf}). When the \rucb algorithms use the \textsf{Linear} oracle they also dominate the baselines with the \textsf{Linear} oracle across all datasets, including \bootstrap{}. This suggests that the gap between \rucb and \bootstrap{} for the \textsf{GB5} oracle may be due to
the approximation we make by only considering a fixed ensemble structure in each epoch, as we noted earlier.\footnote{The aggregate plots for \rucb with the \textsf{Linear} oracle can be found in \pref{app:experiments} along with additional aggregate plots.}

\begin{figure*}
\begin{centering}
\includegraphics[width=0.33\textwidth]{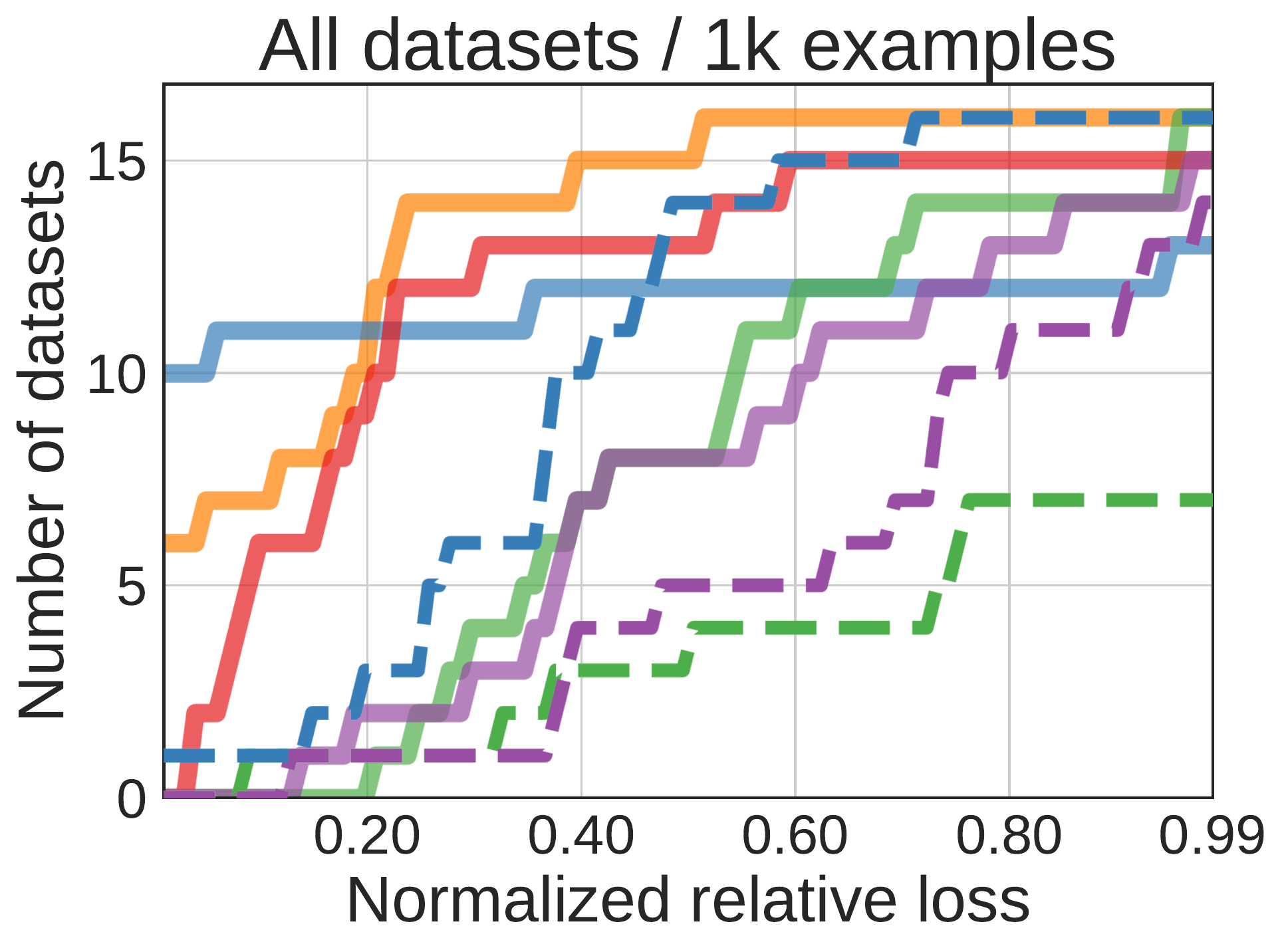}%
\includegraphics[width=0.33\textwidth]{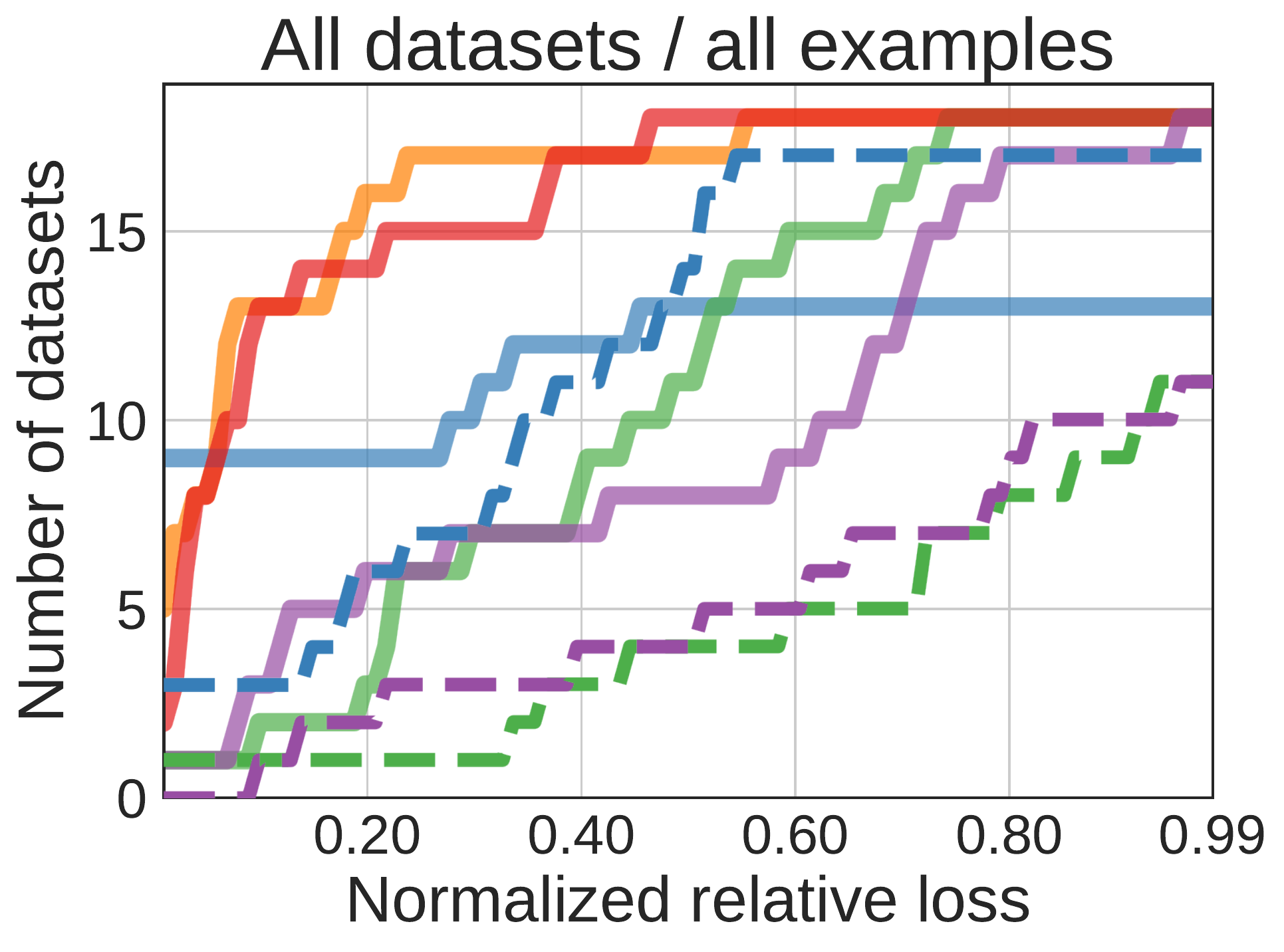}%
\includegraphics[width=0.33\textwidth]{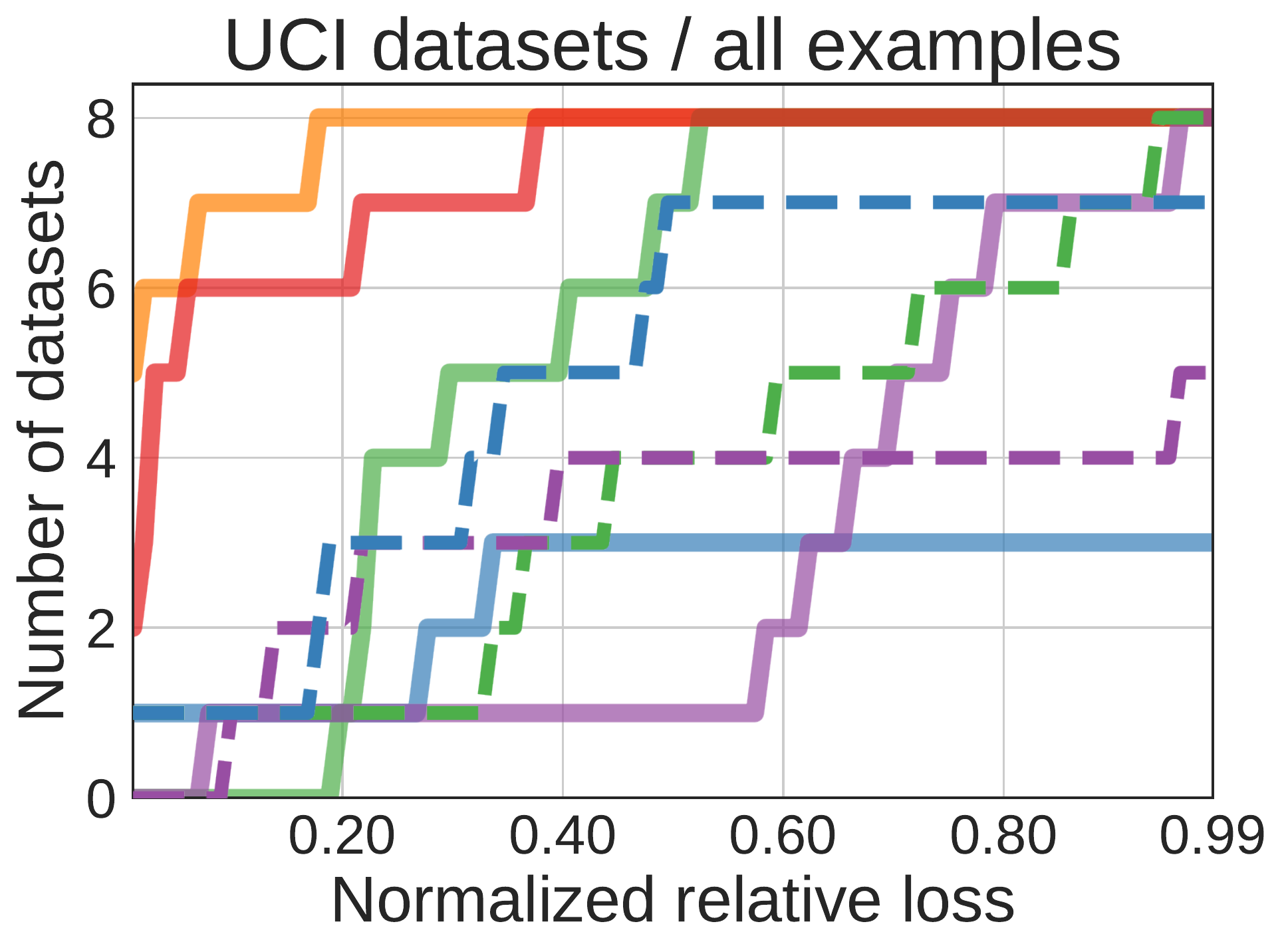}\\
\medskip
\hfill\includegraphics[width=0.66\textwidth]{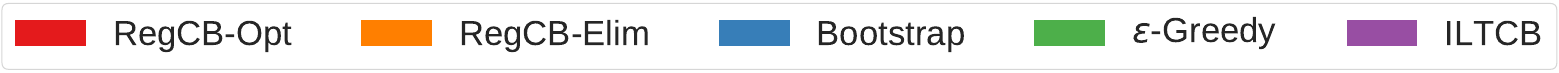}\hfill
\end{centering}
\caption{Aggregate performance across all datasets, at various sample sizes; solid lines --- \textsf{GB5} oracle; dashed lines --- \textsf{Linear} oracle. Left: All datasets (the UCI datasets, their noisy variants, and the Microsoft and Yahoo ranking datasets) at $1\,000$ examples (datasets with fewer examples dropped). Center: All datasets at their final round. Right: Unmodified UCI at their final round.}
\label{fig:cdf}
\end{figure*}

\paragraph{Results: Confidence Width}

The analysis of \rucb relies on distributional assumptions on $D$ (disagreement coefficient or moment parameters) that are not necessarily easy to verify. Note that the main role of these parameters is to control the rate at which confidence width $\Wid_{\cF_m}(x,a) =\High_{\cF_m}(x_t, a)-\Low_{\cF_m}(x_t, a)$ used in \rucb shrinks, since the small widths imply that the algorithm makes good decisions and thus has low regret.

To investigate whether the width $\Wid_{\cF_m}$ indeed shrinks empirically, we compute it on each dataset for \pref{alg:regression_ucb_optimistic}. We also compute an analogous width parameter for \bootstrap (see \pref{app:experiments}). Finally for both \pref{alg:regression_ucb_optimistic} and \bootstrap we compute the size of the ``disagreement set'' $A_{t}$, defined in \pref{alg:regression_ucb_elim},
which measures how many actions the algorithm thinks are plausibly best.\footnote{%
This set is well-defined for both \textsf{RegCB-Opt} and \bootstrap even through neither algorithm instantiates it explicitly. For the \texttt{yahoo} and \texttt{mslr} datasets this $\abs*{A_{t}}$ is technically a lower bound on the true disagreement set size $\abs*{A_{\cF_m}(x_t)}$ because our classes $\cF$ do not have product structure on these datasets---see discussion in \pref{sec:disagreement}.}

\pref{fig:width} shows width and disagreement for a representative sample of datasets under the \textsf{GB5} oracle; the remaining datasets are in \pref{app:experiments}. The figure suggests that our distributional assumptions are reasonable for real-world datasets. In particular, for our algorithm, the width decays roughly as $T^{-1/3}$ for \texttt{letter} and $T^{-1/2}$ for \texttt{letter+N} and \texttt{yahoo}. Interestingly, the best hyper-parameter setting for \bootstrap on \texttt{letter} yields low but essentially constant (i.e., not shrinking) width, which in our experiments is associated
with the poor validation reward. This suggests that while the \bootstrap confidence intervals are small, they may not be faithful in the sense of containing $f^{\star}(x,a)$.

\begin{figure*}[h]
  \begin{centering}
\includegraphics[width=0.33\textwidth]{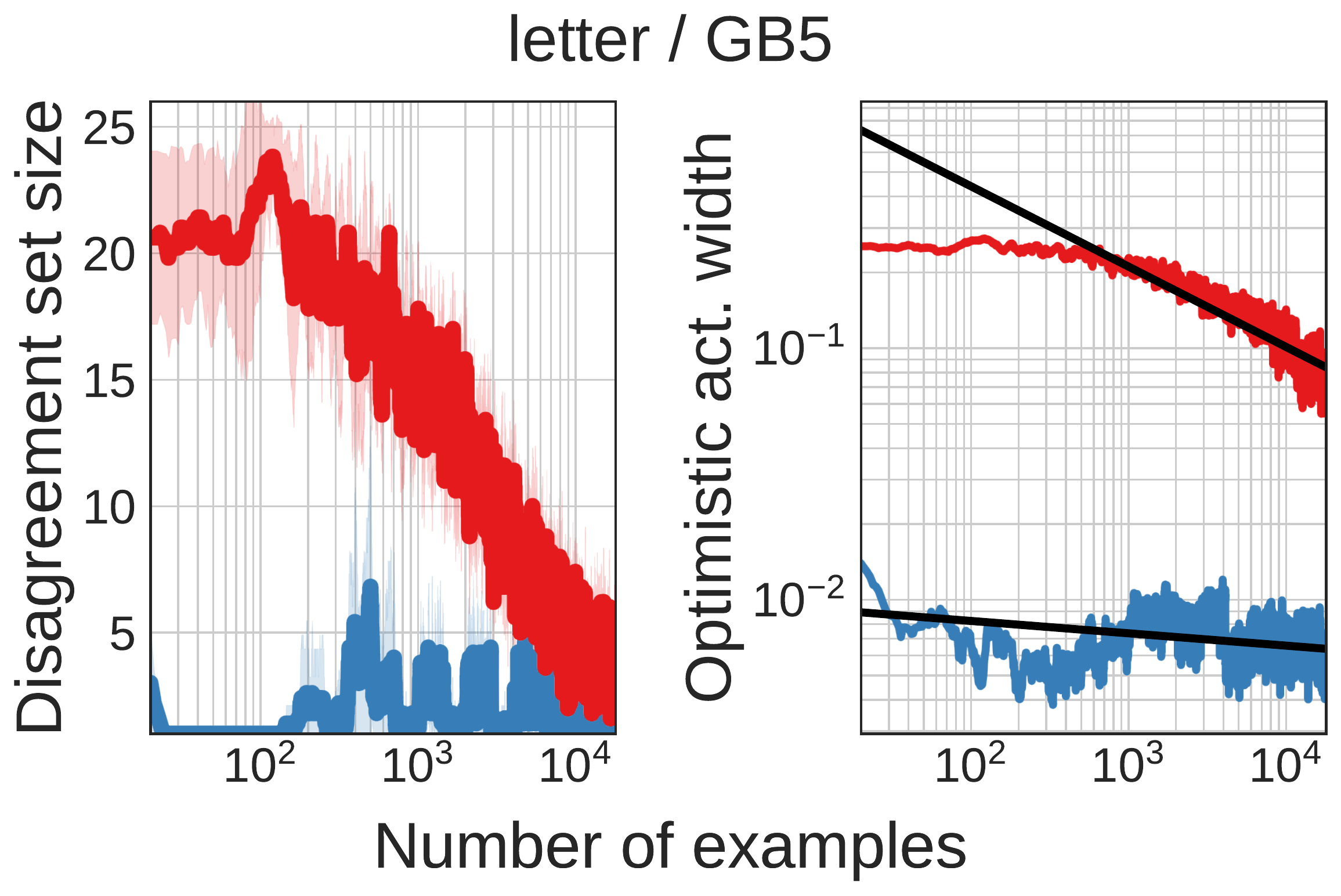}\hfill\includegraphics[width=0.33\textwidth]{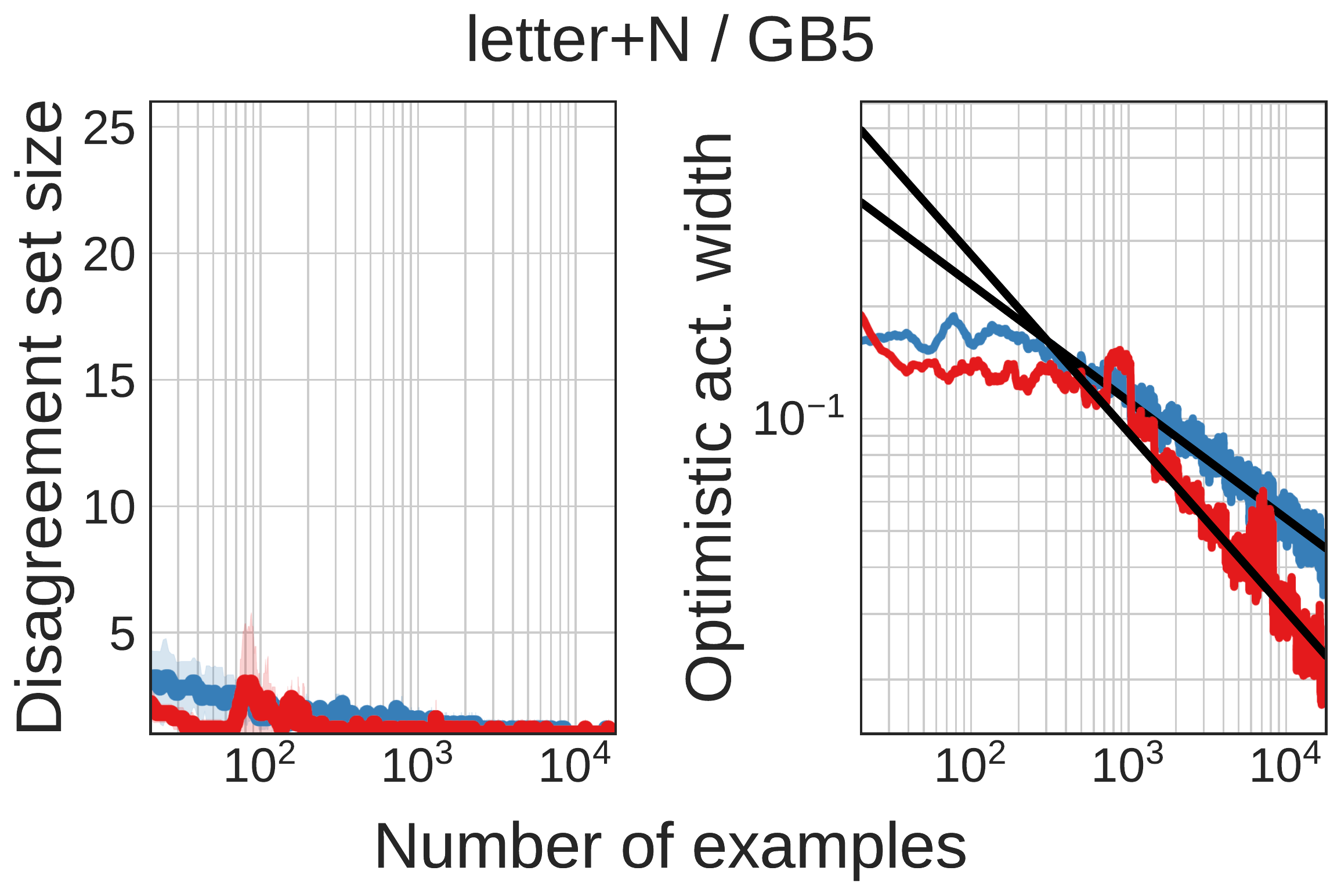}\hfill\includegraphics[width=0.33\textwidth]{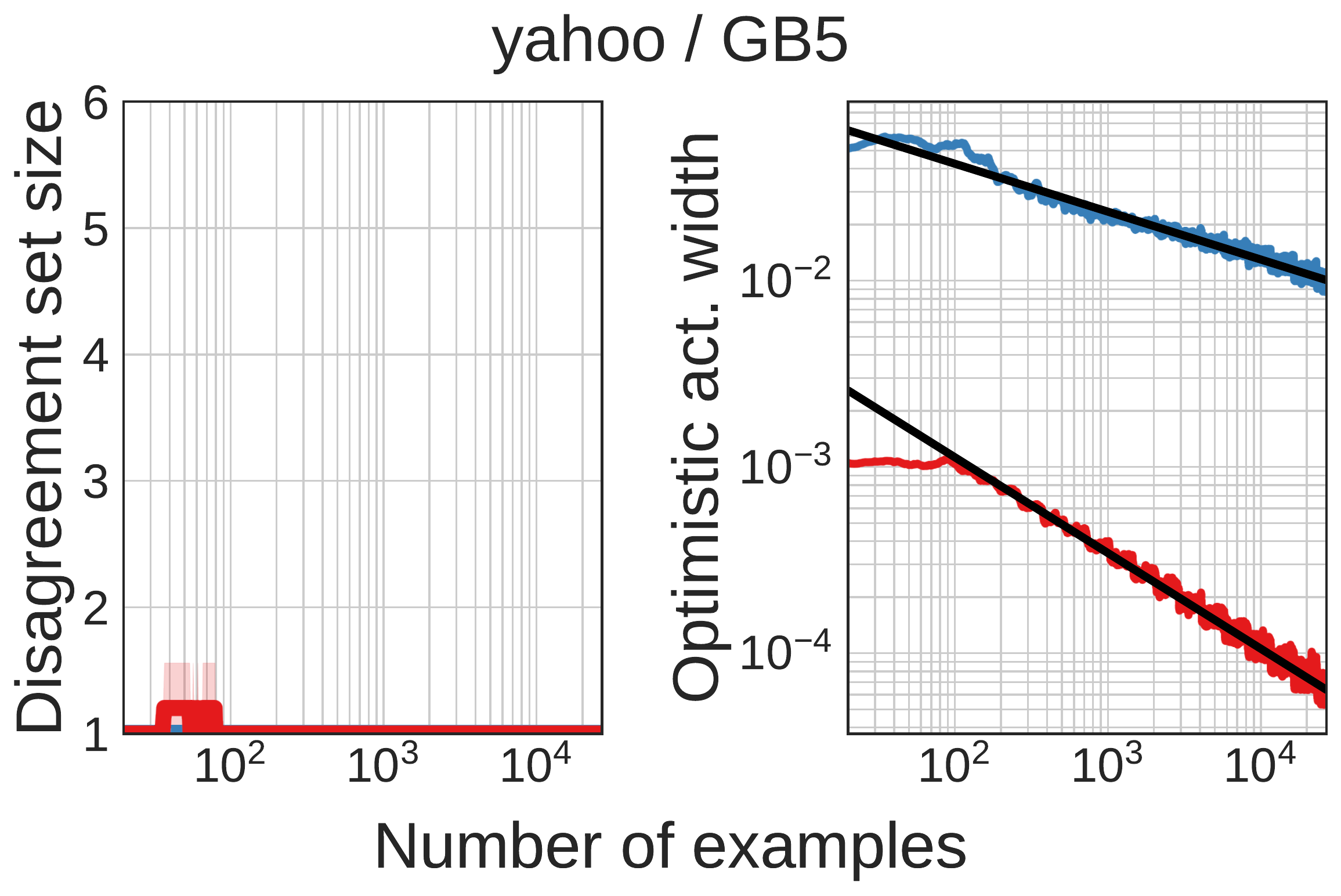}\\
    ~\hfill\includegraphics[width=0.28\textwidth]{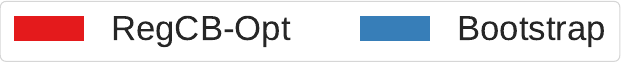}\hfill~
\caption{For each dataset, disagreement set size as a function of number of rounds $t$ (with $t$ on a log scale), and the log-log plot of the width of the optimistic action as a function of $t$; the optimistic action is the action chosen by \pref{alg:regression_ucb_optimistic}. All plots are averaged using a sliding window of length $20$.
Black lines on the width plots are best linear fits, whose slopes suggest the rate of the width decay as follows: letter/\bootstrap: $-0.05$, letter/\rucb: $-0.34$, letter-noise/\bootstrap: $-0.33$, letter-noise/\rucb: $-0.51$, yahoo/\bootstrap: $-0.26$, yahoo/\rucb: $-0.52$.}
\label{fig:width}
\end{centering}
\end{figure*}
\section{Conclusion and Discussion}
\label{sec:conclusion}
This work serves as a starting point for what we hope will be a fruitful line of research on oracle-efficient contextual bandit algorithms in realizability-based settings. We have shown that the $\rucb{}$ family of algorithms have strong empirical performance and enjoy nice theoretical properties. These results suggest some compelling directions for future work:
\begin{itemize}
\item Is there a regression oracle--based algorithm that achieves the optimal $\tO(\sqrt{KT\log\Abs{\cF}})$ regret? For example, can the regressor elimination strategy of \citet{agarwal2012contextual} be oraclized?
\item Given the competitive empirical performance of \bootstrap, are there reasonable distributional assumptions similar to those in \pref{sec:regret} under which it can be analyzed? There is some recent work in this direction for the special case of linear models \citep{lu2017ensemble}.
\item Randomizing uniformly or putting all the mass on the optimistic choice are two extreme cases of choosing amongst the plausibly optimal actions. Are there better randomization schemes amongst these actions that lead to stronger regret guarantees?
\end{itemize} 

\bibliography{refs}
\bibliographystyle{icml2018}

\appendix
\onecolumn

\section{Proofs}

\subsection{Proofs from \pref{subsec:computation}}
\label{app:binary_search}

\begin{algorithm}[t]
\caption{\textsc{BinSearch.Unbounded.High}}
\label{alg:binary_search_high}
\begin{algorithmic}[1]

\State{\textbf{Input}: context-action pair $(x,a)$, history $H$, radius $\beta>0$, and precision $\alpha>0$ }
\State{Let $R(f)\coloneqq\sum_{(x', a', r')\in{}H}(f(x',a')-r')^{2}$.}
\State{Let $\wt{R}(f, w)\coloneqq R(f) + \frac{w}{2}(f(x,a) - 2)^{2}$}
\State{$\wlo\gets 0,\;\whi\gets\beta/\alpha$}
\smallskip
\Statex{\textbf{\texttt{// Invoke oracle twice}}}
\State{$\flo\gets\argmin_{f\in\cF}\wt{R}(f, \wlo),\;\zlo\gets\flo(x,a)$}
\State{$\fhi\gets\argmin_{f\in\cF}\wt{R}(f, \whi),\;\zhi\gets\fhi(x,a)$}
\State{$\Rmin\gets R(\flo)$}
\State{\textbf{if} $\zlo\ge 1$ or $R(\flo)=R(\fhi)$ \textbf{then return} $1$}
\State{$\Delta\gets\alpha\beta/(2-\zlo)^3$}
\While{$\abs{\zhi-\zlo} > \alpha$ and $\abs{\whi-\wlo} > \Delta$}
\State{${w}\gets(\whi+\wlo)/2$}
\smallskip
\Statex{\hspace{\algorithmicindent}\textbf{\texttt{// Invoke oracle.}}}
\State{${f}\gets\argmin_{\tilde{f}\in\cF}\wt{R}(\tilde{f}, w),\;{z}\gets{f}(x,a)$}
\If{$R({f})\ge\Rmin+\beta$}
\State{$\whi\gets {w},\;\zhi\gets {z}$}
\Else
\State{$\wlo\gets {w},\;\zlo\gets {z}$}
\EndIf
\EndWhile
\State{\textbf{return} $\min\set{\zhi,1}$.}
\end{algorithmic}
\end{algorithm}

\begin{algorithm}[t]
\caption{\textsc{BinSearch.Unbounded.Low}}
\label{alg:binary_search_low}
\begin{algorithmic}[1]
\State{\textbf{Input}: context-action pair $(x,a)$, history $H$, radius $\beta>0$, and precision $\alpha>0$ }
\State{Let $R(f)\coloneqq\sum_{(x', a', r')\in{}H}(f(x',a')-r')^{2}$.}
\State{Let $\wt{R}(f, w)\coloneqq R(f) + \frac{w}{2}(f(x,a) + 1)^{2}$}
\State{$\wlo\gets 0,\;\whi\gets\beta/\alpha$}
\smallskip
\Statex{\textbf{\texttt{// Invoke oracle twice}}}
\State{$\flo\gets\argmin_{f\in\cF}\wt{R}(f, \wlo),\;\zlo\gets\flo(x,a)$}
\State{$\fhi\gets\argmin_{f\in\cF}\wt{R}(f, \whi),\;\zhi\gets\fhi(x,a)$}
\State{$\Rmin\gets R(\flo)$}
\State{\textbf{if} $\zlo\le 0$ or $R(\flo)=R(\fhi)$ \textbf{then return} $0$}
\State{$\Delta\gets\alpha\beta/(1+\zlo)^3$}
\While{$\abs{\zhi-\zlo} > \alpha$ and $\abs{\whi-\wlo} > \Delta$}
\State{${w}\gets(\whi+\wlo)/2$}
\smallskip
\Statex{\hspace{\algorithmicindent}\textbf{\texttt{// Invoke oracle.}}}
\State{${f}\gets\argmin_{\tilde{f}\in\cF}\wt{R}(\tilde{f}, w),\;{z}\gets{f}(x,a)$}
\If{$R({f})\ge\Rmin+\beta$}
\State{$\whi\gets {w},\;\zhi\gets {z}$}
\Else
\State{$\wlo\gets {w},\;\zlo\gets {z}$}
\EndIf
\EndWhile
\State{\textbf{return} $\max\set{\zhi,0}$.}
\end{algorithmic}
\end{algorithm}

We prove the statement of \pref{thm:binary_search} for \textsc{BinSearch.Unbounded.High} (\Alg{binary_search_high}), which
does not require the predictors in $\cF$ to be bounded in $[0,1]$. Note however that the actual rewards are still always bounded in $[0,1]$, so that $f^\star(x,a)$ is always bounded by the realizability assumption. Compared with \Alg{binary_search_01},
the algorithm includes some handling of special cases, which are automatically excluded
in \Alg{binary_search_01} by the assumption about boundedness.
The performance guarantee for \textsc{BinSearch.Unbounded.Low} (\Alg{binary_search_low}) is
analogous and therefore is omitted.

\begin{lemma}
Let $\cF$ be convex and closed under pointwise convergence. Consider a run of \Alg{binary_search_high}. Let $R(f)$ and $\Rmin$ be defined as in \pref{alg:binary_search_high}
and let
\[
   z^\star\coloneqq
   \max\bigSet{f(x,a):\: f\in\cF\text{ such that }R(f)\le\Rmin+\beta}
   \enspace.
\]
Then \pref{alg:binary_search_high} returns
$z$ such that $\bigAbs{z-\min\set{z^\star,1}}\le\alpha$
after at most $O\bigParens{\log(1/\alpha) + \log(\max\crl*{2-z_0, 1})}$ iterations, where $z_0=\fmin(x,a)$ and
$\fmin=\argmin_f R(f)$.
\end{lemma}
\begin{corollary}
    If $f(x,a)\in[0,1]$ for all $f\in\cF$, $x\in\cX$ and $a\in\cA$, then \pref{alg:binary_search_high} returns $z$ such that $\abs*{z-z^{\star}}\leq{}\alpha$ after at most $O\prn*{\log(1/\alpha)}$ iterations.
\end{corollary}
\begin{proof}
The proof works by analyzing a univariate auxiliary function $\phi:\R\to\R\cup\set{\infty}$, which
maps $z\in\R$ to the smallest empirical error $R(f)$ among all functions that predict $f(x,a)=z$,
\begin{equation}
\label{eq:phi:z}
  \phi(z)\coloneqq
\begin{cases}
  \infty&\text{if $f(x,a)<z$ for all $f\in\cF$}
\\
  \min\bigSet{ R(f):\: f\in\cF \text{ and } f(x,a)=z }
        &\text{otherwise.}
\end{cases}
\end{equation}
note that we do not need to worry about the case when $f(x,a)$ might take values both larger and smaller than $z$ but not $z$ exactly due to the assumed convexity of $\cF$. We first show that this function is well-defined (i.e., the minimum in the definition is attained),
convex and lower semicontinuous.
We begin by embedding the least-squares optimization in a finite dimensional space. Let
$H=\set{(x_i,a_i,r_i)}_{i=1}^n$ and define $x_{n+1}\coloneqq x$ and $a_{n+1}\coloneqq a$.
We associate each $f$ with a vector $\vv^f\in\R^{n+1}$ with entries $v^f_i=f(x_i,a_i)$. Let
$\cV\coloneqq\set{\vv^f:\:f\in\cF}$. Since $\cF$ is closed under pointwise convergence and
convex, the set $\cV$ must also be closed and convex.

For $\vv\in\R^{n+1}$, let
\[
  \rho(\vv)\coloneqq\sum_{i=1}^n (v_i-r_i)^2
\enspace,
\]
where $r_i$ are the rewards from $H$. Thus,
\[
  R(f)=\sum_{i=1}^n (f(x_i,a_i)-r_i)^2=\rho(\vv^f)
\enspace,
\]
and therefore
\[
  \phi(z)
=
  \min\bigSet{ R(f):\: f\in\cF \text{ and } f(x)=z }
=
  \min\bigSet{ \rho(\vv):\: \vv\in\cV \text{ and } v_{n+1}=z }
\enspace,
\]
where we use the convention that the minimum of an empty set equals $\infty$.
The attainment of the minimum now follows by convexity and continuity of $\rho$ along the affine space $\set{v_{n+1}=z}$.
The convexity and lower semicontinuity of $\phi$ follows by Theorem 9.2 of \citet{rockafellar70}.

The upper confidence value $z^\star$ is then the largest $z$ for which $\phi(z)\le\Rmin+\beta$:
\[
  z^\star=\max\set{z:\:\phi(z)\le\Rmin+\beta}
\enspace.
\]
Furthermore, for any $w\ge 0$, define
\[
  z_w\coloneqq\argmin_{z\in\R}\Bracks{\phi(z)+\frac{w}{2}(2-z)^2}
\enspace.
\]
Thus, $z_w=f(x,a)$ where $f=\argmin_{\tilde{f}\in\cF}\wt{R}(\tilde{f},w)$ with $\wt{R}$ as defined in the algorithm. The algorithm
maintains the identities $\zlo=z_{\wlo}$ and $\zhi=z_{\whi}$, so it can be rewritten as follows:
\begin{algorithmic}[1]
\State{\textbf{if} $z_0\ge 1$ or $\phi(z_0) = \phi(z_{\beta/\alpha})$ \textbf{then return} $1$}
\State{$\wlo\gets 0,\;\whi\gets\beta/\alpha,\;\Delta\gets\alpha\beta/(2-z_0)^3$}
\While{$\abs{z_{\whi}-z_{\wlo}} > \alpha$ and $\abs{\whi-\wlo} > \Delta$}
\State{$w\gets(\whi+\wlo)/2$}
\If{$\phi(z_w)>\phi(z_0)+\beta$}
\State{$\whi\gets w$}
\Else
\State{$\wlo\gets w$}
\EndIf
\EndWhile
\State{\textbf{return} $\min\set{z_{\whi},1}$.}
\end{algorithmic}
Note that $z_0=\fmin(x,a)$ where $\fmin$ is the minimizer of $R$,
and therefore $\phi$ attains its minimum at $z_0$. If $z_0\ge 1$,
then the algorithm terminates and returns $1$. Since $z^\star\ge z_0\ge 1$,
in this case the lemma holds.

Also, note that if $z^\star=z_0<1$ then the algorithm immediately terminates with
$z_{\wlo}=z_{\whi}=z_0$. This is because of the fact that $z^\star=z_0$, given $\beta>0$, implies by
lower semicontinuity that $\phi(z)=\infty$ for all $z>z_0$ and thus $z_w=z_0$ for all $w>0$.

The final special case to consider is when $\phi(2)=\phi(z_0)$, i.e., there exist a minimizer $\tilde{f}_{\min}$ of $R$,
which satisfies $\tilde{f}_{\min}(x,a)=2$ and thus for any $w$, it also minimizes $\tilde{R}(f,w)$. This is exactly
the case when $R(\flo)=R(\fhi)$ in \pref{alg:binary_search_high} and in this case the algorithm returns $1$ and the lemma holds.

In the remainder of the proof we assume that $\phi(2)>\phi(z_0)$, $z_0<1$ and $z_0<z^\star$. By convexity of $\phi$, we know that $\phi$ is
non-decreasing on $[z_0,\infty)$, and we will argue that by performing the binary search over $w$,
the algorithm is also performing a binary search over $z_w$ to find the point $z^\star$.

We begin by characterizing $z_w$ and showing that $z_w<2$ for all $w$. For any $w>0$, by first-order optimality,
\begin{equation}
\label{eq:first:order}
   \phi'(z_w)-w(2-z_w)=0
\end{equation}
for some $\phi'(z_w)\in\partial\phi(z_w)$, where $\partial\phi$ denotes the subdifferential. First, note that $z_w\ge z_0$, because at any $z<z_0\le 1$, we have $w(2-z)>0$ while also $\phi'(z)\le 0$, because $\phi$ is convex and minimized
at $z_0$. Therefore, at $z<z_0$, we have $\phi'(z)-w(2-z)<0$, so \Eq{first:order} can only be satisfied by $z_w\ge z_0$. Rearranging, we obtain
\begin{equation}
\label{eq:w}
  w=\frac{\phi'(z_w)}{2-z_w}
\enspace.
\end{equation}
Since $z_w\ge z_0$, the convexity of $\phi$ implies that $\phi'(z_w)\ge 0$. Since $w>0$, we therefore must in fact have $\phi'(z_w)>0$
and
\begin{equation}
\label{eq:zw:bound}
   z_w<2
\text{ for all $w>0$.}
\end{equation}
\Eq{w} now implies that $z_w$ is non-decreasing as a function of $w$.

Let $w^\star$ be such that $z_{w^\star}=z^\star$ (this can be obtained by Eq.~\ref{eq:w}).
The remainder of the proof proceeds in two steps. The first step establishes
that our initial setting $\whi=\beta/\alpha$ is large enough to guarantee
that the initial interval $[z_{\wlo},z_{\whi}+\alpha] = \brk*{z_0, z_{\beta/\alpha}+\alpha}$ contains the solution $\min\set{z^\star,1}$.
The execution of the algorithm then continues to maintain this condition, i.e., $\min\set{z^\star,1}\in[z_{\wlo},z_{\whi}+\alpha]$, which we refer to as the \emph{invariant},
while halving $\abs{\whi-\wlo}$. That the invariant holds can be seen as follows: First, if $z_{0} \leq{} z^{\star} \leq{} z_{\beta/\alpha}$, then the update rule guarantees that $z_{\wlo} \leq{} z^{\star} \leq{} z_{\whi}$ for every iteration. On the other hand, if $z^{\star} > z_{\beta/\alpha}$, then $z_{\whi}=z_{\beta/\alpha}$ for every iteration, and so Step 1 below guarantees that $z^{\star}\in\brk*{z_{\beta/\alpha}, z_{\beta/\alpha}+\alpha} \supseteq \brk*{z_{\wlo}, z_{\whi}+\alpha}$.

The algorithm terminates after at most
\[
    \log_2\Parens{\frac{\beta/\alpha}{\Delta}}=\log_2\Parens{\frac{(2-z_0)^3}{\alpha^2}}=O\bigParens{\log(1/\alpha) + \log(2-z_0)}
\]
iterations. If the reason for termination is that $\abs{\zhi-\zlo}\le\alpha$ then the lemma follows, thanks to the invariant.
Otherwise, we must have $\abs{\whi-\wlo}\le\Delta$, so our invariant together with the monotonicity of $z_w$ in $w$ implies that
$\whi\le w^\star+\Delta$. Our second step below establishes that in this case we must also have $\zhi\le z^\star+\alpha$. Our invariant separately
also implies that $\min\set{z^\star,1}\le\zhi+\alpha$, so altogether we have $\min\set{\zhi,1}-\alpha\le\min\set{z^\star,1}\le\min\set{\zhi,1}+\alpha$,
proving the lemma. It remains to prove the two steps.

\paragraph{Step 1: $z_0\le \min\set{z^\star,1} \le z_{\beta/\alpha}+\alpha$.} The first inequality is immediate from the definition of $z^\star$ and the fact
that $z_0<1$. The second inequality holds if $z_{\beta/\alpha}\ge 1$, so it remains to consider $z_{\beta/\alpha}\le 1$. Let $w=\beta/\alpha$. Then by \Eq{w},
\[
  \frac{\beta}{\alpha}=w=\frac{\phi'(z_w)}{2-z_w}\le\phi'(z_w)
\enspace,
\]
where the last step follows because $z_w\le 1$. Now by convexity of $\phi$, for any $\tilde{\alpha}>\alpha$
\[
   \phi(z_w+\tilde{\alpha})
   \ge
   \phi(z_w)+\tilde{\alpha}\phi'(z_w)
   \ge
   \phi(z_w)+\tilde{\alpha}\cdot\frac{\beta}{\alpha}
   >
   \phi(z_0)+\beta
\enspace,
\]
where the last step follows because $\phi(z_w)\ge\phi(z_0)$ and $\tilde{\alpha}>\alpha$. This shows that $z^\star\le z_w+\alpha$ and completes Step 1.

\paragraph{Step 2: $z_{w^\star+\Delta}\le z^\star+\alpha$.}
Let $w=w^\star+\Delta$. Then by convexity
\[
   \phi(z_0)\ge\phi(z^\star)+(z_0-z^\star)\phi'(z^\star)
\enspace,
\]
and since $z^\star>z_0$, we can rearrange this inequality to give
\[
  \phi'(z^\star)\ge\frac{\phi(z^\star)-\phi(z_0)}{z^\star-z_0}=\frac{\beta}{z^\star-z_0}\ge\frac{\beta}{2-z_0}
\enspace,
\]
where the last inequality follows by \Eq{zw:bound}. By \Eq{w}, we also have
\[
  w^\star=\frac{\phi'(z^\star)}{2-z^\star}\ge\frac{\phi'(z^\star)}{2-z_0}
\]
because $z^\star>z_0$. Combining the two bounds yields
\begin{equation}
\label{eq:wstar}
  w^\star\ge\frac{\beta}{(2-z_0)^2}
\enspace.
\end{equation}
Applying now \Eq{w} twice, and also using the monotonicity of $\phi'$, we obtain
\[
w  =\frac{\phi'(z_w)}{2-z_w}
   \ge\frac{\phi'(z^\star)}{2-z_w}
   =w^\star\cdot\frac{2-z^\star}{2-z_w}
\enspace.
\]
Therefore,
\begin{align*}
   2-z_w
   &\ge\frac{w^\star}{w}\cdot(2-z^\star)
\\
   z^\star-z_w
   &\ge\frac{w^\star}{w}\cdot(2-z^\star) - (2-z^\star)
\enspace.
\end{align*}
Rearranging,
\[
   z_w-z^\star
   \le\frac{w-w^\star}{w}\cdot(2-z^\star)
   =\frac{\Delta}{w}\cdot(2-z^\star)
   \le\frac{\Delta}{w^\star}\cdot(2-z_0)
\enspace,
\]
where the final inequality uses the fact that $w\ge w^\star$ and $z^\star\ge z_0$. Finally, applying
the bound~\eqref{eq:wstar} and the definition of $\Delta$, we complete Step 2:
\[
   z_w-z^\star
   \le\frac{\Delta(2-z_0)^3}{\beta}=\alpha
\enspace.
\qedhere
\]
\end{proof}

% !TEX root = paper.tex

\subsection{Proof of Proposition~\ref{prop:lb_optimistic}}
\label{app:lb_optimistic}

\begin{proof}[\pfref{prop:lb_optimistic}]
Consider the following contextual bandit instance:
\begin{itemize}
\item Two actions $a_{g}$ and $a_{b}$, so $K=2$.
\item $r_{t}(a_g)=1-\eps$ and $r_{t}(a_b)=0$, regardless of context (there is no noise).
\item $N$ contexts $x^1, \ldots, x^N$. The context distribution $D_\cX$ is uniform over these $N$ contexts.
\item Regressor class $\cF$ contains the following $N+1$ predictors:
\begin{itemize}
\item Ground truth regressor $f^{\star}$ defined by $f^{\star}(x, a_g)=1-\eps,\;\forall{}x$ and $f^{\star}(x, a_b)=0,\;\forall{}x$.
\item For each $i\in\brk{N}$, $f_{i}$ satisfying $f_i(x^i, a_g)=0$, $f_i(x^i, a_b)=1$, and $f_i(x^j, a_g)=1-\eps$, $f_i(x^j, a_b)=0$ for all $j\neq{}i$.
\end{itemize}
\end{itemize}

We can see that $\pi_{f^{\star}}$ has population reward $1-\eps$ and each $\pi_{f_i}$ has population reward $(1-1/N)(1-\eps)$. Thus, each $f_i$ has expected regret of $(1-\eps)/N$.

Suppose $S$ is the set of contexts that have been observed by our algorithms at time $t$, and further assume $\beta_m=0$ (as it will be clear that larger $\beta_m$ can only make things worse), so that only regressors with zero square loss are considered.
Observe that $f^{\star}\in\cF_{m}$ and $f_i\in\cF_{m}$ only if $x^i\notin S$.

Let $x^i$ be the context observed at time $t$. If $x^i\in S$, then all regressors in $\cF_{m}$ agree on it, so $a_g$ will be played. Now, suppose $x^i\notin S$. Then we have $\High_{\cF_{m}}(x^i, a_g) = 1-\eps$ (obtained by $f^\star$), and $\Low_{\cF_{m}}(x^i, a_g) = 0$ (obtained by $f_i$). Likewise, $\High_{\cF_{m}}(x^i, a_b) = 1$ (from $f_i$) and $\Low_{\cF_{m}}(x^i, a_b) = 0$ (from $f^\star$).

We thus see that our algorithms will make a mistake and incur instantaneous regret of $(1-\eps)$ precisely at the time steps for which one of the $N$ contexts is encountered for the first time. The regret of the algorithm after $t$ steps can therefore be lower bounded as $\min\crl*{N, \wt{\Omega}(t)}$.
\end{proof}

\subsection{Proofs from \pref{sec:disagreement}}
\label{app:disagreement}

We recall our earlier definition of the disagreement coefficient
for the reader's convenience.

\begin{definition*}[Disagreement Coefficient]
The disagreement coefficient for $\cF$
(with respect to $D_\cX$)
is defined as
\begin{align*}
\theta_0\ldef\sup_{\delta>0, \veps>0}\;
\frac{\delta}{\veps}\Pr_{D_{\cX}}
  &\BigBracks{
      x\in\mathrm{Dis}(\cF(\veps))\textnormal{ and }
     \exists{}a\in{}A_{\cF(\veps)}(x) : \Wid_{\cF(\veps)}(x, a) > \delta
  }
.
\end{align*}
\end{definition*}

In addition, the following condition on $f^{\star}$ is important to
obtain fast rates, but it is not stated as an assumption because it
is not strictly necessary for any of our algorithms.
\begin{definition}[Massart noise condition]
\label{def:massart}
The distribution $D$ satisfies the Massart noise condition if there
exists $\gamma>0$, called a \emph{margin}, such that
\[
f^{\star}(x, \pi^\star(x))
\ge f^{\star}(x,a)  + \gamma
\quad
\text{for all $x$ and $a\ne\pi^\star(x)$.}
\]
\end{definition}

For all subsequent analyses we will use the following filtration:
\[
\cJ_t \ldef \sigma\prn*{
(x_1, a_1, r_1), \ldots,
(x_{t-1}, a_{t-1}, r_{t-1})
}.
\]
Let $\En_{t}\brk*{\cdot}\ldef\En\brk*{\cdot\mid{}\cJ_t}$ and $\mathrm{Var}_{t}\brk*{\cdot}\ldef\mathrm{Var}\brk*{\cdot\mid{}\cJ_t}$.

\begin{lemma}[Freedman-type inequality e.g. \cite{agarwal2014taming}]
\label{lem:freedman}
For any real-valued martingale difference sequence $(Z_t)_{t\leq{}T}$ with $\abs{Z_t}\leq{}R$ almost surely, it holds that with probability at least $1-\delta$,
\begin{equation}
\label{eq:freedman}
\sum_{t=1}^{T}Z_t \leq{} \eta(e-2)\sum_{t=1}^{T}\En_{t}\prn*{Z_t}^{2} + \frac{R\log(1/\delta)}{\eta}
\end{equation}
for all $\eta\in\brk*{0, 1/R}$.
\end{lemma}

Recall that epoch schedule used by \pref{alg:regression_ucb_elim} and \pref{alg:regression_ucb_optimistic} is $\tau_m = 2^{m-1}$. Denote the length of epoch $m$ by $T_m = \tau_{m+1} - \tau_m = 2^{m-1}$.
In addition, we will use the notation $g^{\star}_{a}(x)\ldef f^{\star}(x,a)$ where $f^{\star}$ as in the main text is the predictor that realizes the mean reward function, and also
\[
M_{t}(g, a) = ((g(x_t) - r_{t}(a))^{2} - (g_a^{\star}(x_t) - r_{t}(a))^{2})\ind\crl*{a = a_t}.
\]
for any $g:\cX\to\brk*{0,1}$, and action $a\in\cA$.
When $f\in\cF=\cG^{\cA}$ we will overload this notation by writing $M_{t}(f,a)\ldef{}M_{t}(f(\cdot,a), a)$.
Also define the class
\[
\tcG_m(\beta,a) = \crl*{g\in\cG\mid{}\frac{1}{\tau_m-1}\sum_{t=1}^{\tau_m-1}\En_{t}\brk*{M_{t}(g, a)}\leq{}\beta}.
\]

To prove the theorem, we make use of following lemmas.
\begin{lemma}\label{lem:expectation_and_variance_of_M_t}
  For any $g:\cX\to\brk*{0,1}$ and $a\in\cA$ we have
\begin{align*}
\En_{t}\brk*{M_{t}(g,a)} &= \En_t\brk*{(g(x_t) - g_a^{\star}(x_t))^{2}\ind\crl*{a=a_t}}, \\
\mathrm{Var}_{t}\brk*{M_t(g, a)} &\leq{} 4\En_{t}\brk*{M_t(g, a)}.
\end{align*}
\end{lemma}

\begin{proof}
Note that $a_t$ and $r_t$ are conditionally independent given $x_t$ and also $\En_{r_t}[r_t(a) \mid{} x_t] = g_a^{\star}(x_t)$. We thus have
\[
\En_{t}\brk*{M_{t}(g,a)} = \En_t\brk*{(g(x_t) - g_a^{\star}(x_t))((g(x_t) + g_a^{\star}(x_t) - 2r_t(a))\ind\crl*{a=a_t}}
= \En_t\brk*{(g(x_t) - g_a^{\star}(x_t))^{2}\ind\crl*{a=a_t}}.
\]
Similarly, since $((g(x_t) + g_a^{\star}(x_t) - 2r_t(a))^2 \leq 4$ we have
\[
\mathrm{Var}_{t}\brk*{M_t(g, a)} \leq \En_{t}\brk*{M_{t}(g,a)^2} \leq 4\En_t\brk*{(g(x_t) - g_a^{\star}(x_t))^2\ind\crl*{a=a_t}} = 4 \En_{t}\brk*{M_t(g, a)}.
\]
\end{proof}

\begin{definition}[Covering number]
  For a class $\cG'\subseteq{}\crl*{g:\cX\to\brk*{0,1}}$, an empirical $L_{p}$-cover on a sequence $x_{1},\ldots,x_{T}$ at scale $\veps$ is a set $V\subseteq{}\bbR^{T}$ such that
  \[
    \forall{}g\in\cG'\;\exists{}v\in{}V\text{ s.t. } \prn*{\frac{1}{T}\sum_{t=1}^{T}(g(x_t) - v_t)^{p}}^{1/p} \leq{} \veps.
  \]
  We define the covering number $\cN_{p}(\cG', \veps, x_{1:T})$ to be the size of the smallest such cover.

\end{definition}

\begin{lemma}
  \label{lem:conc_g}
  For any fixed class $\cG'\subseteq{}\crl*{g:\cX\to\brk*{0,1}}$ and fixed $a\in\cA$, with probability at least $1-\delta$, it holds that
  \begin{equation}
    \label{eq:conc_finite}
    \sum_{t=\tau_1}^{\tau_2}\En_{t}\brk*{M_{t}(g,a)} \leq{} 2\sum_{t=\tau_1}^{\tau_2}M_{t}(g,a) + 
    16\log\left(\frac{|\cG'|T^2}{\delta}\right)
  \end{equation}
  for all $\tau_1\leq{}\tau_2$ and $g\in\cG'$ when $\cG'$ is finite and
  \begin{equation}
    \label{eq:conc_parametric}
    \sum_{t=\tau_1}^{\tau_2}\En_{t}\brk*{M_{t}(g,a)} \leq{} 2\sum_{t=\tau_1}^{\tau_2}M_{t}(g,a) +
    \inf_{\veps>0}\crl*{100\veps{}T + 320\log\prn*{\frac{4\En_{x_{1:T}}\cN_{1}(\cG', \veps, x_{1:T})T^{2}\log(T)}{\delta}}}
  \end{equation}
  for all $\tau_1\leq{}\tau_2$ and $g\in\cG'$ in the general case.
\end{lemma}
\begin{remark}
  Equation \pref{eq:conc_parametric} implies \pref{eq:conc_finite}, but with weaker constants.
\end{remark}
\begin{corollary}
  \label{corr:disagreement_conc}
  Define
  \[
    \concG= \min\crl*{16\log\left(\frac{2|\cG|KT^2}{\delta}\right),
      \inf_{\veps>0}\crl*{100\veps{}T + 320\log\prn*{\frac{8\En_{x_{1:T}}\cN_{1}(\cG, \veps, x_{1:T})KT^{2}\log(T)}{\delta}}}
      }.
  \]
With probability at least $1-\delta/2$, it holds that
\begin{equation}
\label{eq:conc_lower}
\sum_{t=\tau_1}^{\tau_2}\En_{t}\brk*{M_{t}(g,a)} \leq{} 2\sum_{t=\tau_1}^{\tau_2}M_{t}(g,a) + \concG,
\end{equation}
for all $g\in\cG$, $a\in\cA$, and $\tau_1,\tau_2\in\brk*{T}$.

\end{corollary}
\begin{proof}[\pfref{lem:conc_g}]
We first prove the inequality in the finite class case.

For any fixed $g\in\cG'$, $a\in\cA$, and $\tau_1,\tau_2\in\brk*{T}$,
since $Z_t = \En_{t}\brk*{M_{t}(g,a)} - M_{t}(g,a)$ forms a martingale different sequence with $|Z_t| \leq 1$,
applying \pref{lem:freedman} and \pref{lem:expectation_and_variance_of_M_t} we have with probability $1 - \delta$,
\[
\sum_{t=\tau_1}^{\tau_2} (\En_{t}\brk*{M_{t}(g,a)} - M_{t}(g,a)) \leq{} 4\eta(e-2)\sum_{t=\tau_1}^{\tau_2}\En_{t}\brk*{M_{t}(g,a)} + \frac{1}{\eta}\log\prn*{\frac{1}{\delta}}.
\]
This implies
\[
  \sum_{t=\tau_1}^{\tau_2}\En_{t}\brk*{M_{t}(g,a)} \leq{} 2\sum_{t=\tau_1}^{\tau_2}M_{t}(g,a) + 
  16\log\left(\frac{1}{\delta}\right)
\]
after setting $\eta=1/8$ and rearranging. Finally, we apply a union bound over all $g\in\cG'$ and $\tau\leq{}\tau_2\in\brk*{T}$ to get the result.

For the infinite class case, we appeal to Theorem 9 of \cite{krishnamurthy2017active} (see page 36 specifically --- we do not use the final theorem statement but rather an intermediate result that is the consequence of their Lemmas 7, 8, 9, and 10).

Let $\tau_{1}$ and $\tau_{2}$ be fixed. Then the result of \cite{krishnamurthy2017active} implies that for any class $\cG$, any fixed $\veps>0$, $\nu>0$ and $a\in\cA$, letting $c=1/8$,
\[
  \Pr\prn*{
    \sup_{g\in\cG}\crl*{\sum_{t=\tau_{1}}^{\tau_{2}}\frac{1}{2}\En_{t}\brk*{M_{t}(g,a)} - M_{t}(g,a)} > 4\nu + 16T(1+c)\veps
  }
  \leq{} 4\En_{x_{1:T}}\cN_{1}(\cG, \veps, x_{1:T})\exp\prn*{-\frac{2c}{\prn*{3+c}^{2}}\nu}.
\]
Rearranging, this implies that with probability at least $1-\delta$,
\begin{equation}
  \label{eq:l1_high_prob}
  \sup_{g\in\cG}\crl*{\sum_{t=\tau_{1}}^{\tau_{2}}\frac{1}{2}\En_{t}\brk*{M_{t}(g,a)} - M_{t}(g,a)}
  \leq{} 18\veps{}T + 160\log\prn*{4\En_{x_{1:T}}\cN_{1}(\cG, \veps, x_{1:T})/\delta}.
\end{equation}
Now consider a grid $\veps_{i}\ldef{}e^{i}/T$ for $i\in\brk{\log(T)}$. By union bound, \pref{eq:l1_high_prob} implies that with probability at least $1-\delta$,
\[
  \sup_{g\in\cG}\crl*{\sum_{t=\tau_{1}}^{\tau_{2}}\frac{1}{2}\En_{t}\brk*{M_{t}(g,a)} - M_{t}(g,a)}
  \leq{} 18\veps{}_iT + 160\log\prn*{4\En_{x_{1:T}}\cN_{1}(\cG, \veps_i, x_{1:T})\log(T)/\delta}\quad\forall{}i\in\brk*{\log(T)}.
\]
This implies that with probability at least $1-\delta$,
\[
  \sup_{g\in\cG}\crl*{\sum_{t=\tau_{1}}^{\tau_{2}}\frac{1}{2}\En_{t}\brk*{M_{t}(g,a)} - M_{t}(g,a)}
  \leq{} \inf_{\veps>0}\crl*{50\veps{}T + 160\log\prn*{4\En_{x_{1:T}}\cN_{1}(\cG, \veps, x_{1:T})\log(T)/\delta}}.
\]
To see that this inequality is implied by the preceeding inequality, first observe that the infimum over $\veps$ above may be restricted to $\brk*{1/T, 1}$ without loss of generality. This holds because $M_{t}$ lies in $\brk*{-1, 1}$ and $\cN_{1}(\cG, 1, x_{1:T})\leq{}1$, which both follow from the fact that the range of $\cG$ lies in $\brk*{0,1}$. Now let $\veps^{\star}$ obtain the infimum and let $i^{\star}=\min\crl*{i\mid{}\veps_{i}\geq{}\veps^{\star}}$. Then $\cN_{1}(\cG, \veps_{i^{\star}}, x_{1:T})\leq{}\cN_{1}(\cG, \veps^{\star}, x_{1:T})$ and $18\veps_{i^{\star}}T\leq{} 18e\veps^{\star}T \leq{} 50\veps^{\star}T$.

To conclude, we take a union bound over all $\tau_1<\tau_{2}\in\brk*{T}$.
\end{proof}

\begin{lemma}
\label{lem:disagreement_containment}
Conditioned on the event of \pref{corr:disagreement_conc}, it holds that
\begin{enumerate}
\item $g^{\star}_a\in\hcG_m\left(\frac{\concG}{2(\tau_{m}-1)}, a\right)$ for all $m\in\brk{M}$ and $a\in\cA$.

\item For all $\beta\geq{}0$, $m\in\brk{M}$, and $a\in\cA$,
\[
\hcG_m(\beta, a) \subseteq{}\tcG_m\left(2\beta + \frac{\concG}{\tau_m-1}, a\right).
\]

\item For all $\beta\geq{}0$, $m\in\brk{M}$, $k\in\brk{m}$, and $a\in\cA$,
\[
\hcG_m(\beta, a)\subseteq\hcG_{k}\left(\frac{\tau_m - 1}{\tau_k - 1}\beta + \frac{\concG}{\tau_k-1}, a\right).
\]

\item With $\beta_m = \frac{(M-m+1)\concG}{\tau_m - 1}$, we have for any $m\in [M]$, $f^\star \in \cF_m$
and also $\cF_m \subseteq{} \cF_{m-1} \subseteq{} \cdots \subseteq{} \cF_1$.
\end{enumerate}
\end{lemma}

\begin{proof}
Each claim in the lemma statement will be handled separately.~\\
\textbf{First claim.} From \pref{eq:conc_lower} and nonnegativity of $\En_{t}\brk*{M_t(g,a)}$, we have that
\[
\min_{g\in\cG}\crl*{2\sum_{t=1}^{\tau_m-1}M_{t}(g,a)} + \concG \geq{} 0.
\]
Expanding out $M_{t}(g,a)$ and rearranging, this gives $ \hcR_m(g^\star_a, a) - \min_{g\in\cG}\hcR_m(g, a) \leq \frac{\concG}{2(\tau_{m}-1)}$, which implies $g^{\star}_a\in\hcG_m\left(\frac{\concG}{2(\tau_{m}-1)}, a\right)$.

\textbf{Second claim.}
For any $g \in \hcG_m(\beta, a)$, we have by definition
\begin{equation}\label{eq:empirical_tcG_ball}
\frac{1}{\tau_m - 1}\sum_{t=1}^{\tau_m-1} M_{t}(g,a) = \hcR_m(g, a) - \hcR_m(g^\star_a, a) \leq \hcR_m(g, a) - \min_{g' \in \cG} \hcR_m(g', a) \leq \beta.
\end{equation}
Therefore applying \pref{eq:conc_lower} leads to
\[
\frac{1}{\tau_m - 1} \sum_{t=1}^{\tau_m-1} \En_t\brk*{M_{t}(g,a)} \leq \frac{2}{\tau_m - 1} \sum_{t=1}^{\tau_m-1}  M_{t}(g,a) + \frac{\concG}{\tau_m - 1}
\leq 2 \beta + \frac{\concG}{\tau_m - 1},
\]
which implies $g\in\tcG_m\left(2\beta + \frac{\concG}{\tau_m-1}, a\right)$.

\textbf{Third claim.}
For any $g \in \hcG_m(\beta, a)$, we have for any $k \in [m]$,
\begin{align*}
(\tau_{k}-1) \left(\hcR_k(g, a) - \min_{g'\in\cG} \hcR_k(g', a) \right)
&\leq (\tau_{k}-1) \left(\hcR_k(g, a) - \hcR_k(g^\star_a, a)  \right) + \concG/2 \tag{by the first claim} \\
&= \sum_{t=1}^{\tau_m-1} M_{t}(g,a)  - \sum_{t=\tau_k}^{\tau_m - 1} M_t(g,a) + \concG/2 \\
&\leq (\tau_m-1)\beta  - \frac{\sum_{t=\tau_k}^{\tau_m - 1} \En_t\brk*{M_t(g,a)}}{2} + \concG \tag{by \pref{eq:empirical_tcG_ball} and \pref{eq:conc_lower}} \\
&\leq (\tau_m-1)\beta + \concG \tag{by nonnegativity of $\En_t\brk*{M_t(g,a)}$},
\end{align*}
which implies $g\in\hcG_{k}\left(\frac{\tau_m - 1}{\tau_k - 1}\beta + \frac{\concG}{\tau_k-1}, a\right)$.

\textbf{Fourth claim.}
The value of $\beta_m$ ensures that
$\frac{\concG}{2(\tau_{m}-1)} \leq \beta_m$ for any $m \in [M]$, and also for any $k < m$,
\[
\frac{\tau_m - 1}{\tau_k - 1}\beta_m + \frac{\concG}{\tau_k - 1} = \frac{(M-m+2)\concG}{\tau_k - 1} \leq \beta_k.
\]
Therefore by the first and the third statement we have the claimed conclusions.
\end{proof}

\begin{proposition}
\label{prop:confused_set}
For any two classes $\cF, \cF'$ and any context $x$, $A_{\cF}(x)\subseteq{}A_{\cF'}(x)$.
\end{proposition}

\begin{lemma}
\label{lem:disagreement_set}
\pref{alg:regression_ucb_elim} with \textsc{Option I} ensures that for any $m\in[M]$ and $t\in\{\tau_m,\ldots,\tau_{m+1}-1\}$,
\[
A_{t} = \cA_{\cF_m}(x_t) = \bigcup_{f\in\cF_m} \argmax_{a\in \cA} f(x_t, a).
\]
\end{lemma}

\begin{proof}
For any $f \in \cF_m$ and any $a \in  \argmax_{a'\in \cA} f(x_t, a')$, we have by definitions
\[
\High_{\cF_m}(x_t, a) \geq f(x_t, a) = \max_{a'} f(x_t, a') \geq \max_{a'} \Low_{\cF_m}(x_t, a'),
\]
which implies $a \in A_t$.
On the other hand, for each $a \in A_t$, there exists $g_a \in \hcG(\beta_m, a)$ such that $g_a(x_t) \geq \max_{a'} \min_{g \in \hcG(\beta_m, a')} g(x_t)$,
which further implies that for any $a' \neq a$, there exists $g_{a'} \in \hcG(\beta_m, a')$ such that $g_a(x_t) \geq g_{a'}(x_t)$.
Therefore, we can construct an $f$ so that $f(\cdot, a) = g_a(\cdot)$ and $f(\cdot, a') = g_{a'}(\cdot)$ for all $a' \neq a$,
so that clearly $f \in \cF_m$ and $a \in \argmax_{a'\in \cA} f(x_t, a')$.
This proves the lemma.
\end{proof}

\begin{lemma}
\label{lem:risk_query_rule}
Conditioned on the event of \pref{corr:disagreement_conc}, \pref{alg:regression_ucb_elim} with \textsc{Option I} and $\beta_m = \frac{(M-m+1)\concG}{\tau_m - 1}$ ensures that for any $m\in [M]$,
we have $\cF_m \subseteq{} \cF(\veps_m)$ with
\[
\veps_m = \inf_{\eta>0}\crl*{\eta{}P_{\eta} +  \frac{4K^2}{\eta(\tau_m - 1)}(2M - 2m + 3)\concG},
\]
where $P_{\eta}=\Pr_{x}\prn*{ f^{\star}(x, \pi^{\star}(x)) - \max_{a\neq{}\pi^{\star}(x)}f^{\star}(x, a) < \eta}$.
\end{lemma}

\begin{proof}
We first prove that for any $t < \tau_m$ and $f \in \cF_m$, the following holds
\begin{equation}\label{eq:risk_before_averaging}
\En_{x, r}\brk*{r(\pi^{\star}(x)) - r(\pi_{f}(x))}
\leq{} \inf_{\eta>0}\crl*{\eta{}P_{\eta} +  \frac{4K}{\eta}\sum_{a\in\cA}\En_{t}\brk*{M_t(f,a)}}.
\end{equation}

Indeed, note that for any $\eta > 0$, with realizability we have 
\begin{align*}
&\En_{x, r}\brk*{r(\pi^{\star}(x)) - r(\pi_{f}(x))} \\
&= \En_x\brk*{f^{\star}(x, \pi^{\star}(x)) - f^{\star}(x, \pi_{f}(x))} \\
&\leq{} \eta{}\Pr_{x}\prn*{ f^{\star}(x, \pi^{\star}(x)) - f^{\star}(x, \pi_{f}(x)) < \eta \text{ and } \pi^{\star}(x) \neq \pi_f(x)}
  + \frac{1}{\eta}\En_x\prn*{f^{\star}(x, \pi^{\star}(x)) - f^{\star}(x, \pi_{f}(x))}^{2} \\
&\leq{} \eta{}P_{\eta}  + \frac{1}{\eta}\En_x\prn*{f^{\star}(x, \pi^{\star}(x)) - f^{\star}(x, \pi_{f}(x))}^{2}.
\end{align*}
By the definition of $\pi_f$ we also have for any $x$, $f(x, \pi_{f}(x)) - f(x, \pi^{\star}(x))\geq{}0$ and thus
\begin{align*}
\En_x\prn*{f^{\star}(x, \pi^{\star}(x)) - f^{\star}(x, \pi_{f}(x))}^{2}
&\leq{} \En_x\prn*{f^{\star}(x, \pi^{\star}(x)) - f^{\star}(x, \pi_{f}(x)) + f(x, \pi_{f}(x)) - f(x, \pi^{\star}(x))}^{2} \\
&\leq{} 2\En_x\prn*{f^{\star}(x, \pi^{\star}(x)) -f(x, \pi^{\star}(x))}^{2} + 2\En_x\prn*{ f(x, \pi_{f}(x)) - f^{\star}(x, \pi_{f}(x))}^{2}.
\end{align*}
Now suppose round $t$ is in epoch $k$.
Since both $f\in\cF_m \subseteq \cF_k$ and $f^{\star}\in\cF_k$, we have $\pi_{f}(x_t),\pi^{\star}(x_t)\in{}A_t$ by \pref{lem:disagreement_set}.
Therefore, the fact that $a_{t}$ is drawn uniformly from $A_{t}$ implies
\[
\En_x\prn*{f^{\star}(x, \pi^{\star}(x)) -f(x, \pi^{\star}(x))}^{2} \leq{} K\En_{x, a_t}\prn*{f^{\star}(x, a_t) -f(x, a_t)}^{2},
\]
and likewise
\[
\En_x\prn*{f^{\star}(x, \pi_f(x)) -f(x, \pi_{f}(x))}^{2} \leq{} K\En_{x, a_t}\prn*{f^{\star}(x, a_t) -f(x, a_t)}^{2}.
\]
Lastly, plugging the equality
\[
\En_{x, a_t}\prn*{f^{\star}(x_t, a_t) -f(x_t, a_t)}^{2} = \sum_{a\in\cA}\En_{t}\brk*{M_t(f,a)}
\]
proves Eq.~\pref{eq:risk_before_averaging}. Averaging over $t = 1, \ldots, \tau_m-1$ then gives
\[
\En_{x, r}\brk*{r(\pi^{\star}(x)) - r(\pi_{f}(x))}
\leq{} \inf_{\eta>0}\crl*{\eta{}P_{\eta} +  \frac{4K}{\eta(\tau_m - 1)}\sum_{a\in\cA}\sum_{t=1}^{\tau_m - 1}\En_{t}\brk*{M_t(f,a)}}.
\]
Using the second statement of \pref{lem:disagreement_containment} we have $\sum_{t=1}^{\tau_m - 1}\En_{t}\brk*{M_t(f,a)} \leq 2(\tau_m-1)\beta_m + \concG
= (2M - 2m + 3)\concG$
and thus
\[
\En_{x, r}\brk*{r(\pi^{\star}(x)) - r(\pi_{f}(x))}
\leq{} \inf_{\eta>0}\crl*{\eta{}P_{\eta} +  \frac{4K^2}{\eta(\tau_m - 1)}(2M - 2m + 3)\concG} = \veps_m,
\]
completing the proof by the definition of $\cF(\veps_m)$.
\end{proof}

We are now ready to prove \pref{thm:disagreement}, which is restated below with an extra result under the Massart condition.

\begin{theorem}[Full version of \pref{thm:disagreement}]
\label{thm:disagreement_with_massart}
With $\beta_m = \frac{(M-m+1)\concG}{\tau_m - 1}$ and $\concG$ as in \pref{corr:disagreement_conc}, \pref{alg:regression_ucb_elim} with \textrm{Option I} ensures that with probability at least $1- \delta$,
\begin{equation}
\label{eq:disagreement_slow}
\reg_T = O\left(T^{\frac{3}{4}}\concG^{\frac{1}{4}}\sqrt{\theta_0 K \log T} + \log(1/\delta)\right).
\end{equation}
In particular, for finite classes regret is bounded as $\tO\left(T^{\frac{3}{4}}\prn*{\log\abs*{\cG}}^{\frac{1}{4}}\sqrt{\theta_0 K}\right)$.\\
Furthermore, if the Massart noise condition (\pref{def:massart}) is satisfied with parameter $\gamma$, then \pref{alg:regression_ucb_elim} configured as above with $\delta=1/T$ enjoys an in-expectation regret bound of
\begin{equation}
\label{eq:disagreement_fast}
\En\brk*{\sum_{t=1}^{T}r_{t}(\pi^\star(x_t))- \sum_{t=1}^{T}r_{t}(a_t)} = O\left(\frac{\theta_0 K^2 C_{1/T} \log^2 T}{\gamma^2}\right),
\end{equation}
which for finite classes is upper bounded by $\tO\left(\frac{\theta_0 K^2\log\prn*{\abs*{\cG}T}}{\gamma^2}\right)$.
\end{theorem}
\begin{remark}
This theorem and the subsequent regret bounds based on moment conditions (\pref{thm:moment_with_massart} and \pref{thm:moment_optimistic_with_massart}) give a high-probability empirical regret bound in the general case, but only give an in-expectation regret bound under the Massart condition. This is because one incurs an extra $O(\sqrt{T})$ factor in going from a (conditional) expected regret bound to an empirical regret bound, which is a low order term in the general case but may dominate in the Massart case.
\end{remark}

\begin{proof}
We will first provide a bound on
\[
\sum_{t=1}^T \En_t\brk*{r_t(\pi^\star(x_t)) - r_t(a_t)},
\]
then relate this quantity to the left-hand-side of \pref{eq:disagreement_slow} and \pref{eq:disagreement_fast} at the end.

This proof conditions on the above event and the events of \pref{corr:disagreement_conc}, which happen with probability at least $1-\delta/2$,
and bounds the conditional expected regret terms $\En_t\brk*{f^\star(x_t, \pi^\star(x_t)) - f^\star(x_t, a_t)}$ individually.

For any $\eta' > 0$, we recall the definition used in the proof of \pref{lem:risk_query_rule}:
$P_{\eta'}=\Pr_{x}\prn*{ f^{\star}(x, \pi^{\star}(x)) - \max_{a\neq{}\pi^{\star}(x)}f^{\star}(x, a) < \eta'}$.
Further define two events:
\begin{align*}
E_1 &= \{\exists a \in A_t : f^\star(x, a) < f^\star(x, \pi^\star(x))\} \\
E_2 &= \{\exists a \in A_t : f^{\star}(x_t, \pi^{\star}(x_t)) - f^{\star}(x_t, a) \geq \eta'\}.
\end{align*}
We then have
\begin{align*}
\En_t\brk*{f^\star(x_t, \pi^\star(x_t)) - f^\star(x_t, a_t)}
&= \En_t\brk*{f^\star(x_t, \pi^\star(x_t)) - f^\star(x_t, a_t) \mid{} E_1} \Pr_{x_t}(E_1)  \\
&= \En_t\brk*{f^\star(x_t, \pi^\star(x_t)) - f^\star(x_t, a_t) \mid{} E_1, \neg E_2} \Pr_{x_t}(E_1, \neg E_2)  \\
&\quad + \En_t\brk*{f^\star(x_t, \pi^\star(x_t)) - f^\star(x_t, a_t) \mid{} E_1, E_2} \Pr_{x_t}(E_1, E_2) \\
 &\leq \eta' \Pr_{x_t}(E_1, \neg E_2)  +  \Pr_{x_t}(E_1, E_2) \\
 &\leq \eta' P_\eta' +  \Pr_{x_t}(E_1, E_2).
\end{align*}
Next we argue two facts (suppose round $t$ is in epoch $m$): $E_1$ implies $x_t \in \mathrm{Dis}(\cF_m)$,
and $E_2$ implies that there exists $a' \in A_t$ such that $\Wid_{\cF_m}(x, a') > \eta'/2$.
Indeed, with $a$ being the action stated in event $E_1$,
we know that by \pref{lem:disagreement_set} there exists $f \in \cF_m$ such that $a \in \argmax_{a'} f(x_t, a)$.
However, clearly $a$ is not in $\argmax_{a'} f^\star(x_t, a)$, and thus by $f^\star \in \cF_m$ and the definition of disagreement region we have $x_t \in \mathrm{Dis}(\cF_m)$.
On the other hand, with $a$ being the action stated in event $E_2$, we have
\begin{align*}
\eta' &\leq f^{\star}(x_t, \pi^{\star}(x_t)) - f^{\star}(x_t, a) \\
&\leq \High_{\cF_m}(x_t,  \pi^{\star}(x_t)) - \Low_{\cF_m}(x_t, a) \\
&\leq \High_{\cF_m}(x_t,  \pi^{\star}(x_t)) - \Low_{\cF_m}(x_t, \pi^{\star}(x_t)) + \High_{\cF_m}(x_t, a) - \Low_{\cF_m}(x_t, a) \\
&= \Wid_{\cF_m}(x_t,  \pi^{\star}(x_t)) + \Wid_{\cF_m}(x_t, a)
\end{align*}
where the last inequality is by the fact $a \in A_t$ and the definition of $A_t$.
The last inequality thus implies that there exists $a' \in A_t$ such that $\Wid_{\cF_m}(x, a') > \eta'/2$.
We therefore continue with
\begin{align*}
\Pr_{x_t}(E_1, E_2) &\leq \Pr_{x_t}(x_t \in \mathrm{Dis}(\cF_m) \text{ and } \exists a \in A_t : \Wid_{\cF_m}(x, a) > \eta'/2) \\
&\leq \Pr_{x_t}(x_t \in \mathrm{Dis}(\cF_m) \text{ and } \exists a \in A_{\cF_m}(x_t) : \Wid_{\cF_m}(x, a) > \eta'/2) \\
&\leq \Pr_{x_t}(x_t \in \mathrm{Dis}(\cF(\veps_m)) \text{ and } \exists a \in A_{\cF(\veps_m)}(x_t) : \Wid_{\cF(\veps_m)}(x, a) > \eta'/2) \tag{by \pref{lem:risk_query_rule} and \pref{prop:confused_set}}\\
&\leq \frac{2\theta_0\veps_m }{\eta'}. \tag{by the definition of $\theta_0$}
\end{align*}
Combining everything we arrive at for any $\eta, \eta' > 0$,
\begin{align*}
\sum_{t=1}^T  \En_t\brk*{f^\star(x_t, \pi^\star(x_t)) - f^\star(x_t, a_t)}
&\leq  \eta' T P_\eta' + \frac{2\theta_0 }{\eta'}\left(\eta T P_\eta +  \frac{4K^2\concG}{\eta}\sum_{m=1}^M \frac{T_m(2M-2m+3)}{\tau_m - 1} \right)\\
&\leq  \eta' T P_\eta' + \frac{2\theta_0 }{\eta'}\left(\eta T P_\eta +  \frac{8K^2\concG}{\eta} (M^2 + 2M) \right). \\
\end{align*}
In the general case we simply bound $P_\eta$ and $P_{\eta'}$ by $1$ and choose the optimal $\eta$ and $\eta'$ to arrive at a regret bound of order
$O\left(T^{\frac{3}{4}}\concG^{\frac{1}{4}}\sqrt{\theta_0 K \log T} + \log(1/\delta)\right)$.
On the other hand, under the Massart condition (\pref{def:massart}) one can pick $\eta = \eta' = \gamma$ so that $P_\eta = P_{\eta'} = 0$ and obtain
a regret bound of order $O\left(\frac{\theta_0 K^2 \concG \log^2 T}{\gamma^2}\right)$.

Lastly, we relate the sum of conditional expected instantaneous regrets to the left-hand side of \pref{eq:disagreement_slow} and \pref{eq:disagreement_fast}. In the general case, since instantaneous regret lies in $\brk*{-1,1}$, Azuma-Hoeffding implies that
\begin{align*}
\sum_{t=1}^T r_t(\pi^\star(x_t)) - r_t(a_t)
&\leq \sum_{t=1}^T \En_t\brk*{r_t(\pi^\star(x_t)) - r_t(a_t)} + O(\sqrt{T\log(1/\delta)})
\end{align*}
with probability at least $1-\delta/2$. By union bound, the theorem statement holds with probability at least $1-\delta$.

In the Massart case, the law of total expectation implies
\[
\En\brk*{\sum_{t=1}^T r_t(\pi^\star(x_t)) - r_t(a_t)} \leq{} O\left(\frac{\theta_0 K^2 C_{1/T} \log^2 T}{\gamma^2}\right) + \frac{1}{T}\cdot{}T,
\]
where the second term uses boundedness of regret along with the fact that the events of \pref{corr:disagreement_conc} hold with probability at least $1-1/T$.
\end{proof}

\subsection{Proofs from \pref{sec:moment}}
\label{app:moment_proofs}
Similarly to the notation $M_t(g,a)$ for the case $\cF = \cG^\cA$, for a general predictor class $\cF$ we define for any $f \in \cF$
\[
M_{t}(f) = (f(x_t, a_t) - r_{t}(a_t))^{2} - (f^{\star}(x_t, a_t) - r_{t}(a_t))^{2}.
\]
and also the class
\[
\tcF_m(\beta) = \crl*{f\in\cF\mid{}\frac{1}{\tau_m-1}\sum_{t=1}^{\tau_m-1}\En_{t}\brk*{M_{t}(f)}\leq{}\beta}.
\]
Finally, for any $a\in\cA$ we define a class
\[
\cF|_a = \crl*{x\mapsto{}f(x,a)\mid{}f\in\cF}.
\]

We establish several lemmas similar to those in \pref{app:disagreement}.

\begin{lemma}\label{lem:expectation_and_variance_of_M_t(f)}
For any $f\in\cF$ we have
\begin{align*}
\En_{t}\brk*{M_{t}(f)} &= \En_t\brk*{(f(x_t, a_t) - f^{\star}(x_t, a_t))^{2}}, \\
\mathrm{Var}_{t}\brk*{M_t(f)} &\leq{} 4\En_{t}\brk*{M_t(f)}.
\end{align*}
\end{lemma}

\begin{lemma}
  \label{lem:disagreement_conc_F}
  Define
  \begin{equation}
    \label{eq:conc_F}
    \concF = \min\crl*{16\log\prn*{\frac{2\abs*{\cF}T^{2}}{\delta}}, 
    \inf_{\veps>0}\crl*{100\veps{}KT + 320\sum_{a\in\cA}\log\prn*{\frac{8\En_{x_{1:T}}\cN_{1}(\cF|_{a}, \veps, x_{1:T})KT^{2}\log(T)}{\delta}}}
      }.
  \end{equation}

With probability at least $1-\delta/2$, it holds that
\begin{equation*}
\sum_{t=\tau_1}^{\tau_2}\En_{t}\brk*{M_{t}(f)} \leq{} 2\sum_{t=\tau_1}^{\tau_2}M_{t}(f) + \concF,
\end{equation*}
for all $f\in\cF$ and $\tau_1,\tau_2\in\brk*{T}$.
\end{lemma}
\begin{proof}[\pfref{lem:disagreement_conc_F}]
  We first prove the inequality in the finite class case. For any fixed $f\in\cF$, and $\tau_1\leq\tau_2\in\brk*{T}$,
 $Z_t \ldef \En_{t}\brk*{M_{t}(f)} - M_{t}(f)$ forms a martingale different sequence with $|Z_t| \leq 1$. Applying \pref{lem:freedman} and \pref{lem:expectation_and_variance_of_M_t(f)} we have with probability $1 - \delta$,
\[
\sum_{t=\tau_1}^{\tau_2} (\En_{t}\brk*{M_{t}(f)} - M_{t}(f)) \leq{} 4\eta(e-2)\sum_{t=\tau_1}^{\tau_2}\En_{t}\brk*{M_{t}(f)} + \frac{1}{\eta}\log\prn*{\frac{1}{\delta}}, 
\]
which implies after setting $\eta=1/8$ and rearranging.
\[
  \sum_{t=\tau_1}^{\tau_2}\En_{t}\brk*{M_{t}(f)} \leq{} 2\sum_{t=\tau_1}^{\tau_2}M_{t}(f) + 
  16\log\left(\frac{1}{\delta}\right)
\]
We apply a union bound over all $f\in\cF$ and $\tau_1\leq{}\tau_2\in\brk*{T}$ to get the result.

To handle the infinite class case we invoke \pref{lem:conc_g}. In particular, for any fixed $a$, the lemma with $\cG'=\cF|_{a}$ implies that with probability at least $1-\delta$,
\[
    \sum_{t=\tau_1}^{\tau_2}\En_{t}\brk*{M_{t}(f(\cdot,a),a)} \leq{} 2\sum_{t=\tau_1}^{\tau_2}M_{t}(f(\cdot,a),a) +
    \inf_{\veps>0}\crl*{100\veps{}T + 320\log\prn*{\frac{4\En_{x_{1:T}}\cN_{1}(\cF|_{a}, \veps, x_{1:T})T^{2}\log(T)}{\delta}}}
\]
for all $f\in\cF$ and $\tau_1\leq{}\tau_2$. Observe that $M_{t}(f)=\sum_{a\in\cA}M_{t}(f(\cdot,a),a)$. Taking a union bound and then summing over all actions, the preceding statement therefore implies that with probability at least $1-\delta$,
\[
    \sum_{t=\tau_1}^{\tau_2}\En_{t}\brk*{M_{t}(f)} \leq{} 2\sum_{t=\tau_1}^{\tau_2}M_{t}(f) +
    \sum_{a\in\cA}\inf_{\veps>0}\crl*{100\veps{}T + 320\log\prn*{\frac{4\En_{x_{1:T}}\cN_{1}(\cF|_{a}, \veps, x_{1:T})KT^{2}\log(T)}{\delta}}}
\]
for all $f\in\cF$ and $\tau_1\leq{}\tau_2$.  The final result follows from superadditivity of the infimum.
\end{proof}

\begin{lemma}
\label{lem:disagreement_containment_F}
Conditioned on the event of \pref{lem:disagreement_conc_F}, it holds that
\begin{enumerate}
\item $f^{\star} \in\hcF_m\left(\frac{\concF}{2(\tau_{m}-1)}\right)$ for all $m\in\brk{M}$.

\item For all $\beta\geq{}0$ and $m\in\brk{M}$,
\[
\hcF_m(\beta) \subseteq{}\tcF_m\left(2\beta + \frac{\concF}{\tau_m-1}\right).
\]
Consequently, we have $\En_{\tau_{m-1}}\brk*{M_{\tau_{m-1}}(f)} \leq \frac{2\beta(\tau_m-1) + \concF}{T_{m-1}}$ for any $f \in \hcF_m(\beta)$.

\item For all $\beta\geq{}0$, $m\in\brk{M}$, and $k\in\brk{m}$,
\[
\hcF_m(\beta)\subseteq\hcF_{k}\left(\frac{\tau_m - 1}{\tau_k - 1}\beta + \frac{\concF}{\tau_k-1}\right).
\]

\item With $\beta_m = \frac{(M-m+1)\concG}{\tau_m - 1}$, we have for any $m\in [M]$, $f^\star \in \cF_m$
and also $\cF_m \subseteq{} \cF_{m-1} \subseteq{} \cdots \subseteq{} \cF_1$.
\end{enumerate}
\end{lemma}
\begin{proof}
The proof of this lemma is essentially the same as that of \pref{lem:disagreement_containment} in \pref{app:disagreement}.
The only new statement is the second statement of the second claim in \pref{lem:disagreement_containment_F}.
This holds because for any $f \in \hcF_m(\beta) \subseteq{}\tcF_m\left(2\beta + \frac{\concF}{\tau_m-1}\right)$,
we have
\[
\sum_{t=\tau_{m-1}}^{\tau_m-1}\En_{t}\brk*{M_{t}(f)}  \leq \sum_{t=1}^{\tau_m-1}\En_{t}\brk*{M_{t}(f)} \leq 2\beta(\tau_m - 1) + \concF,
\]
and also by the epoch schedule of the algorithm $\En_{t}\brk*{M_{t}(f)}$ remains the same for all $t \in \{\tau_{m-1}, \ldots, \tau_m - 1\}$
and thus $T_{m-1} \En_{\tau_{m-1}}\brk*{M_{\tau_{m-1}}(f)} = \sum_{t=\tau_{m-1}}^{\tau_m-1}\En_{t}\brk*{M_{t}(f)}  \leq 2\beta(\tau_m - 1) + \concF$,
proving the statement.
\end{proof}

We are now ready to prove the main theorems, which are again restated with extra results under the Massart condition.

\begin{theorem}[Full version of \pref{thm:moment}]
\label{thm:moment_with_massart}
With $\beta_m = \frac{(M-m+1)\concF}{\tau_m - 1}$ and  $\concF$ as in \pref{lem:disagreement_conc_F},
\pref{alg:regression_ucb_elim} with \textrm{Option II} ensures that with probability at least $1- \delta$,
\[
\reg_T = O\left(\sqrt{T\consL_{2,0}\concF}\log T + \log(1/\delta)\right).
\]
In particular, for finite classes regret is bounded as $\tO\left(\sqrt{T\consL_{2,0}\log\abs*{\cF}}\right)$.

Furthermore, if the Massart noise condition \pref{def:massart} is satisfied with parameter $\gamma$, then \pref{alg:regression_ucb_elim} configured as above with $\delta=1/T$ enjoys an in-expectation regret bound of
\[
\En\brk*{\sum_{t=1}^T r_t(\pi^\star(x_t)) - r_t(a_t)} = O\left(\frac{\consL_{2,0} C'_{1/T} \log^2 T}{\gamma}\right),
\]
which for finite classes is bounded as $\tO\left(\frac{\consL_{2,0} \log\abs*{\cF}}{\gamma}\right)$.
\end{theorem}

\begin{proof}
Similar to the proof of \pref{thm:disagreement}, we condition on the events of \pref{lem:disagreement_conc_F}, which happen with probability at least $1-\delta/2$.

With $m$ denoting the epoch to which round $t$ belongs and $P_{\eta}=\Pr_{x}\prn*{ f^{\star}(x, \pi^{\star}(x)) - \max_{a\neq{}\pi^{\star}(x)}f^{\star}(x, a) < \eta}$
for any $\eta > 0$, we have
\begin{align*}
&\En_t\brk*{f^\star(x_t, \pi^\star(x_t)) - f^\star(x_t, a_t)} \\
&\leq \eta P_\eta + \frac{1}{\eta}\En_t\brk*{(f^\star(x_t, \pi^\star(x_t)) - f^\star(x_t, a_t))^2} \\
&\leq \eta P_\eta + \frac{1}{\eta}\En_t\brk*{(f^\star(x_t, \pi^\star(x_t))  - \Low_{\cF_m}(x_t, \pi^\star(x_t))
+ \High_{\cF_m}(x_t, a_t) - f^\star(x_t, a_t))^2)^2} \tag{$a_t \in A_t$} \\
&\leq \eta P_\eta + \frac{2}{\eta}\En_t\brk*{(f^\star(x_t, \pi^\star(x_t))  - \Low_{\cF_m}(x_t, \pi^\star(x_t)))^2}
+ \frac{2}{\eta}\En_t\brk*{(\High_{\cF_m}(x_t, a_t) - f^\star(x_t, a_t))^2} \\
&\leq \eta P_\eta + \frac{2}{\eta}\En_t\brk*{\sup_{f\in\cF_m}(f^\star(x_t, \pi^\star(x_t))  - f(x_t, \pi^\star(x_t)))^2}
                                                                                 + \frac{2}{\eta}\En_t\brk*{\sup_{f\in\cF_m}(f(x_t, a_t) - f^\star(x_t, a_t))^2} \\
  &\leq \eta P_\eta + \frac{4}{\eta}\sup_{x\in\cX,a\in\cA}\sup_{f\in\cF_m}\crl*{(f^\star(x, a)  - f(x, a))^2} \\
  &= \eta P_\eta + \frac{4}{\eta}\sup_{f\in\cF_m}\sup_{x\in\cX,a\in\cA}\crl*{(f^\star(x, a)  - f(x, a))^2} \\
&\leq \eta P_\eta + \frac{4\consL_{2,0}}{\eta}\sup_{f\in\cF_m}\En_{x\sim D_\cX}\En_{a\sim \Unif{\cA}}
    \brk*{\one\bigSet{x\in U_0(a)}(f^\star(x, a)  - f(x, a))^2} \tag{by \pref{def:l2norm_ua}} \\
    &\leq \eta P_\eta + \frac{4\consL_{2,0}}{\eta}\sup_{f\in\cF_m}\En_{x\sim D_\cX}\En_{a\sim \Unif{\cA}}
    \brk*{\one\bigSet{a \in A_{\tau_{m-1}}}(f^\star(x, a)  - f(x, a))^2} ,
\end{align*}
where the last step holds because $x \in U_0(a)$ along with the fact $f^\star \in \cF_{m-1}$ implies
\[
\High_{\cF_{m-1}}(x, a) \geq f^\star(x, a) = \max_{a'} f^\star(x, a') \geq \max_{a'} \Low_{\cF_{m-1}}(x, a'),
\]
and thus by definition $a \in A_{\tau_{m-1}}$.
We continue with
\begin{align*}
\En_t\brk*{f^\star(x_t, \pi^\star(x_t)) - f^\star(x_t, a_t)}
&\leq \eta P_\eta + \frac{4\consL_{2,0}}{\eta}\sup_{f\in\cF_m}\En_{x\sim D_\cX}
    \En_{a\sim \Unif{A_{\tau_{m-1}}}}\brk*{(f^\star(x, a)  - f(x, a))^2} \\
&= \eta P_\eta + \frac{4\consL_{2,0}}{\eta}\sup_{f\in\cF_m}\En_{\tau_{m-1}}\brk*{M_{\tau_{m-1}}(f)} \\
&\leq \eta P_\eta + \frac{4\consL_{2,0}}{\eta}\cdot\frac{2\beta_{m}(\tau_m-1) + \concF}{T_{m-1}}
            \tag{by the second claim of \pref{lem:disagreement_containment_F}} \\
&= \eta P_\eta + \frac{4\consL_{2,0}(2M - 2m + 3)\concF}{\eta T_{m-1}}.
\end{align*}
Summing over $t = 1, \ldots, T$, we arrive at
\[
\sum_{t=1}^T \En_t\brk*{f^\star(x_t, \pi^\star(x_t)) - f^\star(x_t, a_t)}
 = \eta TP_\eta + \sum_{m=1}^M T_m \frac{4\consL_{2,0}(2M - 2m + 3)\concF}{\eta T_{m-1}}
 = \eta TP_\eta + \frac{8\consL_{2,0}(M^2+2M)\concF}{\eta}.
\]
Finally in the general case we bound $P_\eta$ by $1$ and pick the optimal $\eta$ to arrive at a conditional expected regret bound of order
$O(\sqrt{T\consL_{2,0}\concF}\log T + \log(1/\delta))$,
while under the Massart condition (\pref{def:massart}) one can pick $\eta = \gamma$ so that $P_\eta = 0$ and obtain
a conditional expected regret bound of order $O\left(\frac{\consL_{2,0} \concF \log^2 T}{\gamma}\right)$.

As in the proof of \pref{thm:disagreement_with_massart}, we relate the sum of conditional expected instantaneous regrets back to the quantities in the theorem statement differently in the general case and the Massart case. In the general case we have
\begin{align*}
\sum_{t=1}^T r_t(\pi^\star(x_t)) - r_t(a_t)
&\leq \sum_{t=1}^T \En_t\brk*{r_t(\pi^\star(x_t)) - r_t(a_t)} + O(\sqrt{T\log(1/\delta)})
\end{align*}
with probability at least $1-\delta/2$.

In the Massart case, the law of total expectation implies
\[
\En\brk*{\sum_{t=1}^T r_t(\pi^\star(x_t)) - r_t(a_t)} \leq{} O\left(\frac{\consL_{2,0} C'_{1/T}\log^2 T}{\gamma}\right) + \frac{1}{T}\cdot{}T.
\]
\end{proof}

\begin{theorem}[Full version of \pref{thm:moment_optimistic}]
\label{thm:moment_optimistic_with_massart}
With $\beta_m = \frac{(M-m+1)\concF}{\tau_m - 1}$, where $\concF$ is as in \pref{lem:disagreement_conc_F}, and $M_0 = 2 + \left\lfloor \log_2\left(1 + \frac{(2M+3)\consL_1\concF}{\lambda^2}\right) \right\rfloor$ for any $\lambda \in (0,1)$,
\pref{alg:regression_ucb_optimistic} ensures that with probability at least $1- \delta$,
\[
\reg_T = O\left(\frac{\consL_1 \concF \log T}{\lambda^2} +  \sqrt{T\consL_{2,\lambda}\concF}\log T \right),
\]
which for finite classes is bounded by $\tO\left(\frac{\consL_1\log\abs*{\cF}}{\lambda^2} +  \sqrt{T\consL_{2,\lambda}\log\abs*{\cF}}\right)$.\\
Furthermore, if the Massart noise condition (\pref{def:massart}) is satisfied with parameter $\gamma$, then \pref{alg:regression_ucb_optimistic} configured as above with $\delta=1/T$ enjoys an expected regret bound of
\[
\En\brk*{\sum_{t=1}^T r_t(\pi^\star(x_t)) - r_t(a_t)} = O\left(\frac{\consL_1 C'_{1/T} \log T}{\lambda^2} + \frac{\consL_{2,\lambda} C'_{1/T} \log^2 T}{\gamma}\ind\{\lambda > \gamma\}\right),
\]
which for finite classes is bounded by $\tO\left(\frac{\consL_1 \log\abs*{\cF}}{\lambda^2} + \frac{\consL_{2,\lambda} \log\abs*{\cF}}{\gamma}\ind\{\lambda > \gamma\}\right)$.
\end{theorem}
\begin{proof}

We condition on the same events of \pref{lem:disagreement_conc_F}, which hold with probability at least $1-\delta/2$.
By the second claim of \pref{lem:disagreement_containment_F}, we have for any $f \in \cF_{M_0}$,
\[
\sum_{t=1}^{\tau_{M_0}-1}\En_{t}\brk*{M_{t}(f)} \leq 2\beta_{M_0}(\tau_{M_0} - 1) + \concF = 2\frac{(M-M_0+1)\concF}{\tau_{M_0}-1}(\tau_{M_0} - 1) + \concF \leq (2M+3)\concF.
\]
Since \pref{alg:regression_ucb_optimistic} performs pure exploration for any $t$ before epoch $M_0$, we conclude that
\[
\En_{t}\brk*{M_{t}(f)} = \En_{x\sim D}\En_{a\sim\Unif{\cA}}(f(x,a)-f^\star(x,a))^2 \leq \frac{(2M+3)\concF}{\tau_{M_0}-1},
\]
and therefore together with \pref{def:infinity_l2}, we have for any $f \in \cF_{M_0}$, $x\in \cX$, and $a\in\cA$,
\begin{equation}
\label{eq:bounded_width}
(f(x,a)-f^\star(x,a))^2 \leq \consL_1\En_{x'\sim D_\cX} \En_{a'\sim\Unif{\cA}}(f(x',a')-f^\star(x',a'))^2 \leq \frac{(2M+3)\consL_1\concF}{\tau_{M_0}-1}
< \lambda^2,
\end{equation}
where the last step holds by the choice of $M_0$.
Next we claim that for any $t \geq \tau_{M_0}$, if $x_t \in U_{\lambda}(a)$ for some $a$, then it must be the case $a_t = a = \pi^\star(x_t)$.
To begin, we have that $a = \pi^\star(x_t)$, which is by the definition of $U_{\lambda}(a)$. Moreover,
with $m$ being the epoch to which $t$ belongs and $a' = \argmax_{a\neq \pi^\star(x_t)}\High_{\cF_m}(x_t, a)$, we have
\begin{align*}
&\High_{\cF_m}(x_t, a) - \High_{\cF_m}(x_t, a') \\
&= \underbrace{f^\star(x_t, a) - f^\star(x_t, a')}_{\geq{}\lambda} + \underbrace{\High_{\cF_m}(x_t, a) - f^\star(x_t, a)}_{\geq{}0} + \underbrace{f^\star(x_t, a')  - \High_{\cF_m}(x_t, a')}_{>-\lambda} \\
&> \lambda  + 0 - \lambda = 0,
\end{align*}
where the inequality is by $x_t \in U_{\lambda}(a)$, $f^\star \in \High_{\cF_m}$, and Eq.~\pref{eq:bounded_width}.
By the optimistic strategy of \pref{alg:regression_ucb_optimistic}, this implies $a_t = a$.

Finally we proceed exactly the same as the proof of \pref{thm:moment} to arrive at for any $\eta > 0$, $\lambda \in (0,1)$, $m > M_0$, and $t$ in epoch $m$,
\begin{align*}
\En_t\brk*{f^\star(x_t, \pi^\star(x_t)) - f^\star(x_t, a_t)}
&\leq \eta P_\eta + \frac{4\consL_{2,\lambda}}{\eta}\sup_{f\in\cF_m}\En_{x\sim D_\cX}\En_{a\sim \Unif{\cA}}
    \brk*{\one\bigSet{x \in U_{\lambda}(a)}(f^\star(x, a)  - f(x, a))^2}.
\end{align*}
With the fact established above, since $\tau_{m-1} \geq \tau_{M_0}$ we continue with
\begin{align*}
\En_t\brk*{f^\star(x_t, \pi^\star(x_t)) - f^\star(x_t, a_t)}
&\leq \eta P_\eta + \frac{4\consL_{2,\lambda}}{K\eta}\sup_{f\in\cF_m}\En_{x_{\tau_{m-1}}}
    \brk*{(f^\star(x_{\tau_{m-1}}, a_{\tau_{m-1}})  - f(x_{\tau_{m-1}}, a_{\tau_{m-1}}))^2} \\
&= \eta P_\eta + \frac{4\consL_{2,\lambda}}{\eta}\sup_{f\in\cF_m}\En_{\tau_{m-1}}\brk*{M_{\tau_{m-1}}(f)} \\
&\leq \eta P_\eta + \frac{4\consL_{2,\lambda}}{\eta}\cdot\frac{2\beta_{m}(\tau_m-1) + \concF}{T_{m-1}}
            \tag{by the second claim of \pref{lem:disagreement_containment_F}} \\
&= \eta P_\eta + \frac{4\consL_{2,\lambda}(2M - 2m + 3)\concF}{\eta T_{m-1}}.
\end{align*}
Therefore, the regret bound is
\[
\reg_T \leq \tau_{M_0+1} + \eta T P_\eta + \sum_{m = M_0+1}^M T_m \frac{4\consL_{2,\lambda}(2M - 2m + 3)\concF}{\eta T_{m-1}}
\leq O\left(\frac{\consL_1 \concF \log T}{\lambda^2}\right) + \eta TP_\eta + \frac{8\consL_{2,\lambda}(M^2+2M)\concF}{\eta}.
\]
Again in general we bound $P_\eta$ by $1$ and pick the optimal $\eta$ to arrive at
\[
\sum_{t=1}^T \En_{t}\brk*{r_t(\pi^\star(x_t)) - r_t(a_t)} = O\left(\frac{\consL_1 \concF \log T}{\lambda^2} +  \sqrt{T\consL_{2,\lambda}\concF}\log T \right),
\]
while under the Massart condition we pick $\eta = \gamma$ so that $P_\eta = 0$ and
\[
\sum_{t=1}^T \En_{t}\brk*{r_t(\pi^\star(x_t)) - r_t(a_t)} = O\left(\frac{\consL_1 \concF \log T}{\lambda^2} + \frac{\consL_{2,\lambda} \concF \log^2 T}{\gamma}\right).
\]
Specifically, if we choose $\lambda \leq \gamma$, then every $x_t$ is in $U_{\lambda}(\pi^\star(x_t))$
and thus $a_t = \pi^\star(x_t)$ for all $t \geq \tau_{M_0}$ and the algorithm suffers no regret at all after the warm start,
that is, $\reg_T = O\left(\frac{\consL_1 \concF \log T}{\lambda^2}\right)$.

To conclude we proceed as in the proof of \pref{thm:moment_with_massart}: In the general case we have
\begin{align*}
\sum_{t=1}^T r_t(\pi^\star(x_t)) - r_t(a_t)
&\leq \sum_{t=1}^T \En_t\brk*{r_t(\pi^\star(x_t)) - r_t(a_t)} + O(\sqrt{T\log(1/\delta)})
\end{align*}
with probability at least $1-\delta/2$ by Azuma-Hoeffding.

In the Massart case, the law of total expectation implies
\[
\En\brk*{\sum_{t=1}^T r_t(\pi^\star(x_t)) - r_t(a_t)} \leq{} O\left(\frac{\consL_1 C'_{1/T} \log T}{\lambda^2} + \frac{\consL_{2,\lambda} C'_{1/n} \log^2 T}{\gamma}\right) + \frac{1}{T}\cdot{}T.
\]

\end{proof}

\begin{proof}[\pfref{prop:linear_ex}]
For this proof we will adopt the shorthand $\ind\crl*{U_{\lambda}(a)}\ldef\ind\crl*{x\in{}U_{\lambda}(a)}$.

We first consider the $\ls_2$ case. In this case (using $w\in{}\bbR^d$ as a stand-in for $f-f^{\star}$ and $\cW^{\star}\ldef{}\cW-w^{\star}\setminus{}\crl*{0}$) it is sufficient to take
\begin{align*}
  L_{2,\lambda} &\leq{} \sup_{w\in\cW^{\star}}\frac{\sup_{x\in\cX,a}\tri*{w,\phi(x,a)}^{2}}{\frac{1}{K}\sum_{a\in\cA}\En_{x}(\tri*{w,\phi(x,a)}\ind\crl*{U_{\lambda}(a)})^{2}} \\
                &\leq{} \sup_{w\in\cW^{\star}}\frac{\nrm*{w}_{2}^{2}}{\frac{1}{K}\sum_{a\in\cA}\En_{x}(\tri*{w,\phi(x,a)}\ind\crl*{U_{\lambda}(a)})^{2}} \\
                &\leq{} \sup_{w\in\cW^{\star}}\frac{\nrm*{w}_{2}^{2}}{\nrm*{w}_{2}^{2}\lambda_{\textrm{min}}\prn*{\frac{1}{K}\sum_{a\in\cA}\En_{x}\phi(x,a)\phi(x,a)^{\top}\ind\crl*{U_{\lambda}(a)}}}\\
                &= \frac{1}{\lambda_{\textrm{min}}\prn*{\frac{1}{K}\sum_{a\in\cA}\En_{x}\phi(x,a)\phi(x,a)^{\top}\ind\crl*{U_{\lambda}(a)}}}.
\end{align*}
In the sparse high-dimensional setting we have
\begin{align*}
  L_{2,\lambda} &\leq{} \sup_{w\in\cW^{\star}}\frac{\sup_{x\in\cX,a}\tri*{w,\phi(x,a)}^{2}}{\frac{1}{K}\sum_{a\in\cA}\En_{x}(\tri*{w,\phi(x,a)}\ind\crl*{U_{\lambda}(a)})^{2}} \\
                &\leq{} \sup_{w\in\cW^{\star}}\frac{2s\nrm*{w}_{2}^{2}}{\frac{1}{K}\sum_{a\in\cA}\En_{x}(\tri*{w,\phi(x,a)}\ind\crl*{U_{\lambda}(a)})^{2}} \\
                &\leq{} \sup_{w\in\cW^{\star}}\frac{2s\nrm*{w}_{2}^{2}}{\nrm*{w}_{2}^{2}\psi_{\textrm{min}}\prn*{\frac{1}{K}\sum_{a\in\cA}\En_{x}\phi(x,a)\phi(x,a)^{\top}\ind\crl*{U_{\lambda}(a)}}} \\
                &= \frac{2s}{\psi_{\textrm{min}}\prn*{\frac{1}{K}\sum_{a\in\cA}\En_{x}\phi(x,a)\phi(x,a)^{\top}\ind\crl*{U_{\lambda}(a)}}}.
\end{align*}

As remarked in the main body, in general it holds that $L_{1}\leq{}L_{2,0}$. Nonetheless, it is also possible to directly bound $L_{1}$ using similar reasoning to the proof above:
\[
L_{1} \leq{} \frac{1}{\lambda_{\textrm{min}}\prn*{\frac{1}{K}\sum_{a\in\cA}\En_{x}\phi(x,a)\phi(x,a)^{\top}}}
\]
for the $\ls_2$ example
and
\[
L_{1} \leq{} \frac{2s}{\psi_{\textrm{min}}\prn*{\frac{1}{K}\sum_{a\in\cA}\En_{x}\phi(x,a)\phi(x,a)^{\top}}},
\]
for the sparsity example.
\end{proof}

%%% Local Variables:
%%% mode: latex
%%% TeX-master: "paper"
%%% End:

\section{Experimental Details}
\label{app:experiments}
\subsection{Datasets}
\label{app:datasets}

We evaluated on datasets for learning-to-rank and for multiclass classification.

The learning-to-rank datasets, which were previously used for evaluating contextual semibandits in \cite{krishnamurthy2016contextual}, are as follows:
\begin{itemize}
\item
Microsoft Learning to Rank \citep{qin2010mslr}. We use the \texttt{MSLR-WEB30K} variant available at \url{https://www.microsoft.com/en-us/research/project/mslr/}. This dataset has $T=31278$, $d=136$. We limit the choices to $K=10$ documents (actions) per query. The MSLR repository comes partitioned into five segments, each with $T=31278$ queries and a varying number of documents. We use the first three segments for the documents in our training dataset and use documents from the fourth segment for validation.
\item
Yahoo! Learning to Rank Challenge V2.0 \cite{chapelle2011yahoo} (variant C14B at \url{https://webscope.sandbox.yahoo.com/catalog.php?datatype=c}). The dataset has $T=33850$, $d=415$, and $K=6$. We hold out 7000 examples for test.
\end{itemize}
Each learning-to-rank dataset contains over $30,000$ queries, with the number of documents varying.
In both datasets feedback each document is labeled with relevance score in $\crl*{0,\ldots, 4}$. We transform this to a contextual bandit problem by presenting $K$ documents as actions and their relevance scores as the rewards, so that the goal of the learner is to choose the document with the highest relevance each time it is presented with a query.

The multiclass classification datasets are taken from the UCI repository \citep{lichman2013uci} summarized in \pref{tab:uci}. This collection was previously used for evaluating contextual bandit learning in \cite{dudik2011doubly}. Each context is labeled with the index in $\brk{K}$ of the class to which the context belongs, and the reward for selecting a class is $1$ if correct, $0$ otherwise.

\paragraph{Validation}
Validation is performed by simulating the algorithm's predictions on examples from a holdout set without allowing the algorithm to incorporate these examples. The validation error at round $t$ therefore approaches the instantaneous expected reward $\En_{x_t,a_t}\brk*{f^{\star}(x_t, a_t)}$ at a rate determined by uniform convergence for the class $\cF$. We also plot the validation reward of a ``supervised'' baseline obtained by training the oracle (either \textsf{Linear} or \textsf{GB5}) on the entire training set at once (including rewards for all actions).

\paragraph{Noisy dataset variants}

For all of the multiclass datasets we also create an alternate version with real-valued costs by constructing a reward matrix $R_{t}\in\bbR^{K\times{}K}$ and returning $R_{t}(a,a^{\star})$ as the reward for selecting action $a$ when $a^{\star}$ is the correct label at time $t$. $R_t$ is constructed as a (possibly asymmetric) matrix with all ones on the diagonal ($R_{t}(a,a)=1$) and random values in the range $\brk*{0,1}$ for each off-diagonal entry. The off diagonal elements are generated through the following process: 1) For each off-diagonal pair $(a,a')$ draw a ``mean'' $\mu(a,a')\in\brk*{0,1}$ uniformly at random. This value of $\mu$ is held constant across all timesteps and all repetitions. 2) At time $t$, sample $R_{t}(a,a')$ as a Bernoulli random variable with bias $\mu(a,a')$. The reward matrices that were sampled are included in \pref{sec:reward} for reference.

\begin{table}[h!]
\caption{UCI datasets (before validation split).}
\label{tab:uci}
\begin{center}
\begin{tabular}{ | c | c | c | c |}
 \hline Dataset & $n$ & $d$ & $K$ \\
 \hline
letter & 20000 & 17 & 26 \\
\hline
optdigits & 5620 & 65 & 10 \\
  \hline
adult & 45222 & 105 & 2\\
\hline
page-blocks & 5473 & 11 & 5 \\
\hline
pendigits & 10992 & 17 & 10 \\
\hline
satimage & 6430 & 37 & 6\\
\hline
vehicle & 846 & 19 & 4\\
\hline
yeast  &1479 & 9 & 9  \\
\hline \end{tabular}
\end{center}
\end{table}

\subsection{Benchmark algorithms}
We compared with the following benchmark algorithms:
\begin{itemize}
\item \textsf{$\eps$-Greedy} \citep{langford2008epoch}. Policy is updated on a doubling schedule: Every $2^{i/2}$ rounds. We use an exploration probability of $\max\crl*{1/\sqrt{t}, \eps}$ at time $t$, then tune $\eps$ as described in the main paper.
\item \textsf{ILOVETOCONBANDITS} \citep{agarwal2014taming}: Updated every $2^{i/2}$ rounds. We tune the constant in front of the parameter $\mu_{m}$ described in Algorithm 1 in \cite{agarwal2014taming}.

\item \textsf{Bootstrap} \cite{dimakopoulou2017estimation}: At each epoch, the algorithm draws $N$ bootstrap replicates of the dataset so far, then fits a predictor in $\cF$ to each replicate, giving a collection of predictors $(f_{i})_{i\in\brk*{N}}$. To predict on a new context $x$ we compute the mean and variance of the predictions $f_{i}(x,\cdot)$, then pick the action that maximizes the upper confidence bound for a Normal distribution with this mean and variance. We tune the confidence $\beta$ on these predictions, so that the algorithm picks the action maximizing $\textrm{Mean}(a) + \sqrt{\beta\textrm{Var}(a)}$.
\item As discussed in the main body, we tune the parameter $\beta=\beta_m$ for both \rucb{} variants.

\end{itemize}

For each algorithm we tried 8 different values of the relevant parameter coming from a logarithmically spaced grid ranging from $10^{2}$ to $10^{-8}$ for the confidence interval-based algorithms (\rucb{} and \bootstrap{}) and range $10^{-1}$ to $10^{-8}$ for \egreedy{} and \iltcb{}.

Each algorithm must be supplied with a model class $\cF$ and an optimization oracle for this class. Both the model class $\cF$ and the oracle implementation are hyperparameters. How to choose the oracle once the class $\cF$ is been fixed is discussed below.

\subsection{Oracle implementation}
All of the oracle-based algorithms require optimization oracles, for either predictor classes or policy classes. We consider the following three types of basic oracles.

\begin{enumerate}
\item Weighted regression onto single action
\[
\argmin_{f\in\cF}\sum_{t=1}^{T}w_{t}(f(x_t, a_t) - r_t(a_t))^{2}.
\]
\item Weighted regression onto all actions
\[
\argmin_{f\in\cF}\sum_{t=1}^{T}\sum_{a\in\cA}w_{t,a}(f(x_t, a) - r_t(a))^{2}.
\]
\item Weighted multiclass classification
\[
\argmin_{f\in\cF}\sum_{t=1}^{T}w_{t}\ind\crl*{\pi_{f}(x_t) \neq{} a_t}.
\]
\end{enumerate}

\paragraph{Oracles for importance-weighted observations} One of the most common datasets one needs to optimize over to implement oracle-based contextual bandit algorithms is an importance weighted history of interactions. That is, $H_{T}=\crl*{(x_t, a_t, r_t(a_t), p_t(a_t))}_{t=1}^{T}$, where $x_{t}$ and $r_{t}(a_t)$ are the unmodified context and reward provided by nature, $a_t$ is the action selected by a randomized contextual bandit algorithm, and $p_{t}(a_t)$ is the (positive) probability that $a_t$ was selected. The core optimization problem that must be solved for such a dataset (e.g., in $\eps$-Greedy) is
\begin{equation}\label{eq:argmax_importance_weighted}
\argmax_{f\in\cF}\sum_{t=1}^{T}\frac{r_{t}(a_t)}{p_{t}(a_t)}\ind\crl*{\pi_{f}(x_t)=a_t}.
\end{equation} This problem most naturally reduces to weighted multiclass classification, but under the realizability assumption in \pref{ass:realizable} it can also be reduced to regression in a number of principled ways. The full list of possible reductions we consider is as follows:
\begin{itemize}
\item Unweighted regression
\begin{equation}
\label{eq:reduction_a}
\tag{A}
\argmin_{f\in\cF}\sum_{t=1}^{T}(f(x_t, a_t) - r_t(a_t))^{2}.
\end{equation}
\item Importance-weighted regression
\begin{equation}
\label{eq:reduction_b}
\tag{B}
\argmin_{f\in\cF}\sum_{t=1}^{T}\frac{1}{p_{t}(a_t)}(f(x_t, a_t) - r_t(a_t))^{2}.
\end{equation}
\item Regression with importance weighted targets
\begin{equation}
\label{eq:reduction_c}
\tag{C}
\argmin_{f\in\cF}\sum_{t=1}^{T}(f(x_t, a_t) - r_t(a_t)/p_{t}(a_t))^{2} + \sum_{a\neq{}a_t}(f(x_t,a))^{2}.
\end{equation}
\item Importance-weighted multiclass
\begin{equation}
\label{eq:reduction_d}
\tag{D}
\argmin_{\pi\in\Pi}\sum_{t=1}^{T}\frac{r_{t}(a_t)}{p_t(a_t)}\ind\crl*{\pi(x_t)\neq{}a_t}.
\end{equation}
Note that in this case the policy class $\Pi$ is not necessarily induced by a predictor class $\cF$, though when it is it may be possible to further reduce this optimization problem to one of the first three problems.
\end{itemize}

The minimizer of \pref{eq:reduction_d} corresponds to the maximizer of \pref{eq:argmax_importance_weighted}. Reductions \pref{eq:reduction_a}, \pref{eq:reduction_b}, and \pref{eq:reduction_c} all have the property that if the conditional expectation version of the loss (e.g. $\sum_{t=1}^{T}\En_{(x_t,r_t)\sim{}\cD}\En_{a_t\mid{}x_t, H_{t-1}}\brk*{(f(x_t, a_t) - r_t(a_t))^{2}}$ for \pref{eq:reduction_a}) is used, then the Bayes predictor $f^{\star}(x,a)=\En\brk*{r(a)\mid{}x}$ is the minimizer when $f^{\star}\in\cF$, which (via uniform convergence) justifies the use of the empirical versions.

\paragraph{Oracle choices for benchmark algorithms}
Depending on the needs of each benchmark algorithm, \pref{eq:reduction_a}, \pref{eq:reduction_b}, \pref{eq:reduction_c}, or \pref{eq:reduction_d} as well as other oracles may be possible to use or required. We discuss the choices for each benchmark
\begin{itemize}
\item \egreedy{}: This strategy only needs to solve an importance weighted argmax of the form \pref{eq:argmax_importance_weighted}, so all of \pref{eq:reduction_a}, \pref{eq:reduction_b}, \pref{eq:reduction_c}, and \pref{eq:reduction_d} can be used under realizability. Note that since actions are sampled uniformly in the non-greedy rounds, \pref{eq:reduction_a} and \pref{eq:reduction_b} are equivalent under this strategy. In experiments we use \pref{eq:reduction_b}.
\item \bootstrap{}: Like \rucb{}, this strategy is tailored to the realizable regression setting, so \pref{eq:reduction_a} suffices. While \pref{eq:reduction_b} and \pref{eq:reduction_c} could also be used, we expect them to have higher variance.
\item \minimonster{}: This algorithm requires two different oracles. First, it requires the optimization problem \pref{eq:argmax_importance_weighted} to be solved on the unmodified reward/context sequence. Second, it requires a bonafide \emph{cost-sensitive classification} optimization oracle of the form
\[
\argmax_{f\in\cF}\sum_{t=1}^{T}r_{t}(\pi_f(x_t))
\]
 for an artificial sequence of rewards which may not be realizable even when the rewards given by nature are. As in \egreedy{}, the first oracle can use \pref{eq:reduction_a}, \pref{eq:reduction_b}, \pref{eq:reduction_c}, and \pref{eq:reduction_d}. The second oracle is more complicated. Cost-sensitive classification is typically not implemented directly and instead is reduced to either weighted multiclass \pref{eq:reduction_d} or multi-output regression, for which \pref{eq:reduction_c} is a special case. Note that \pref{eq:reduction_d} can further be reduced to \pref{eq:reduction_a}, \pref{eq:reduction_b}, \pref{eq:reduction_c}, but because we do not expect realizability to hold it is more natural to use the direct reduction to \pref{eq:reduction_c} in this case. In experiments we used \pref{eq:reduction_b} for empirical regret minimizer and \pref{eq:reduction_c} for the cost-sensitive classifier to solve the optimization problem OP in \citet{agarwal2014taming}.
\end{itemize}

\paragraph{Label-dependent features} For different datasets we consider different instantiations of the general predictor class setup described in the main paper. We assume there is a base context space $\cZ$ and predictor class $\cG:\cZ\to\bbR$. Give such a class there are two natural ways to build a class of predictors over the joint context-action space depending on how the dataset is featurized.
\begin{itemize}
\item \textbf{Label-dependent features} For the MSLR and Yahoo datasets, each context comes with a distinct set of features for each action.
This is captured by our abstraction by defining a fixed feature map $\phi:\cX\times{}\cA\to\cZ$, then defining the class $\cF$ via $\cF=\crl*{(x,a)\mapsto{}g(\phi(x,a))\mid{}g\in\cG}$.
\item \textbf{Label-independent features} When the contexts do not have label-dependent features, we use one instance of the base real-valued predictor class $\cG$ for each action, i.e. we set $\cZ=\cX$ and take $\cF=\crl*{(x,a)\mapsto{}g_{a}(x)\mid{} g=(g_a)_{a\in\cA}\in\cG^{\cA}}$.
\end{itemize}

\paragraph{Predictor class and base oracle implementation}
We use real-valued predictors from the scikit-learn library \cite{pedregosa2011scikit}. The two predictor classes used were
\begin{itemize}
\item \texttt{sklearn.linear\_model.Ridge(alpha=1)}
\item \texttt{sklearn.tree\_model.GradientBoostingRegressor(max\_depth=5, n\_estimators=100)}.
\end{itemize}
Each of the scikit-learn predictor classes handles this real-valued output case directly, via the $\texttt{fit()}$ function for each class. In the label-dependent feature case we use a single oracle for $\cG$, and in the label-independent feature case we use the oracle for $\cG$, then take $\cF=\cG^{\cA}$, so that there are actually $\abs*{\cA}$ oracle instances.

\paragraph{Incremental implementation for \rucb{}}
As mentioned in the main body, we restrict the optimization for the gradient boosting oracle when used with \rucb{}. At the beginning of each epoch $m$, we find best regression tree ensemble on the dataset so far (with respect to $\hR_m$). For each round within the epoch, we keep this tree structure fixed for the call to $\textsc{Oracle}(H)$, so that only the ensemble and leaf weights need to be re-optimized.

\subsection{Holdouts and multiple trials}
Each dataset shuffled via random permutation, then presented to the learner in order.

Each (algorithm, parameter configuration, dataset) tuple was run for 5 repetitions. For a given trial we distinguish between two sources of randomness: Randomness from the dataset, which may come from the random ordering or from randomness in the labels as described in the dataset section, and randomness in the contextual bandit algorithm's decisions. We control for randomness in the dataset across different (algorithm, parameter) configurations by giving each repetition an index $k$ and using the same random seed to select the dataset randomness across all configurations. This means that when $k$ is fixed, all variance is induced by the algorithm's action distribution.

Validation reward was computed every $T/15$ steps.

\subsection{Full collection of plots}
\begin{figure*}
  \begin{centering}
    ~\hfill\includegraphics[width=0.66\textwidth]{legend_all.pdf}\hfill~\\
    \medskip

\includegraphics[width=0.33\textwidth]{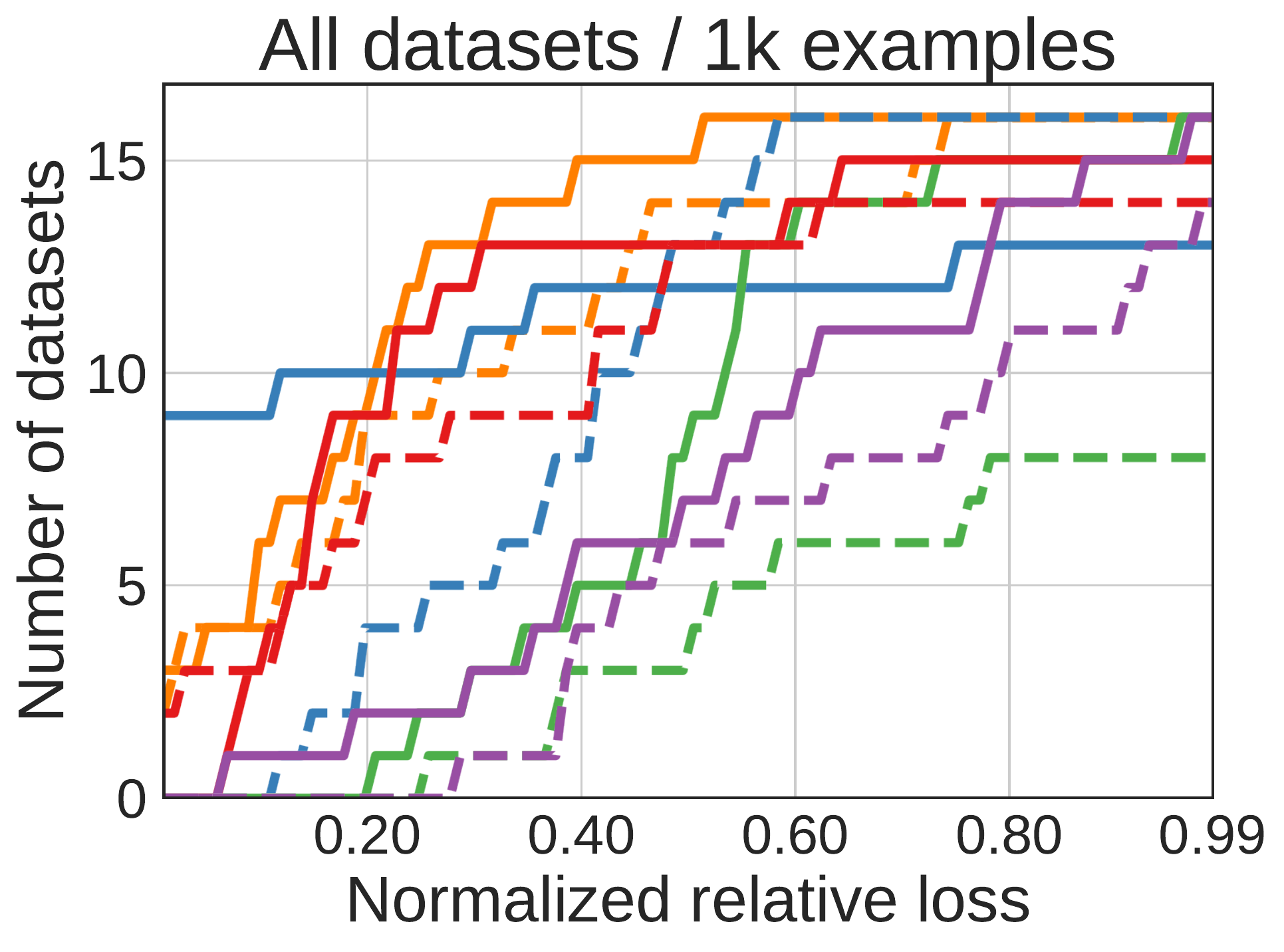}%
\includegraphics[width=0.33\textwidth]{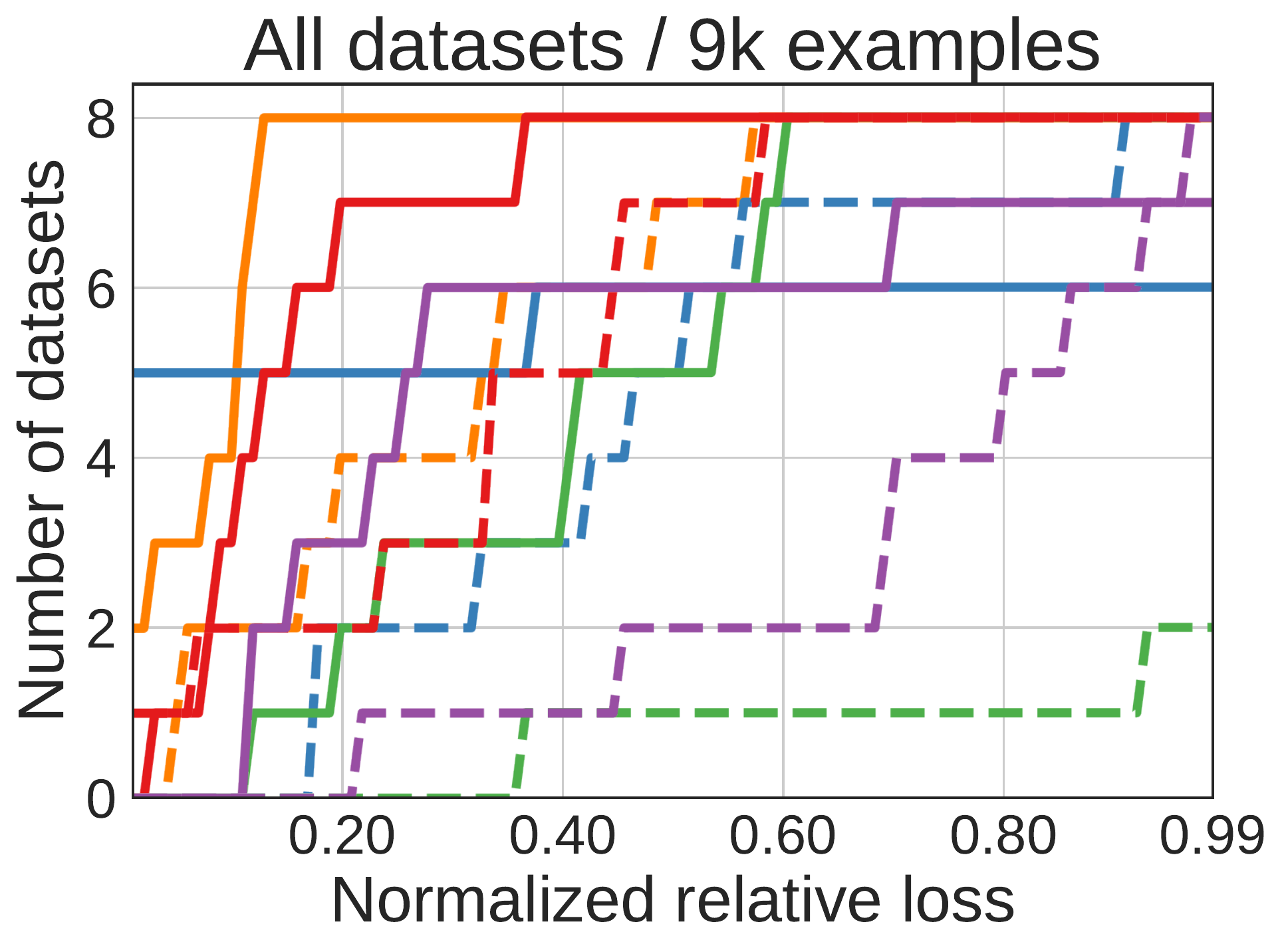}%
\includegraphics[width=0.33\textwidth]{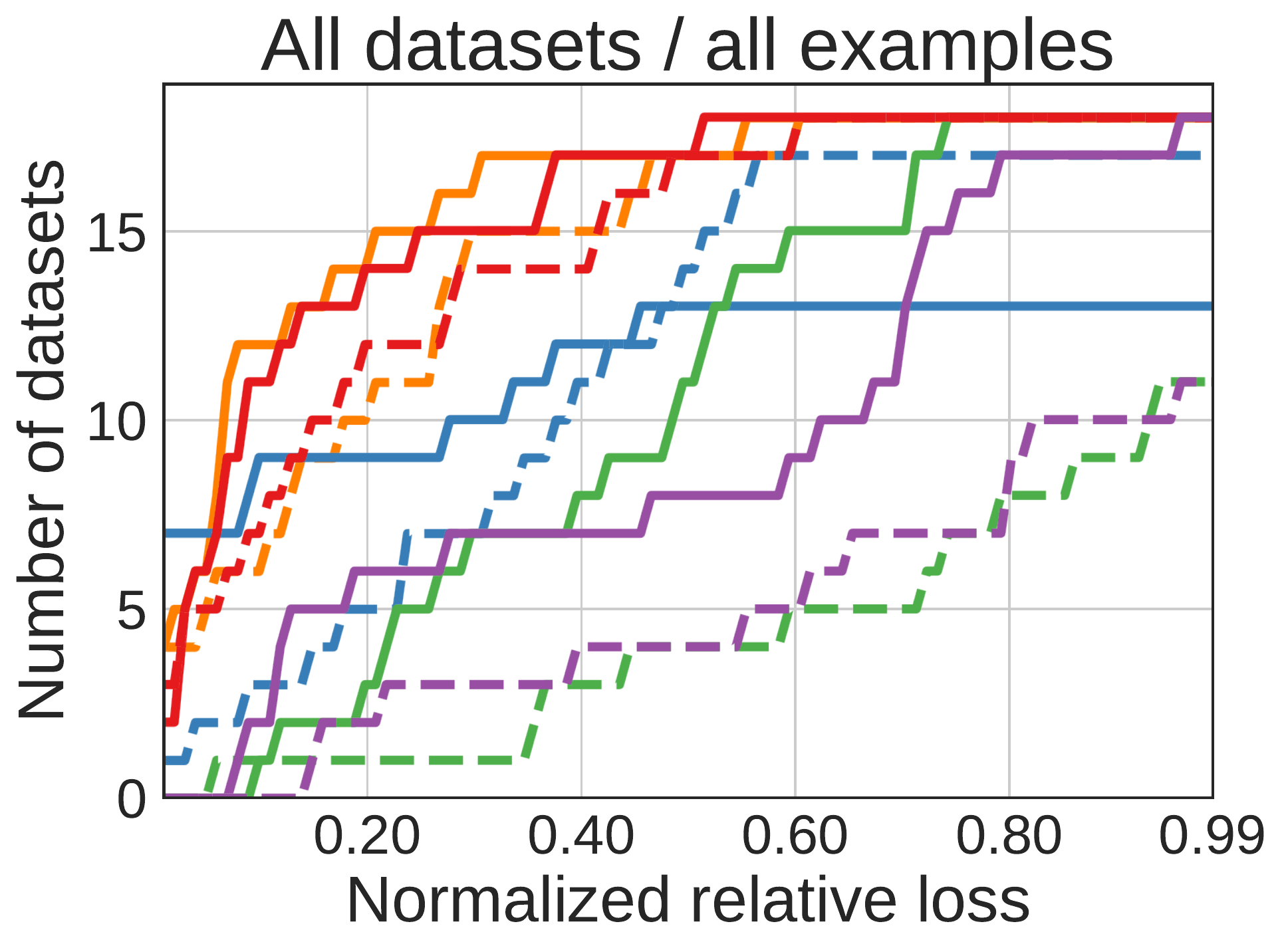}
\medskip

\includegraphics[width=0.33\textwidth]{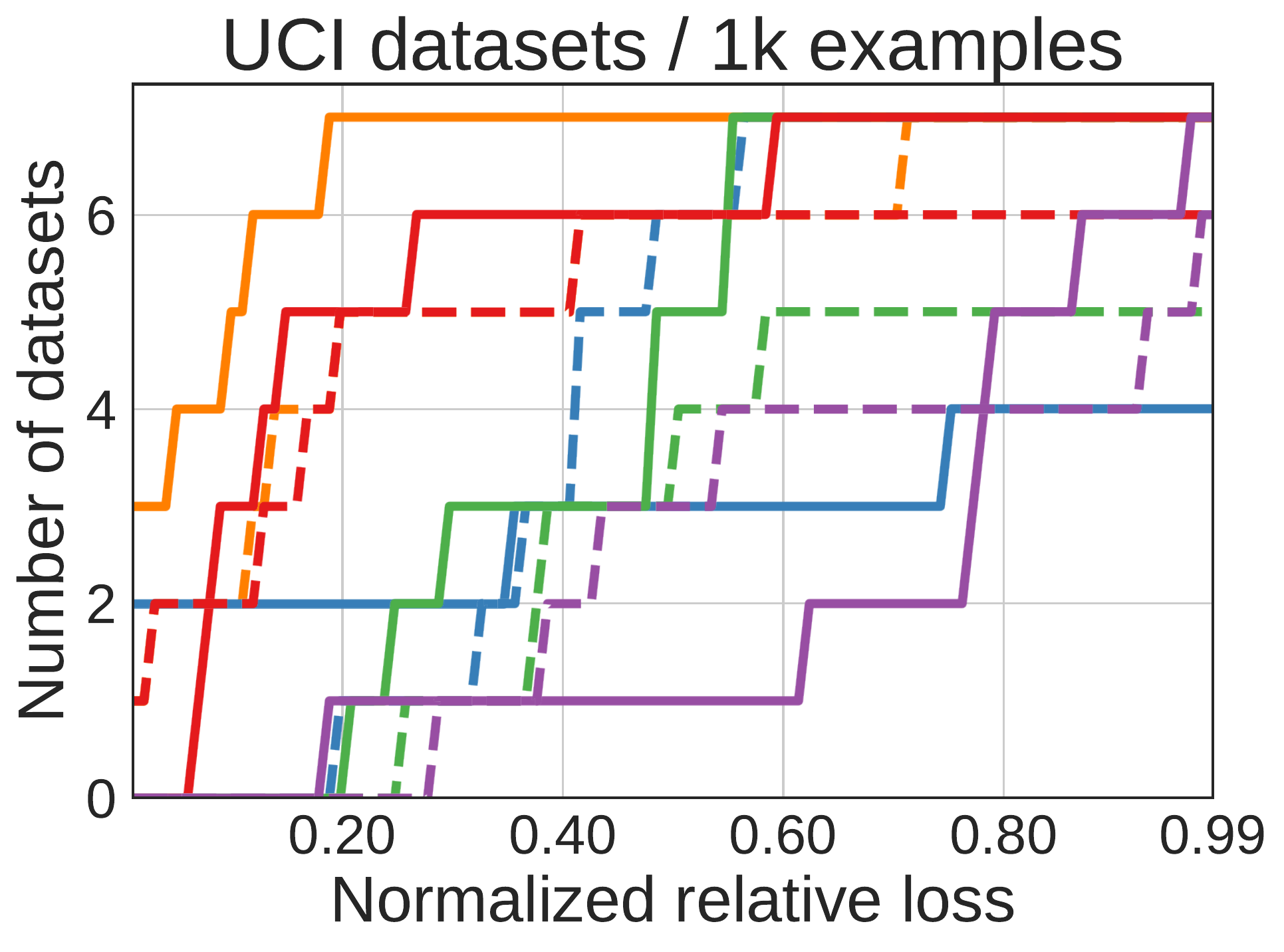}%
\includegraphics[width=0.33\textwidth]{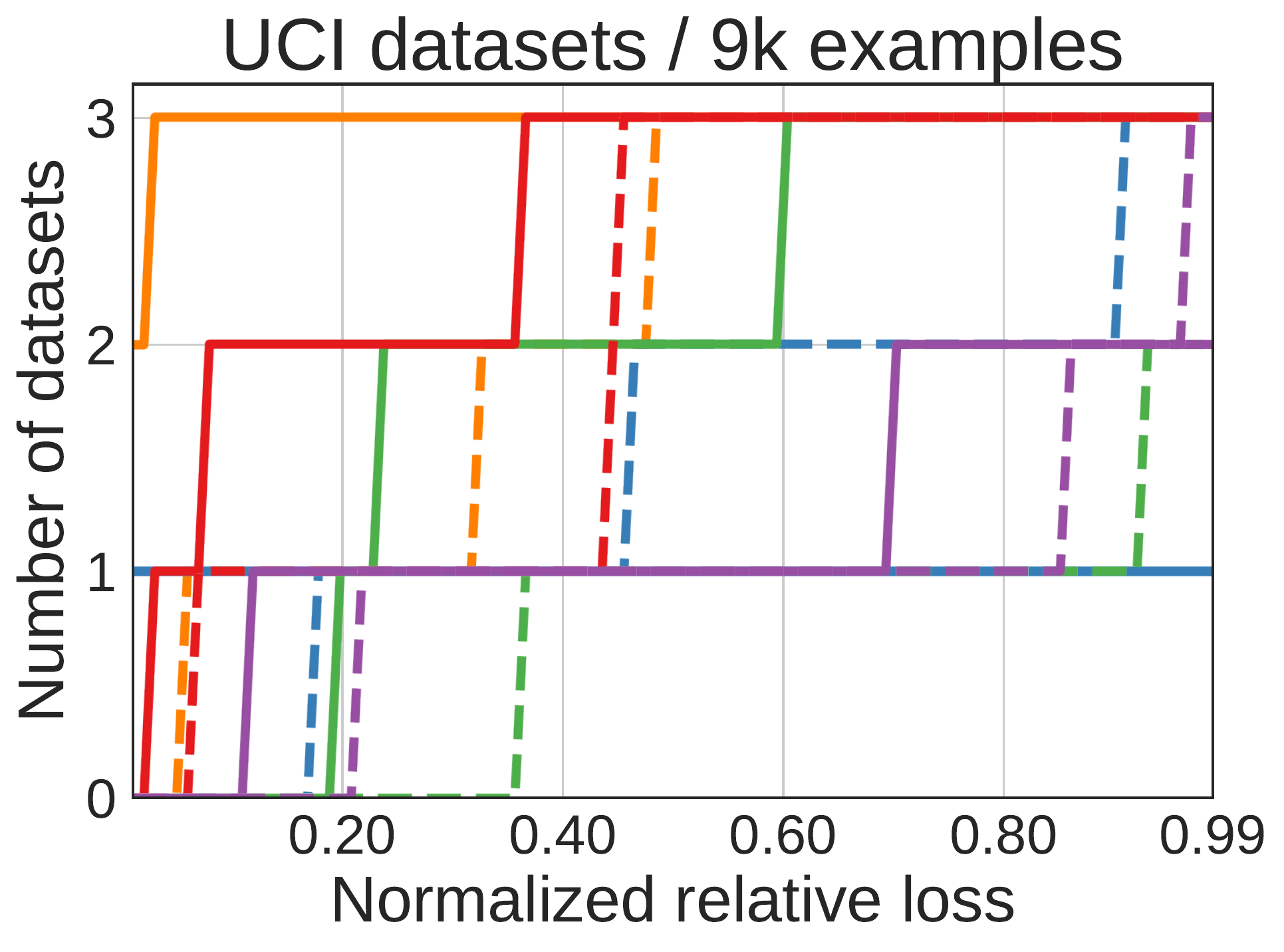}%
\includegraphics[width=0.33\textwidth]{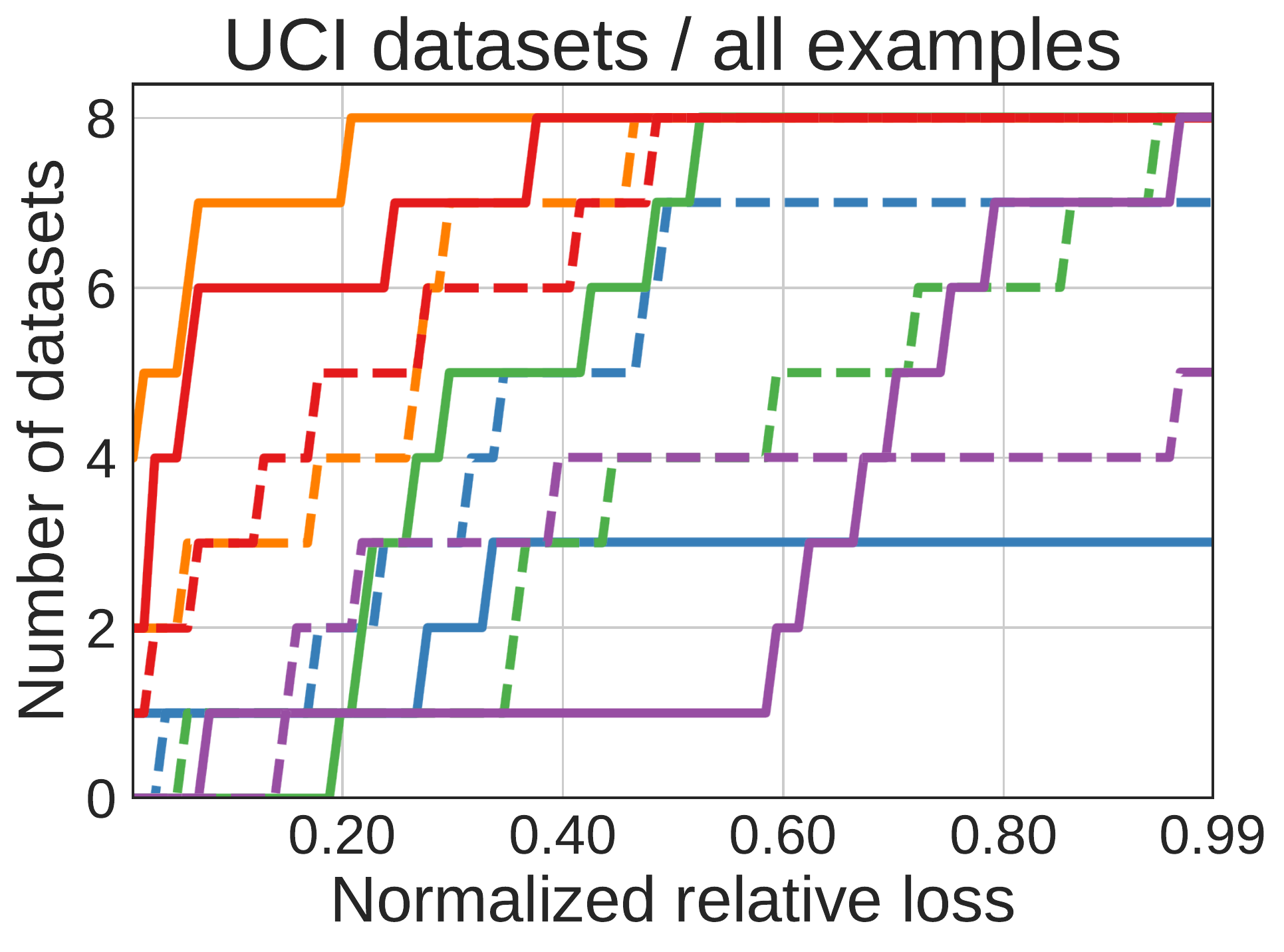}%
\medskip

\includegraphics[width=0.33\textwidth]{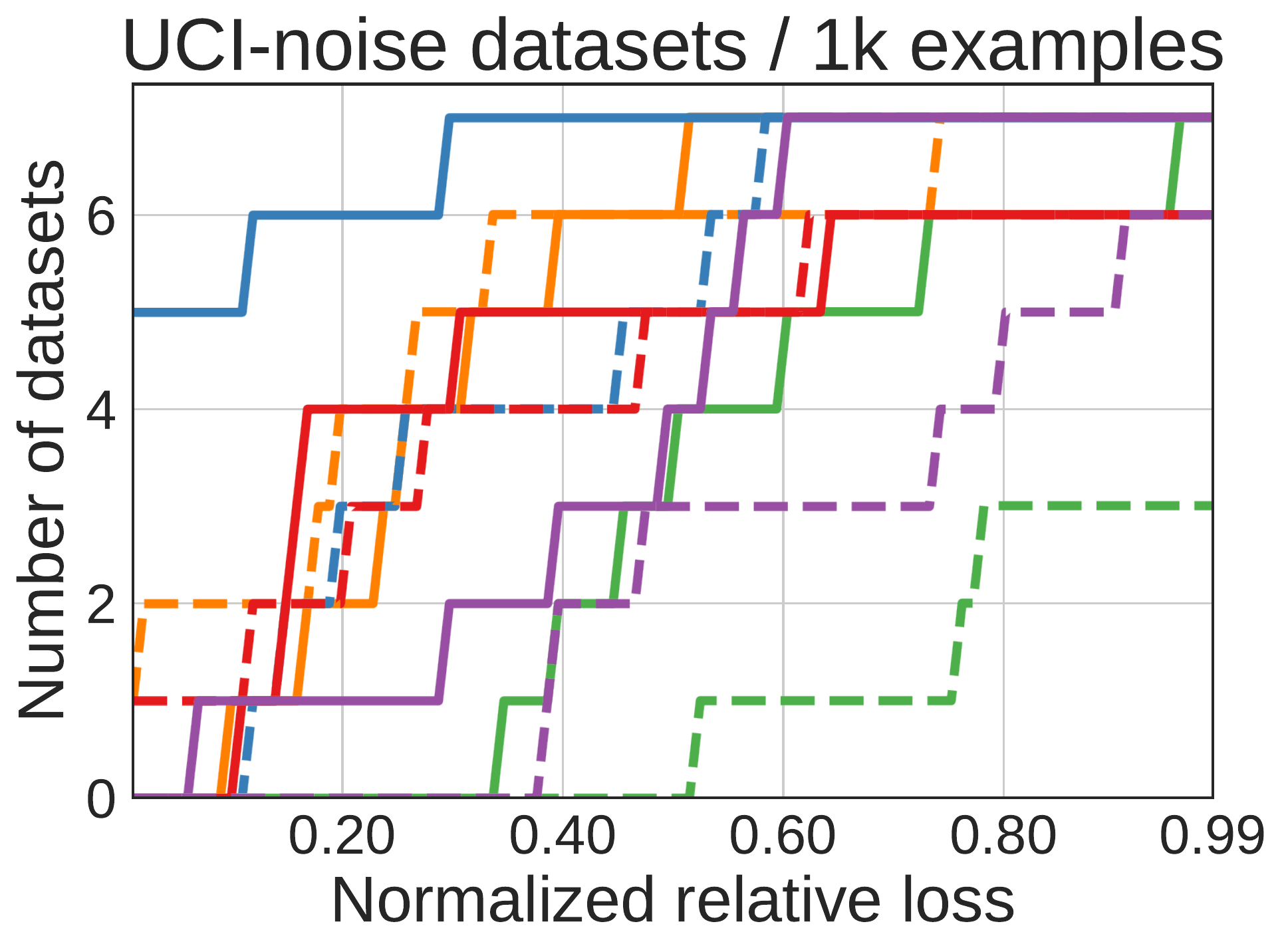}%
\includegraphics[width=0.33\textwidth]{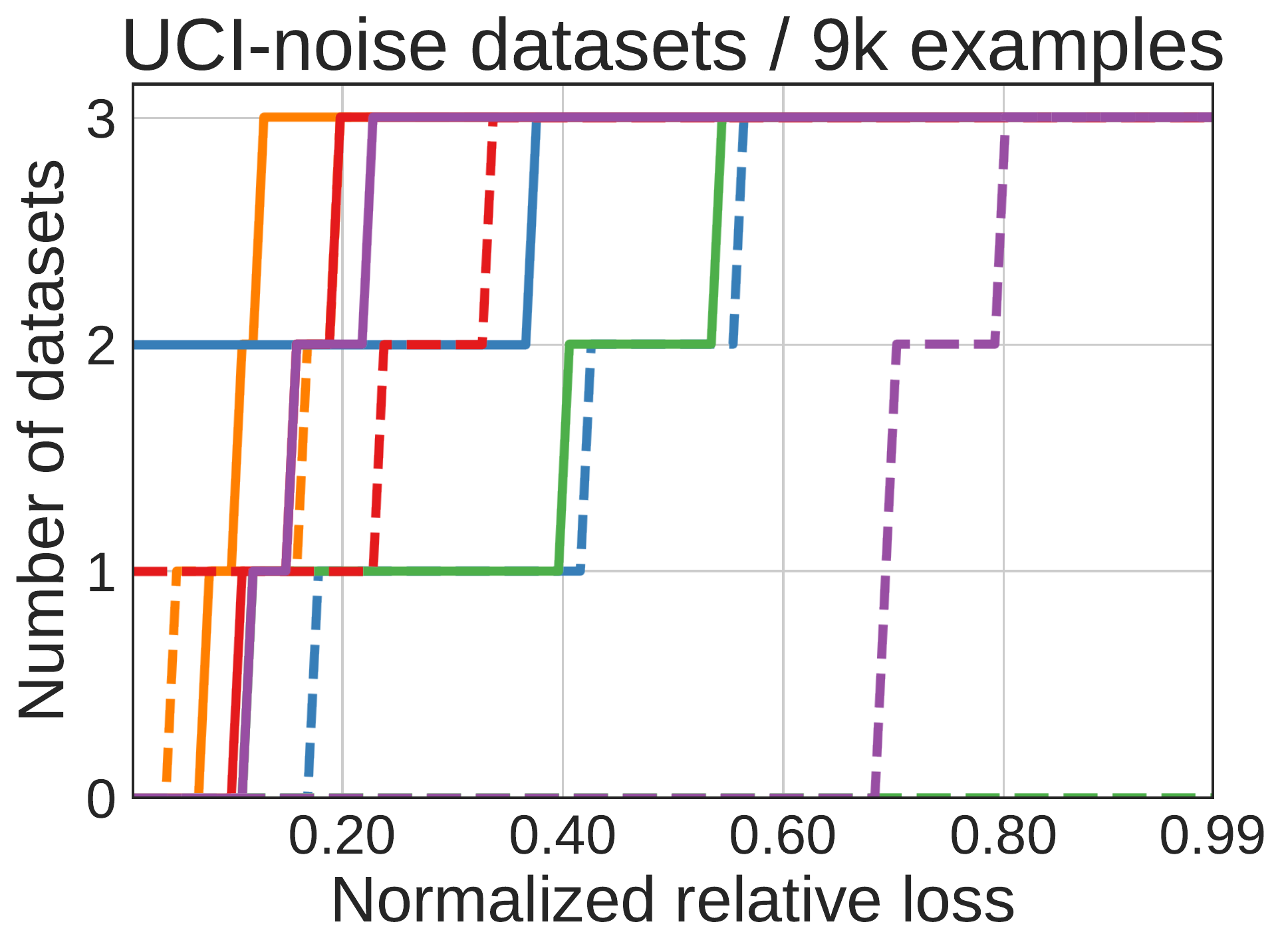}%
\includegraphics[width=0.33\textwidth]{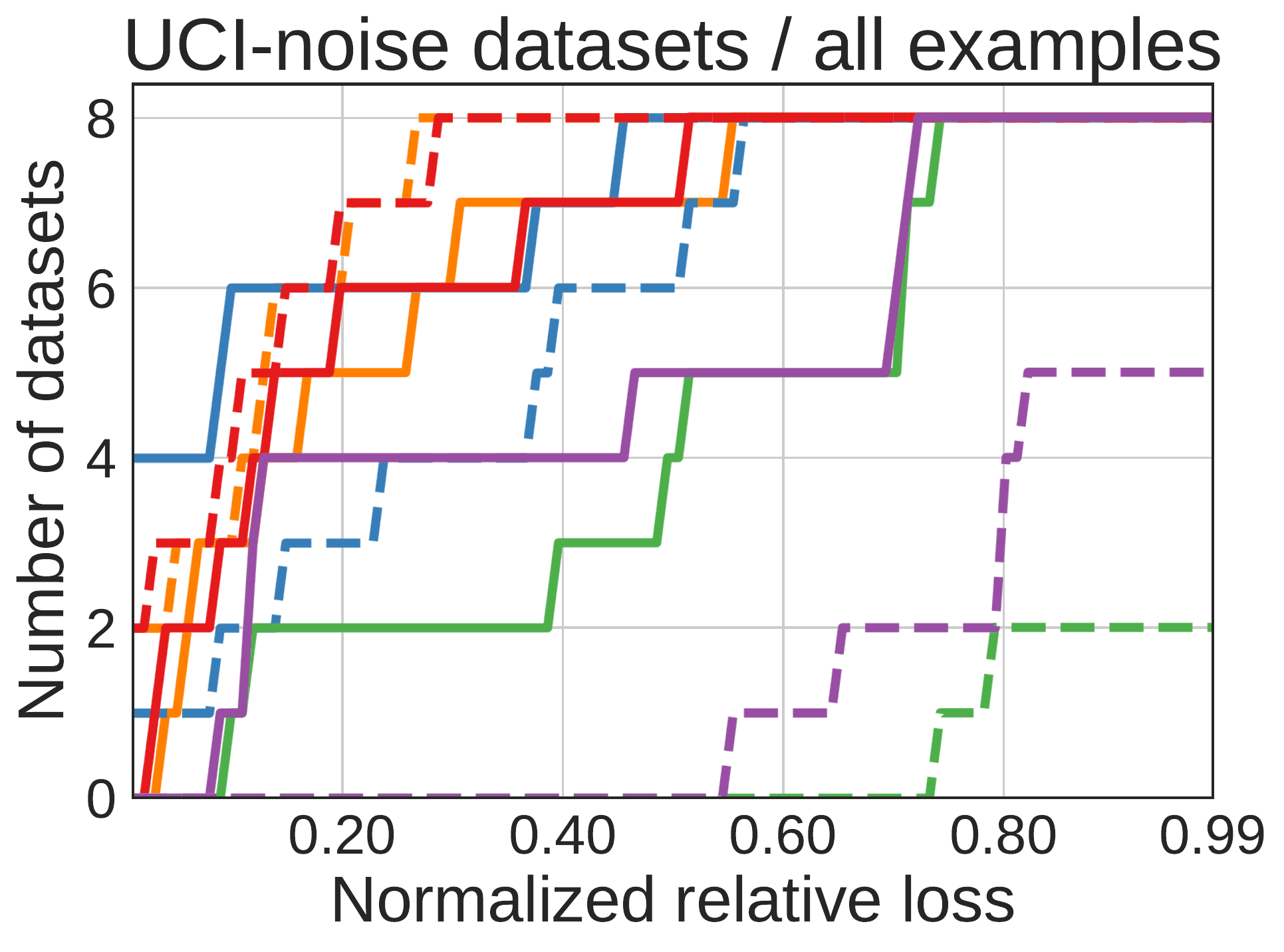}%
\end{centering}
\caption{Cumulative performance across all data sets at various sample sizes.}
\end{figure*} 

\begin{figure*}
  \begin{centering}
    ~\hfill\includegraphics[width=0.75\textwidth]{legend_supervised.pdf}\hfill~\\
\includegraphics[width=0.33\textwidth]{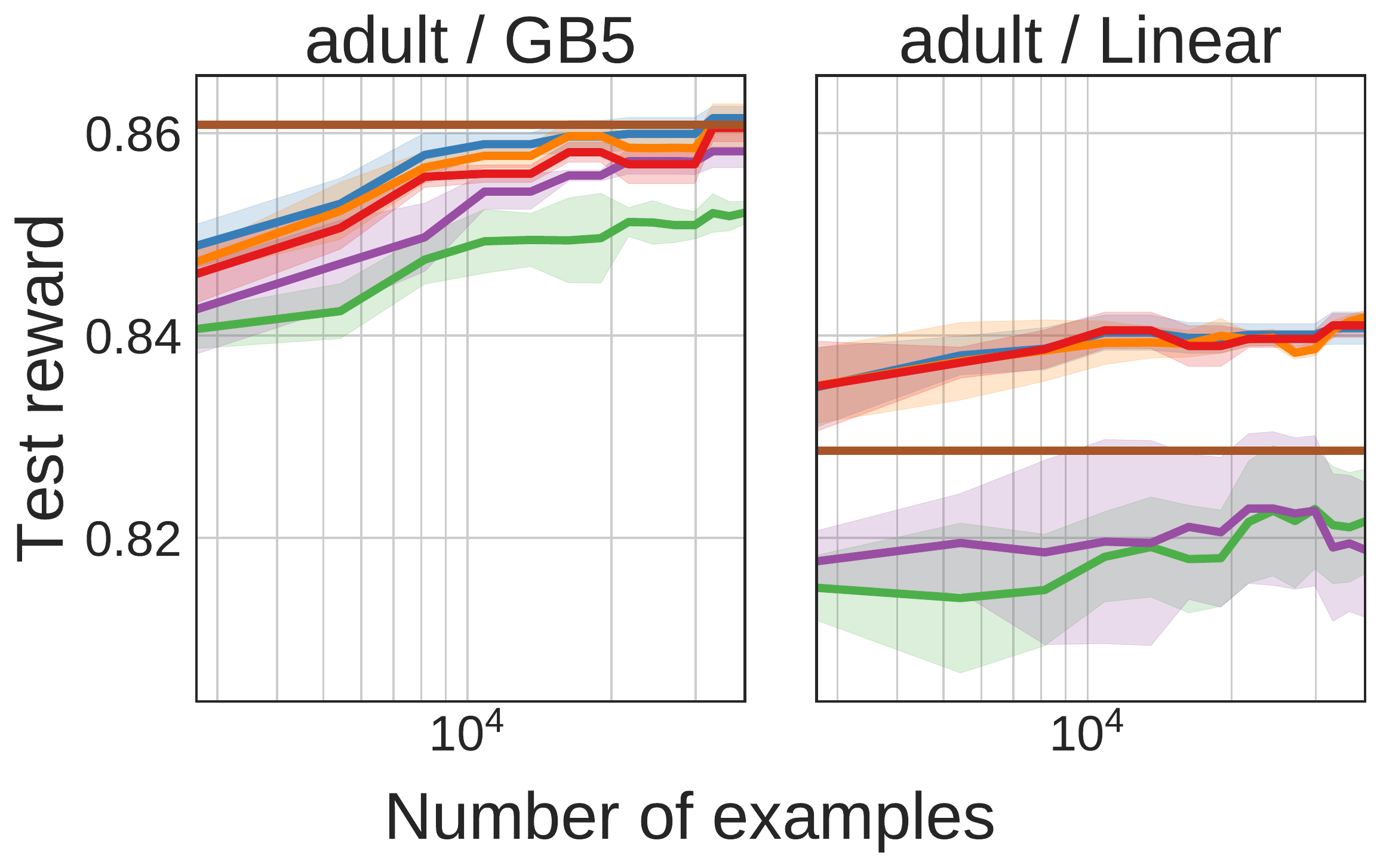}\hfill
\includegraphics[width=0.33\textwidth]{paper_plots/performance/plot_best_letter_validation.pdf}\hfill
\includegraphics[width=0.33\textwidth]{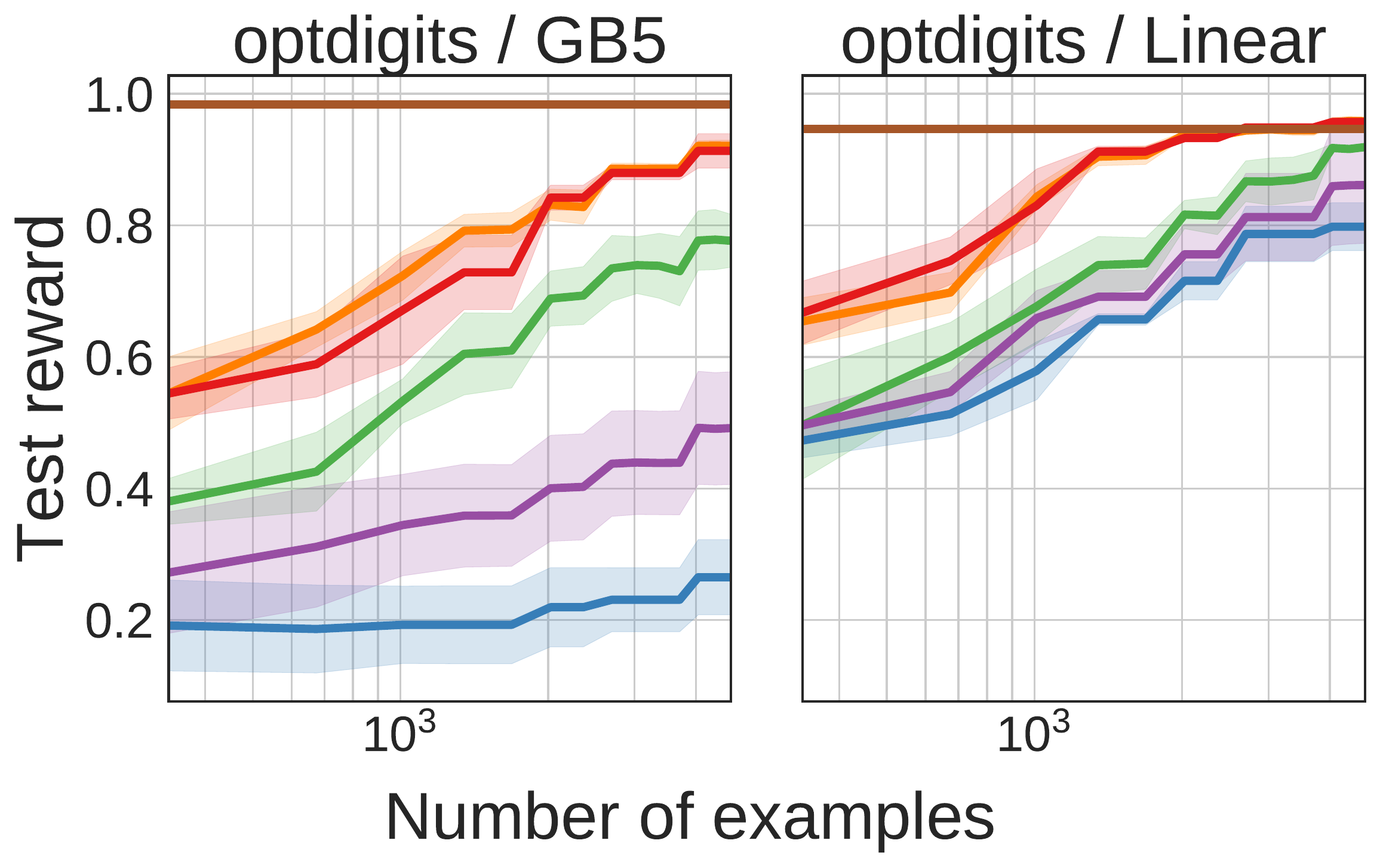}

\includegraphics[width=0.33\textwidth]{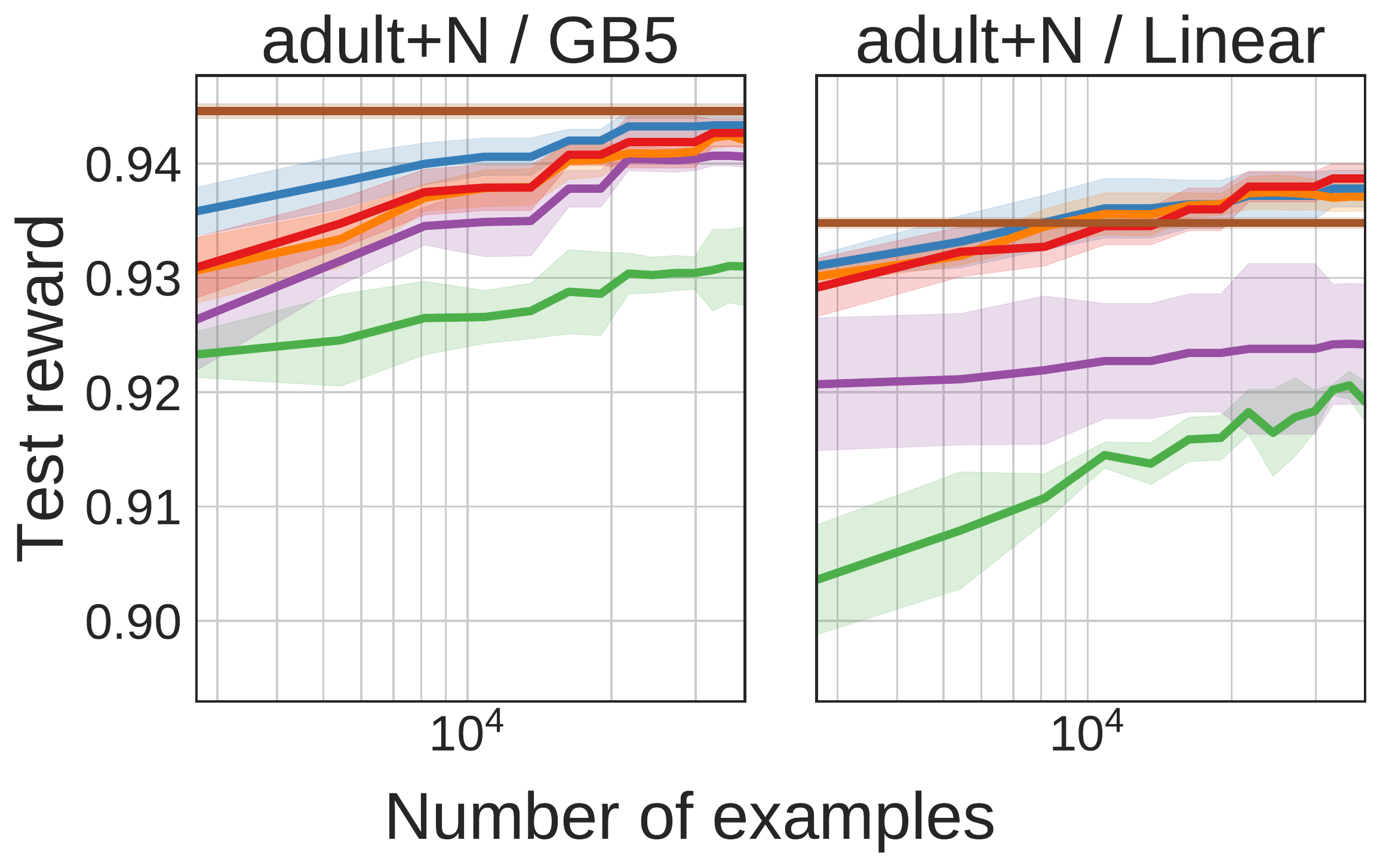}\hfill
\includegraphics[width=0.33\textwidth]{paper_plots/performance/plot_best_letter-noise_validation.pdf}\hfill
\includegraphics[width=0.33\textwidth]{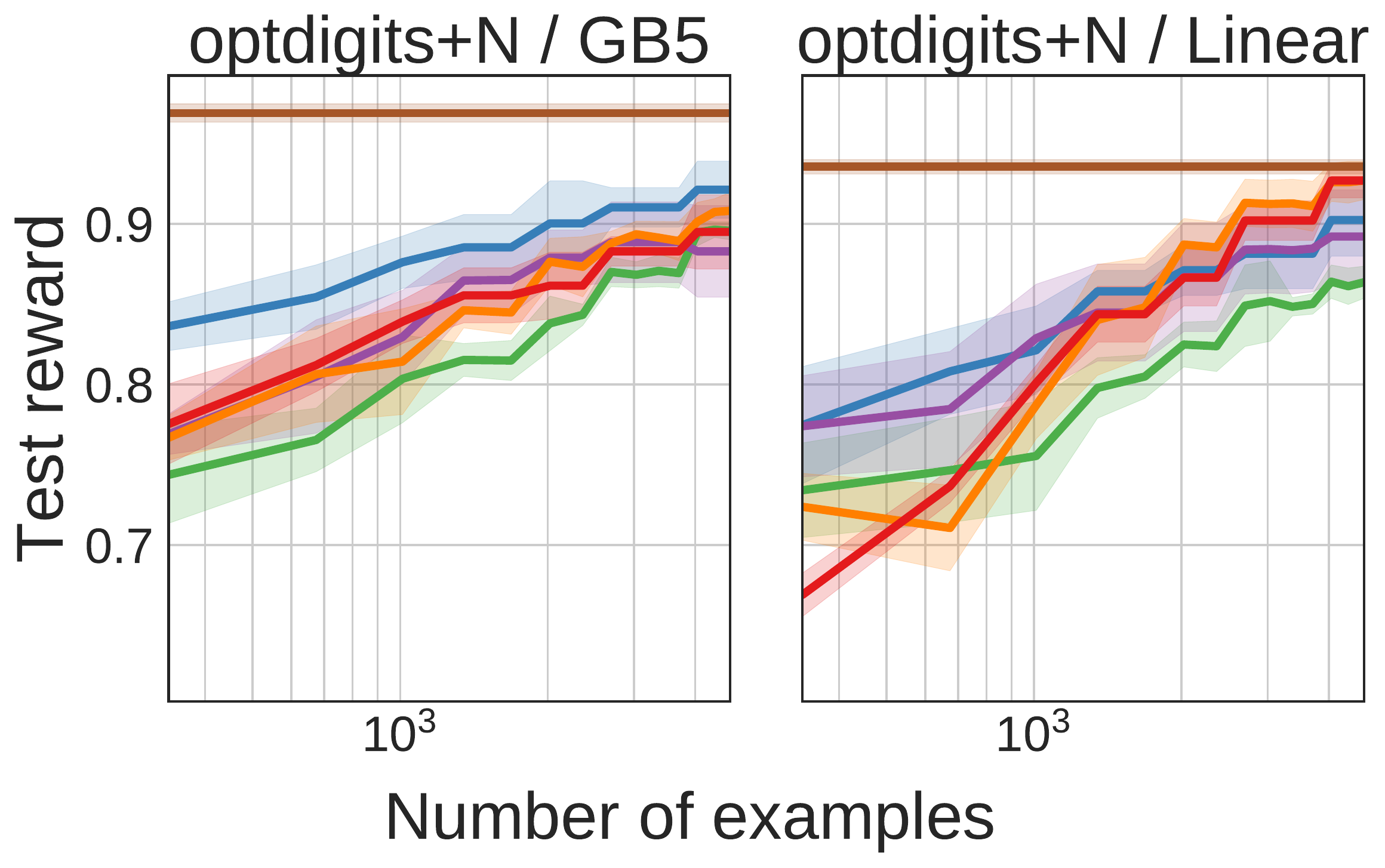}

\includegraphics[width=0.33\textwidth]{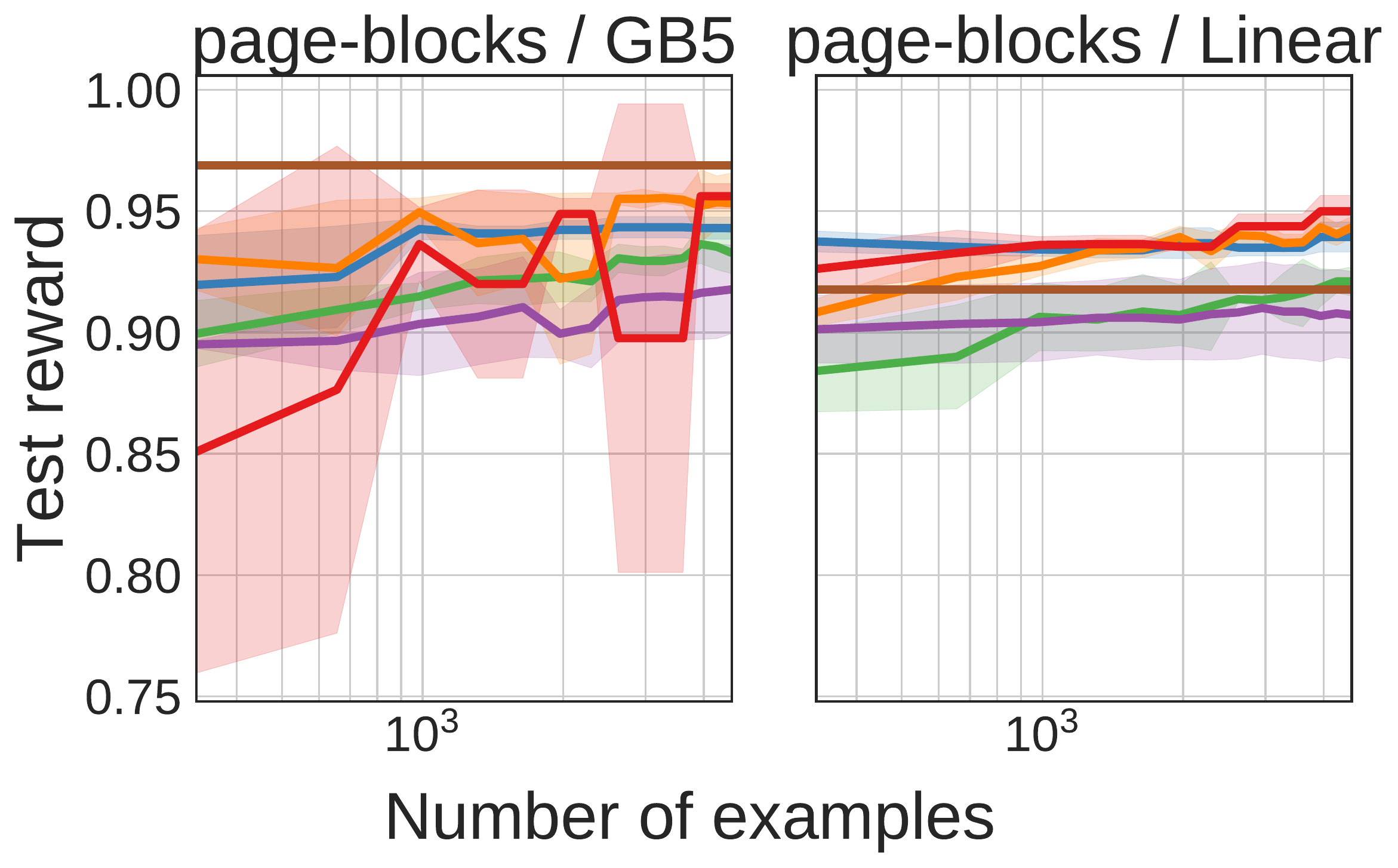}\hfill
\includegraphics[width=0.33\textwidth]{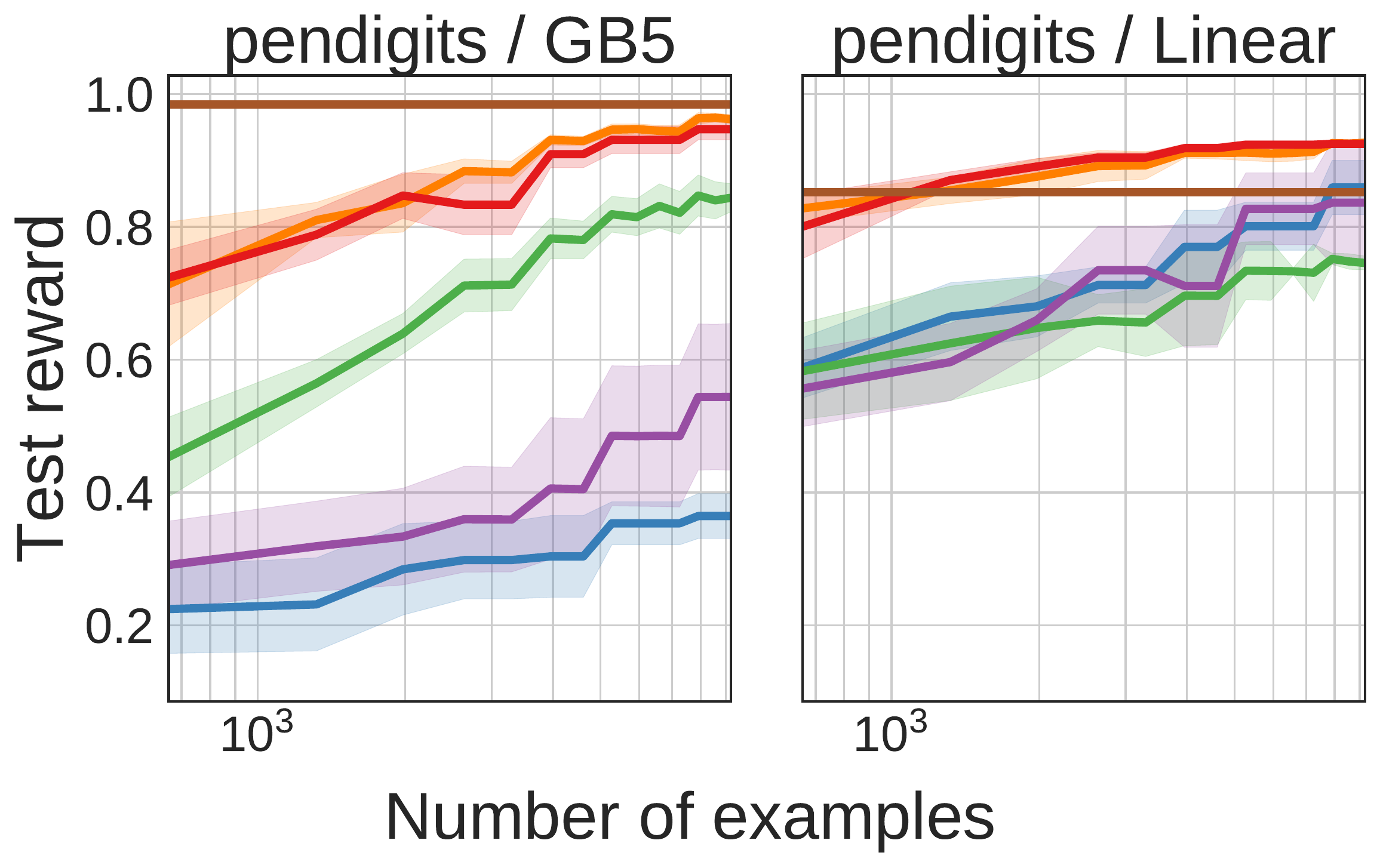}\hfill
\includegraphics[width=0.33\textwidth]{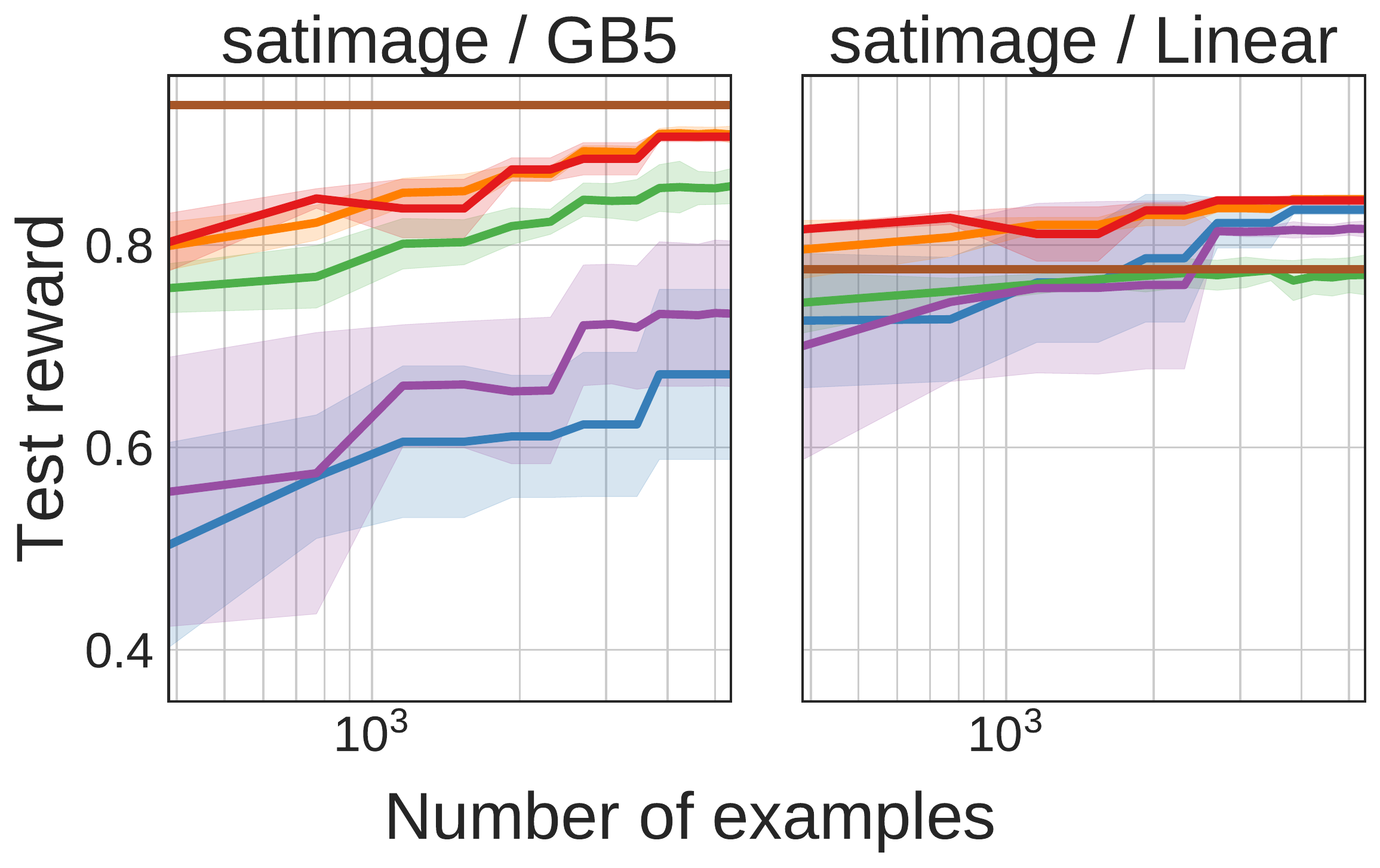}

\includegraphics[width=0.33\textwidth]{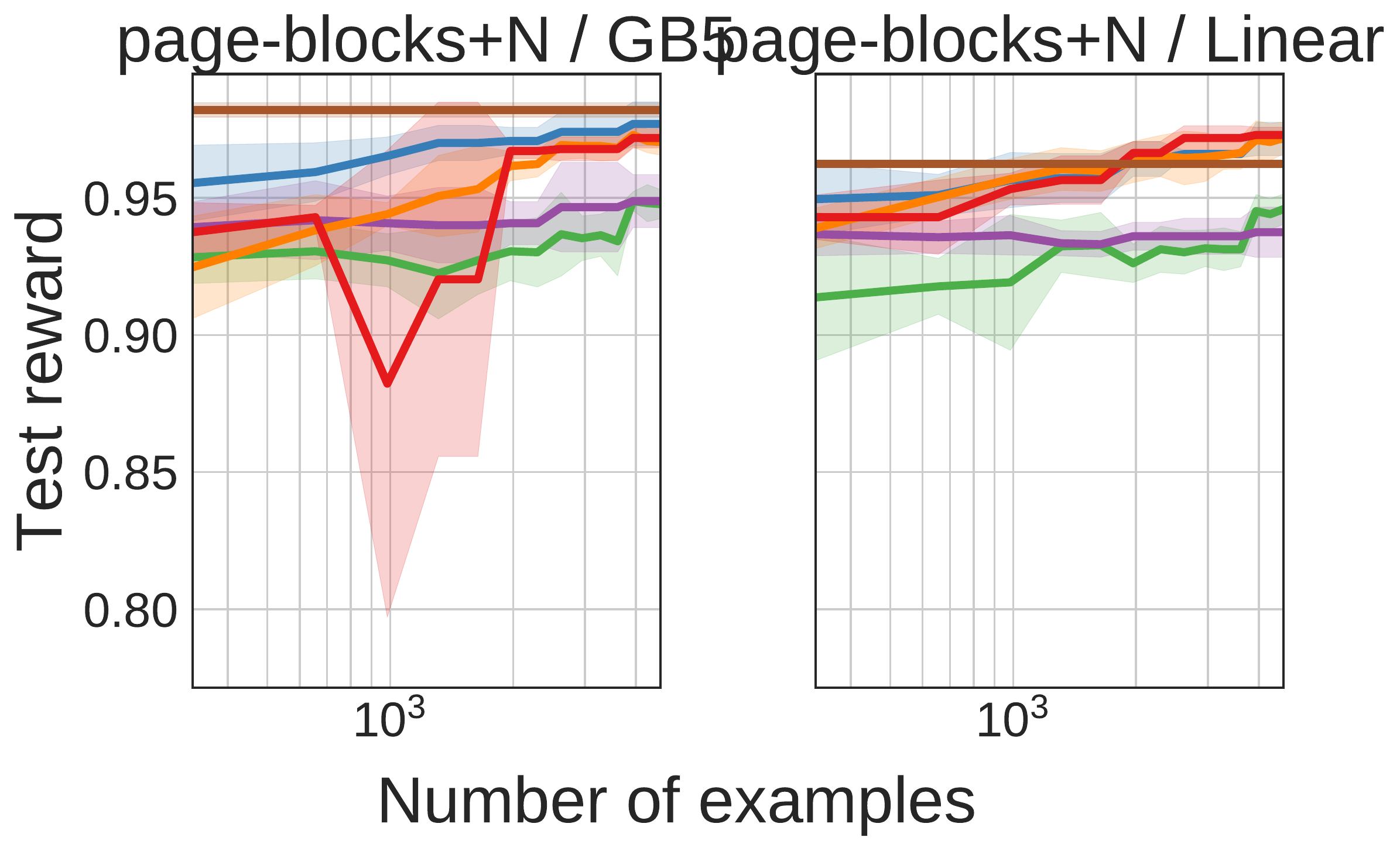}\hfill
\includegraphics[width=0.33\textwidth]{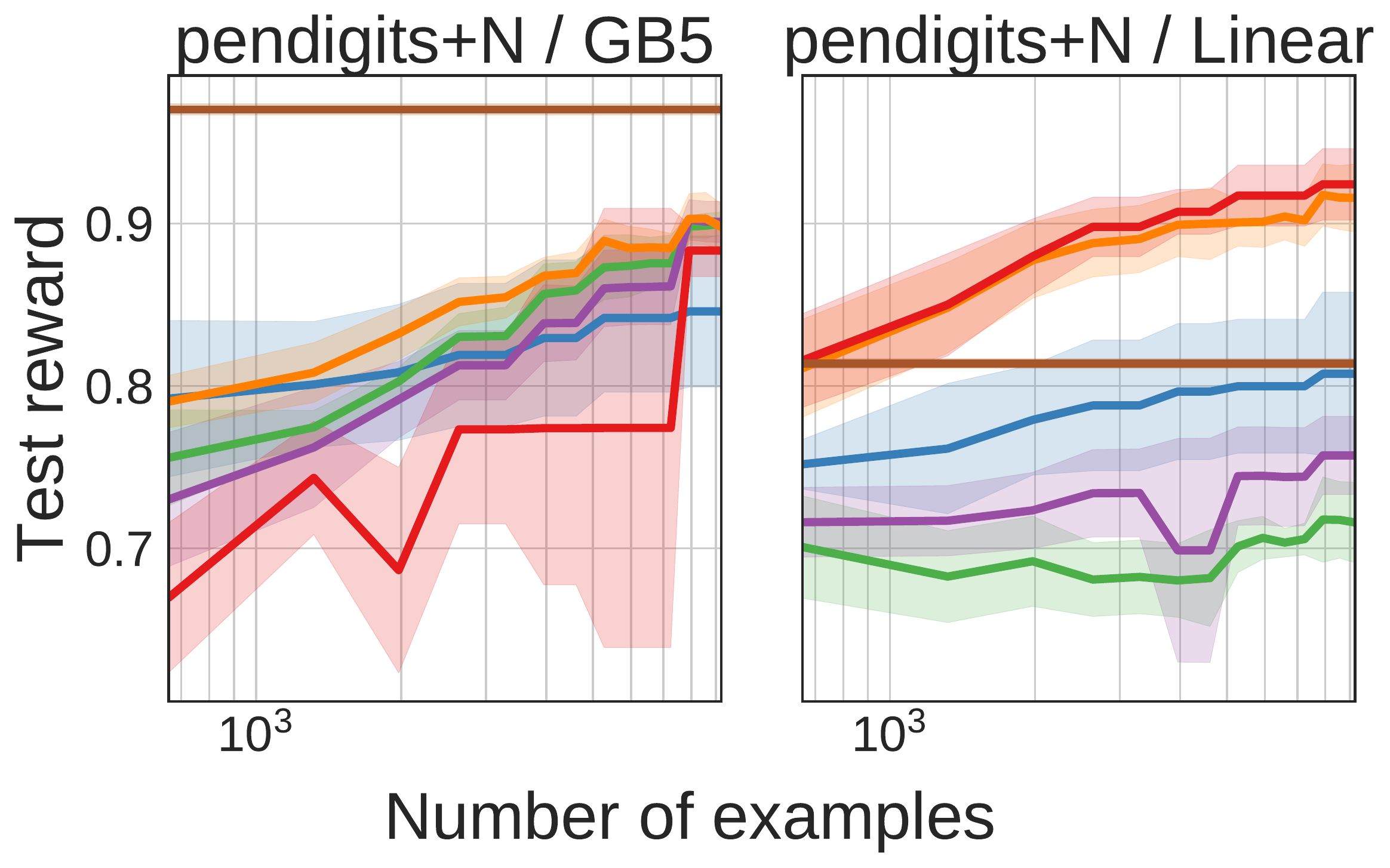}\hfill
\includegraphics[width=0.33\textwidth]{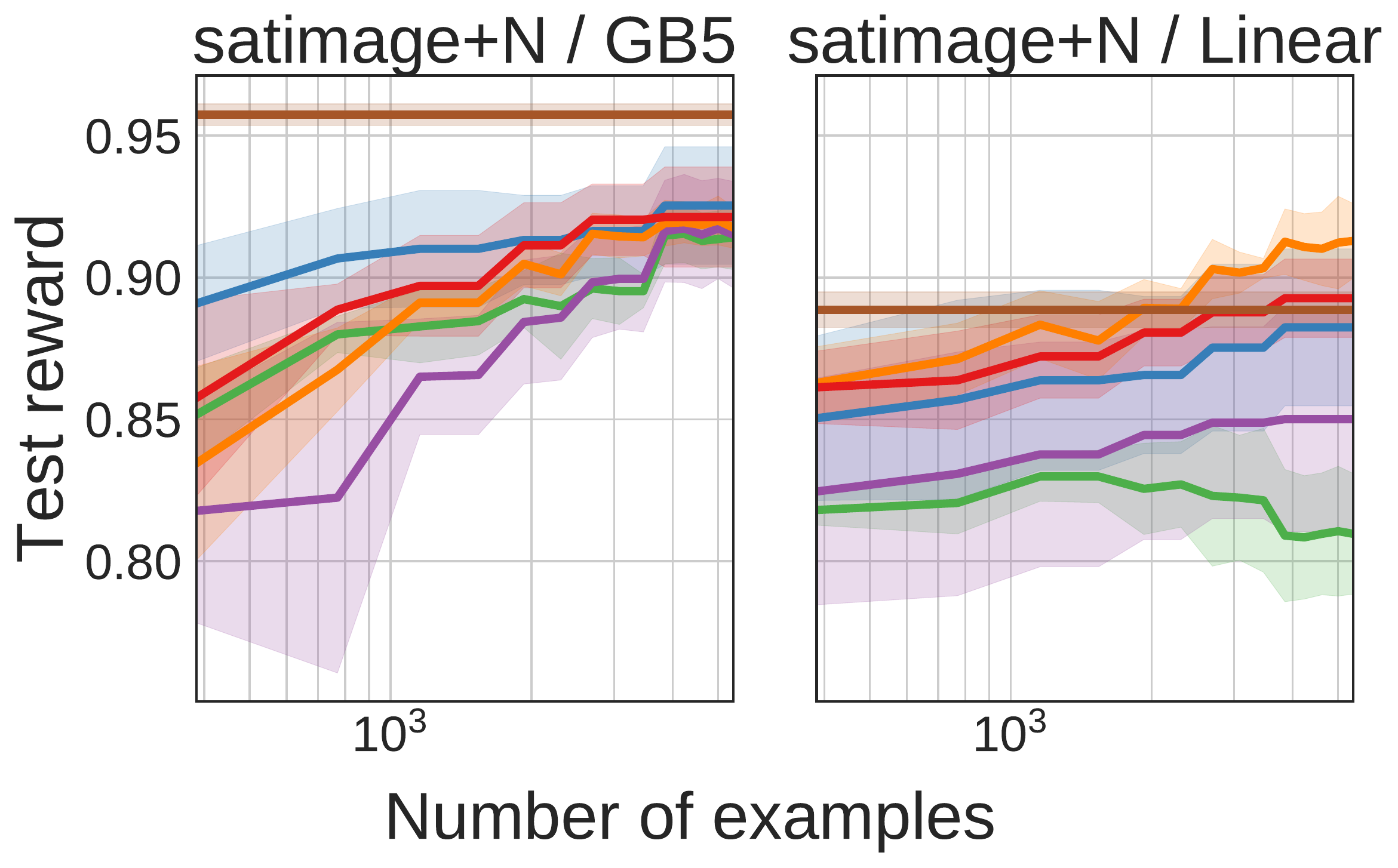}

\includegraphics[width=0.33\textwidth]{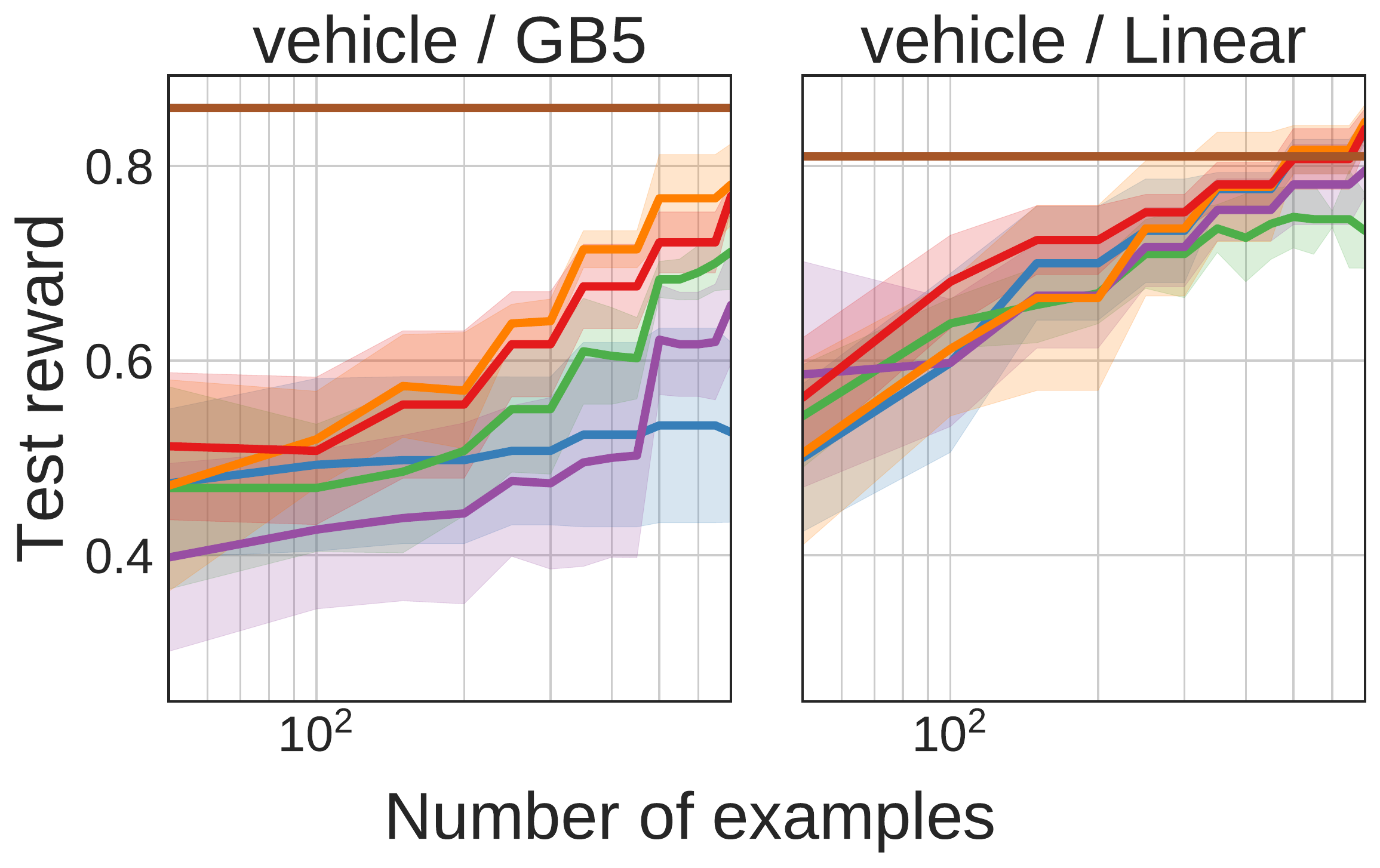}\hfill
\includegraphics[width=0.33\textwidth]{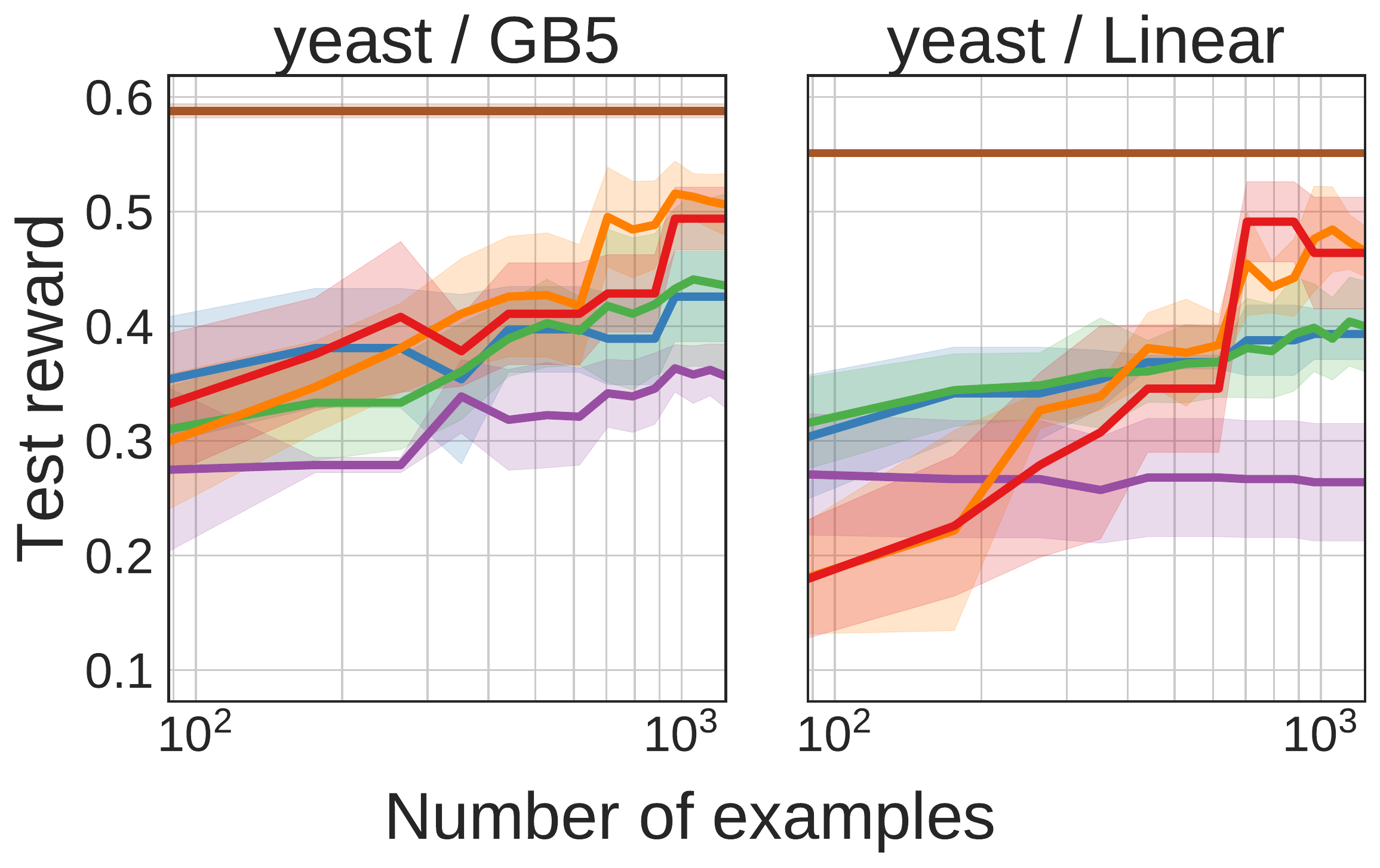}\hfill
\includegraphics[width=0.33\textwidth]{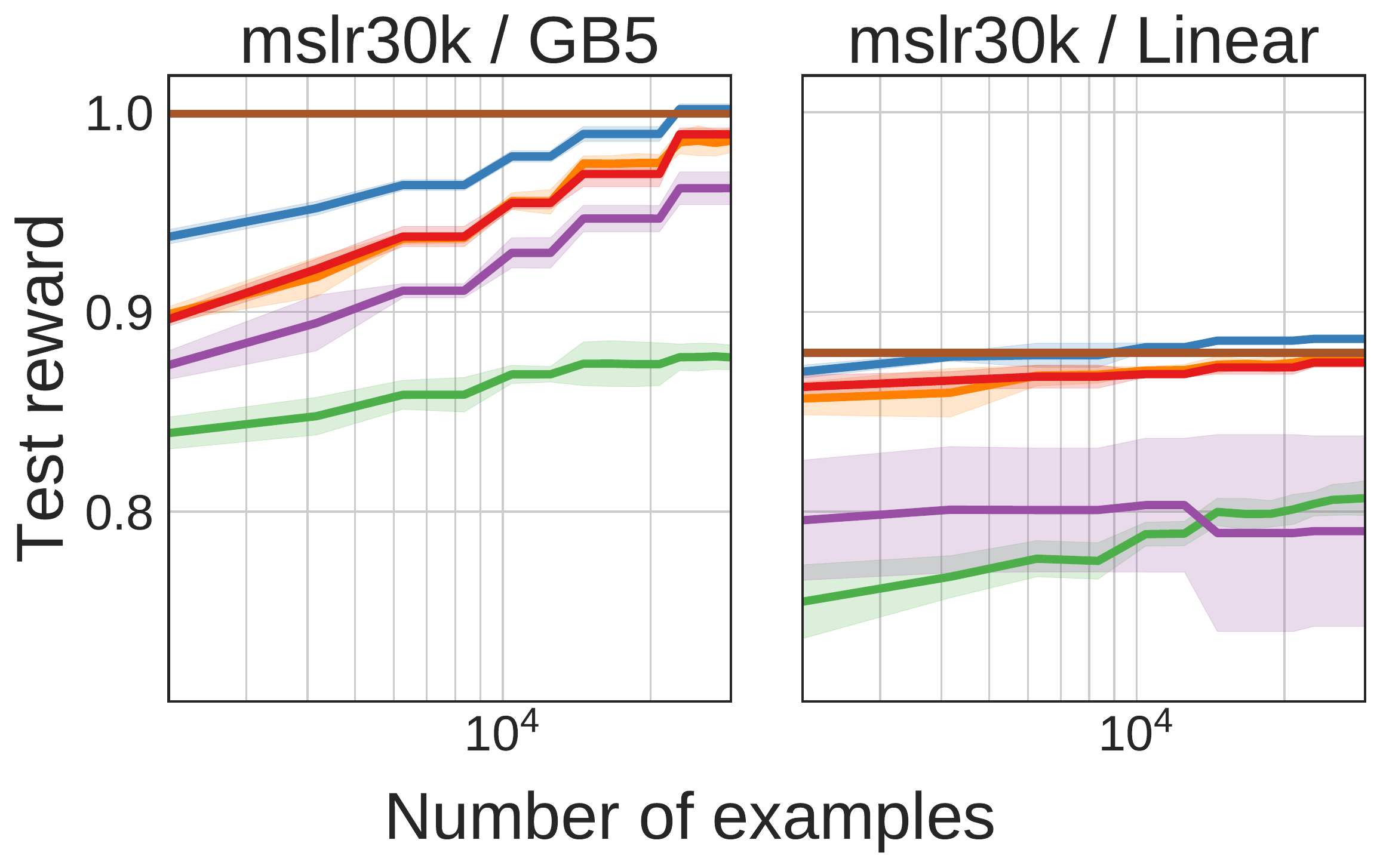}

\includegraphics[width=0.33\textwidth]{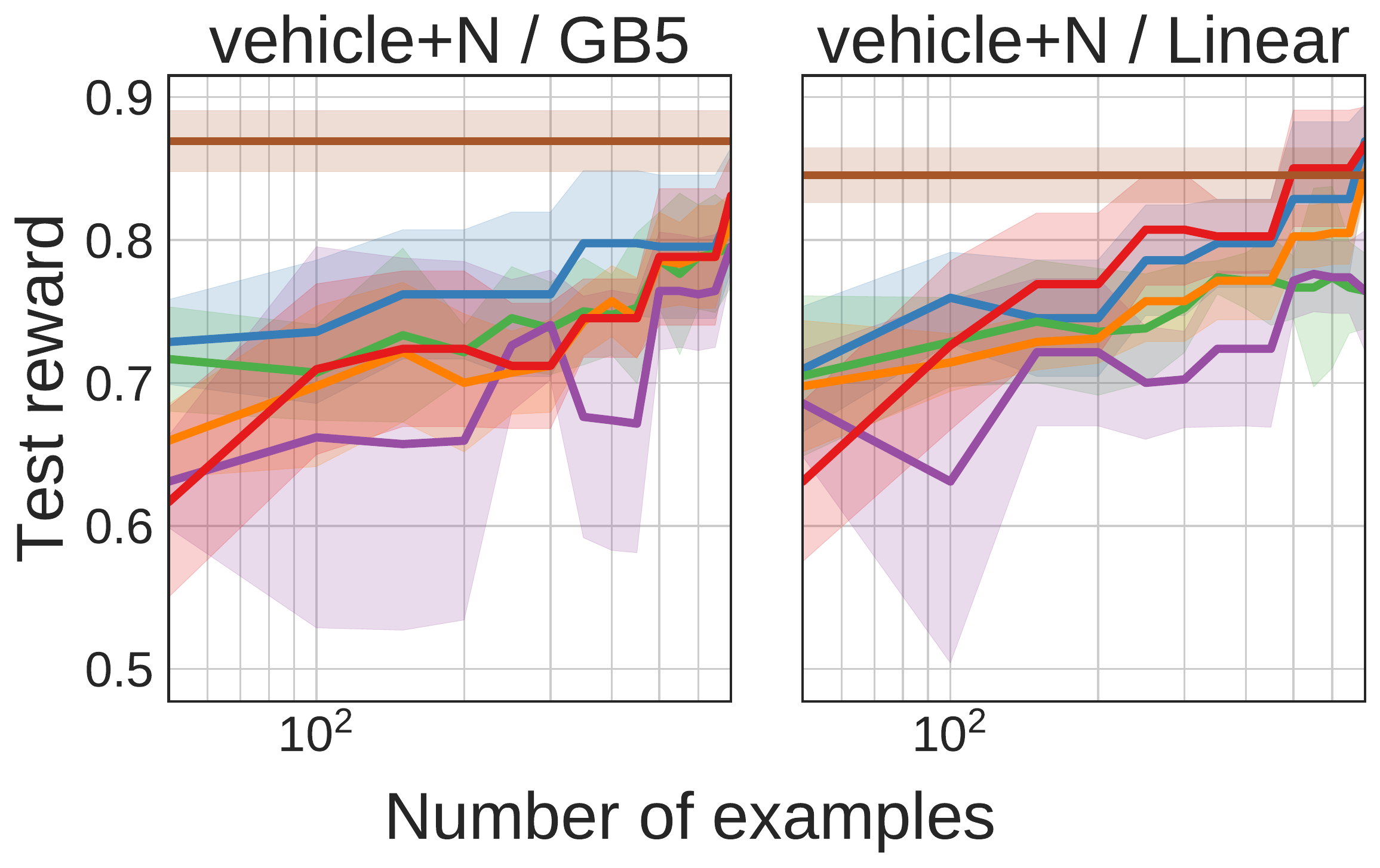}\hfill
\includegraphics[width=0.33\textwidth]{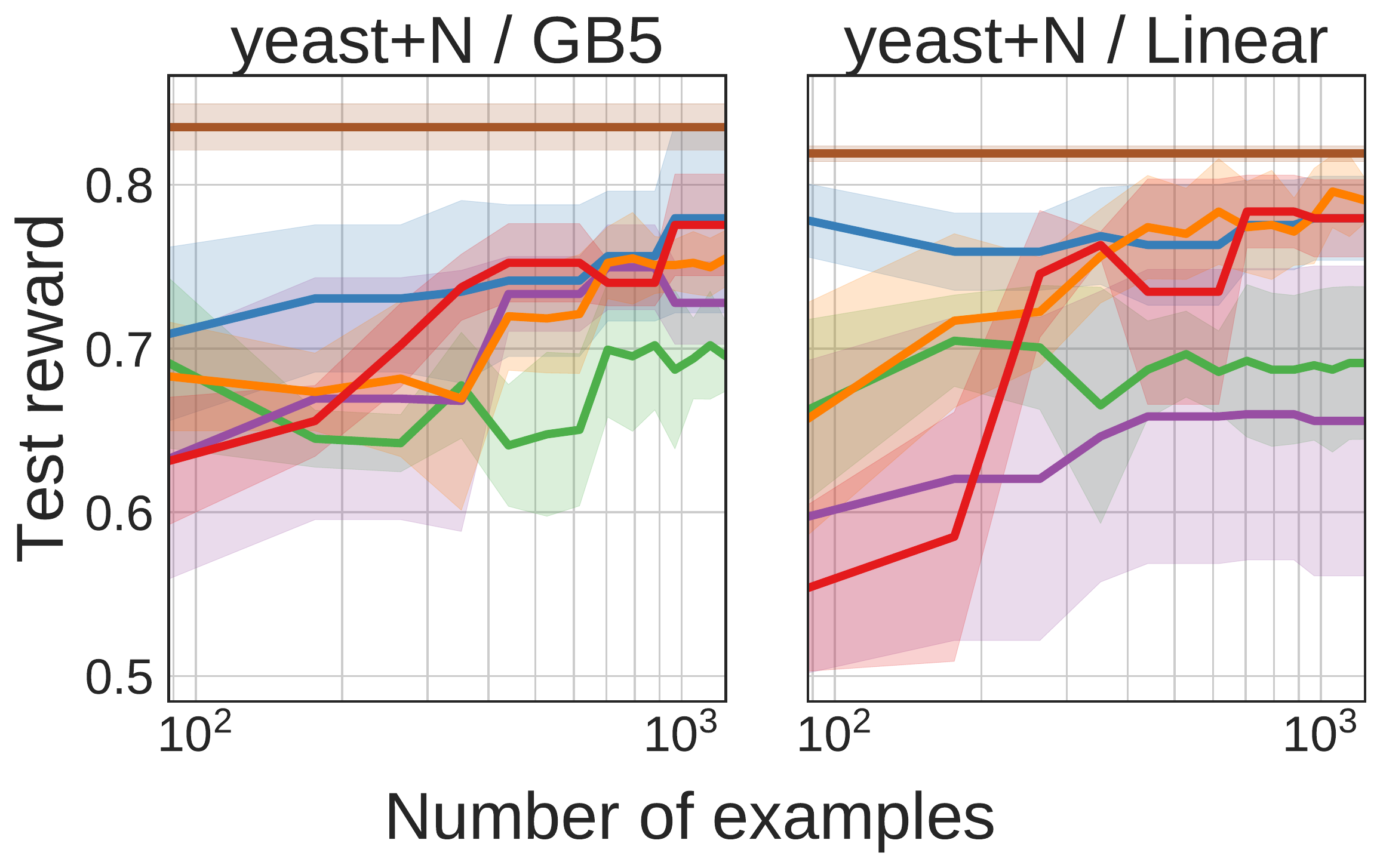}\hfill
\includegraphics[width=0.33\textwidth]{paper_plots/performance/plot_best_yahoo_validation.pdf}

\end{centering}
\caption{Performance on individual datasets.}
\end{figure*}

\begin{figure*}
  \begin{centering}
    ~\hfill\includegraphics[width=0.26\textwidth]{legend_width.pdf}\hfill~\\
    \medskip
\includegraphics[width=0.33\textwidth]{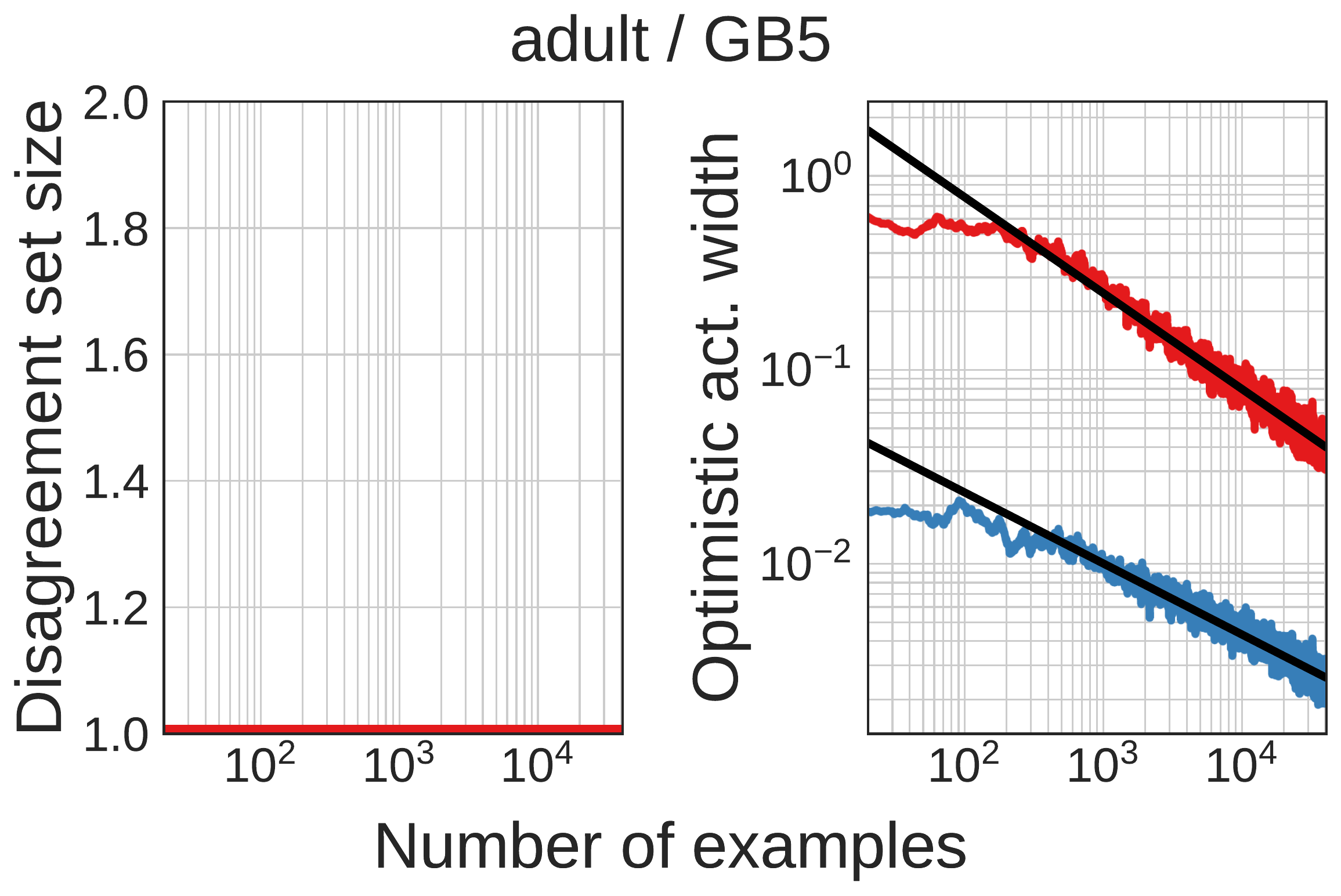}\hfill
\includegraphics[width=0.33\textwidth]{paper_plots/width/plot_disagreement_uci-letter.pdf}\hfill
\includegraphics[width=0.33\textwidth]{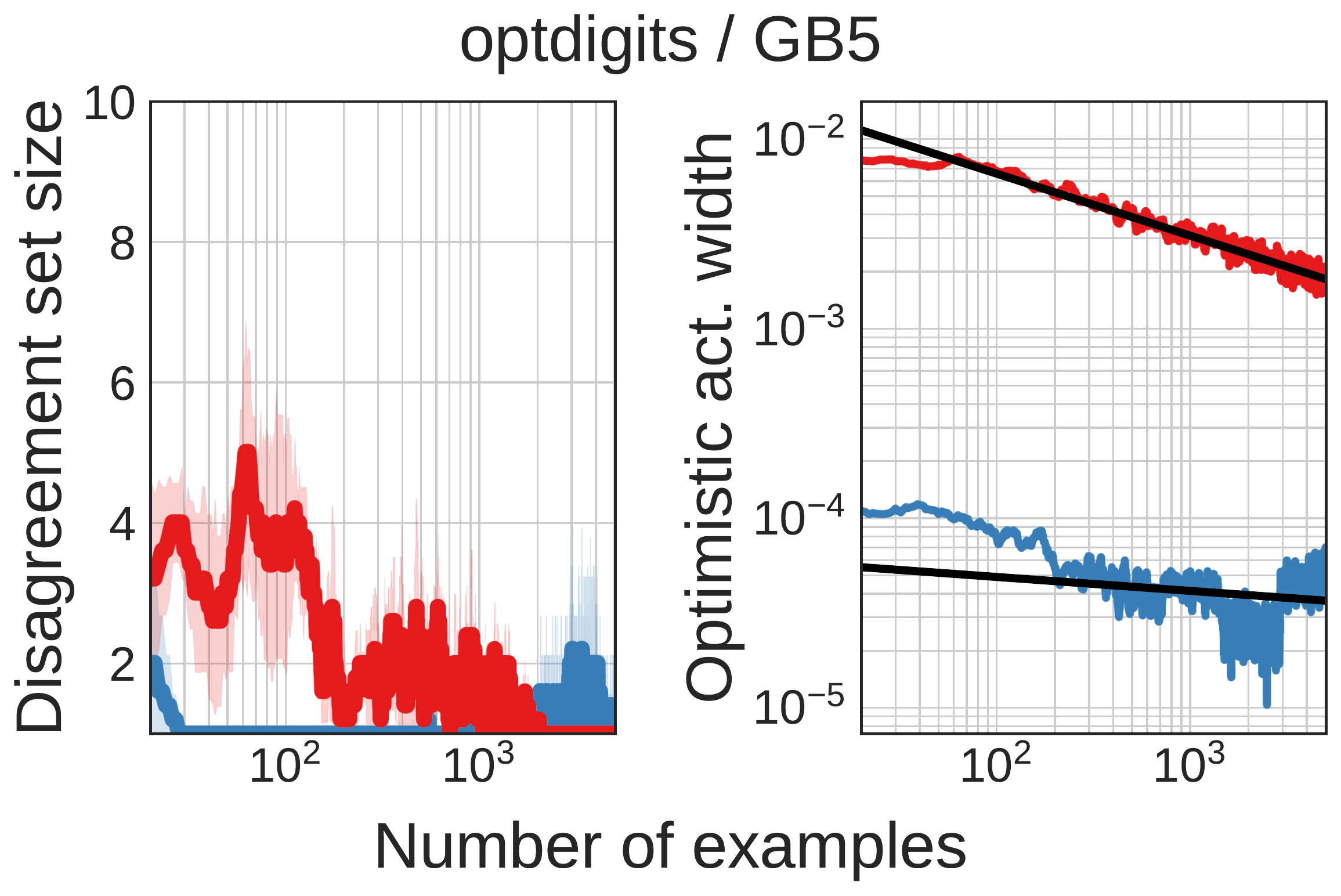}

\includegraphics[width=0.33\textwidth]{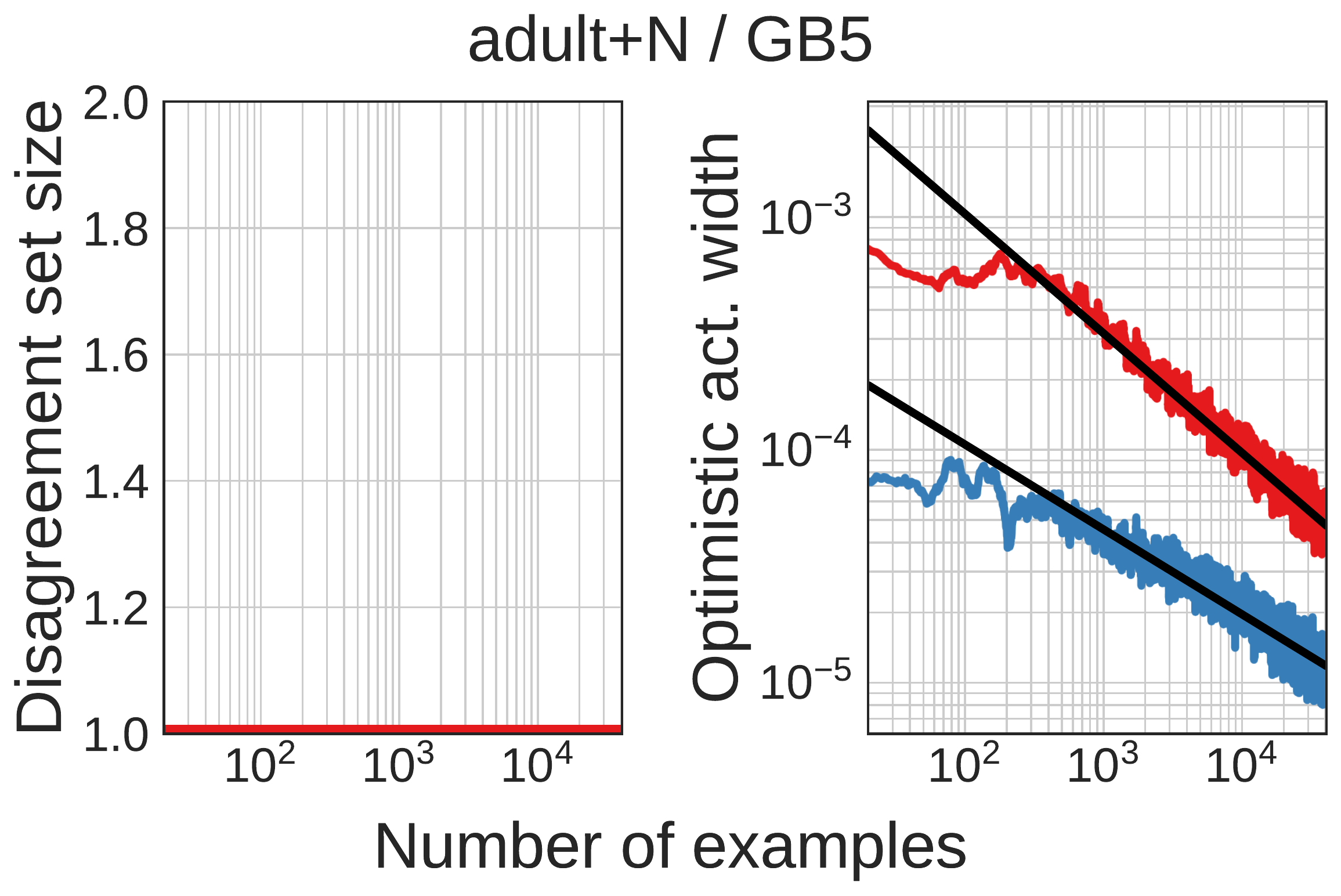}\hfill
\includegraphics[width=0.33\textwidth]{paper_plots/width/plot_disagreement_uci-letter-noise.pdf}\hfill
\includegraphics[width=0.33\textwidth]{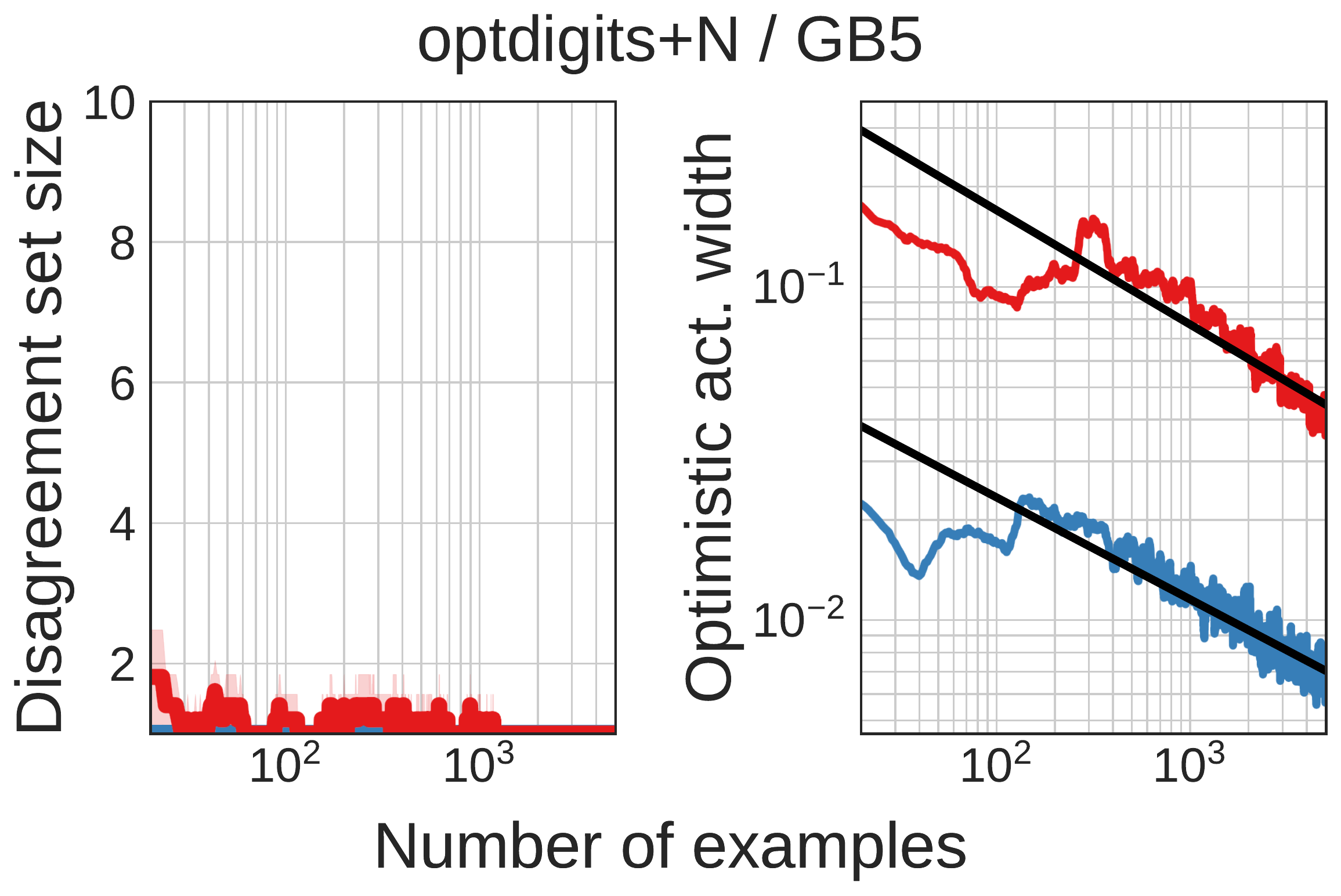}

\includegraphics[width=0.33\textwidth]{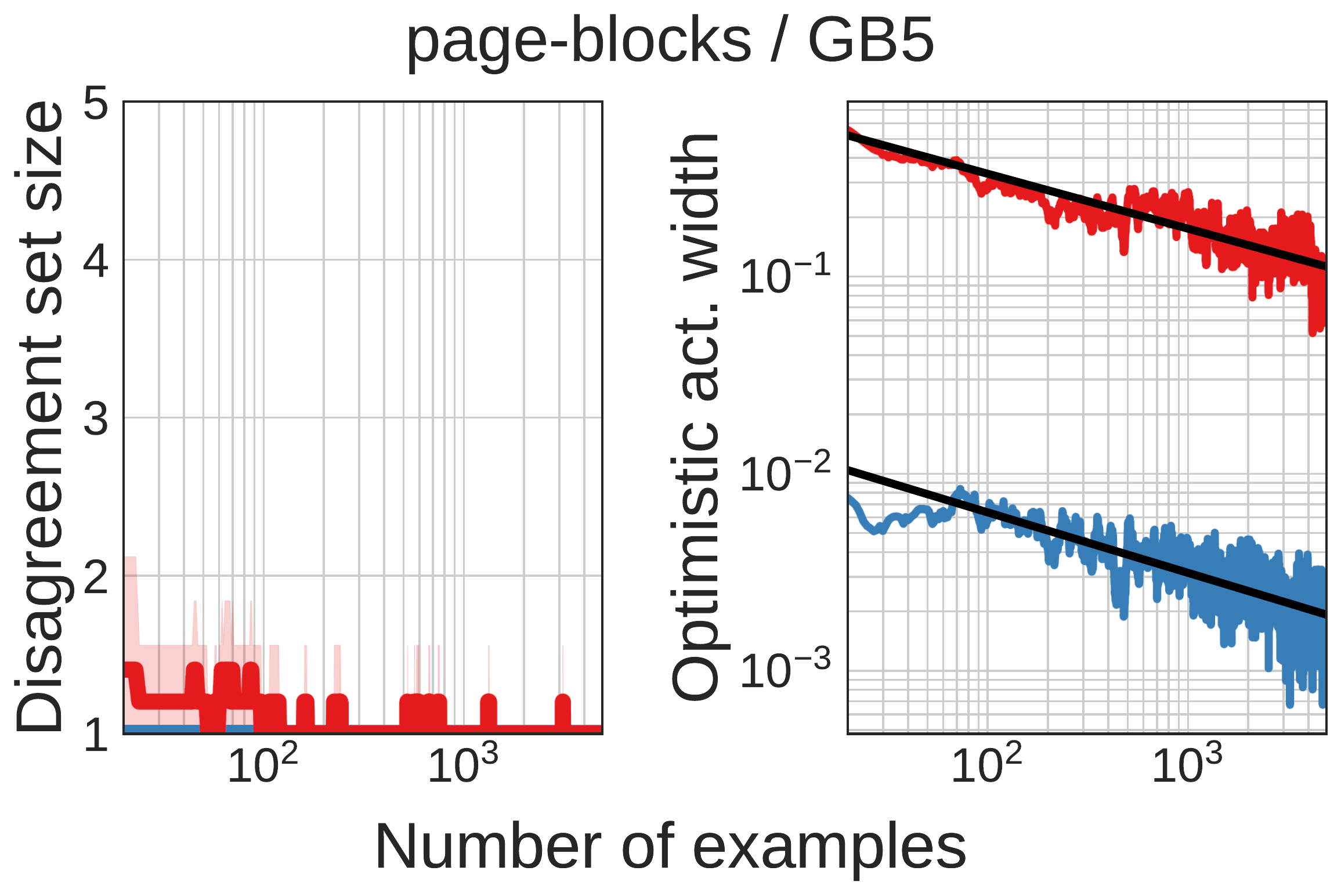}\hfill
\includegraphics[width=0.33\textwidth]{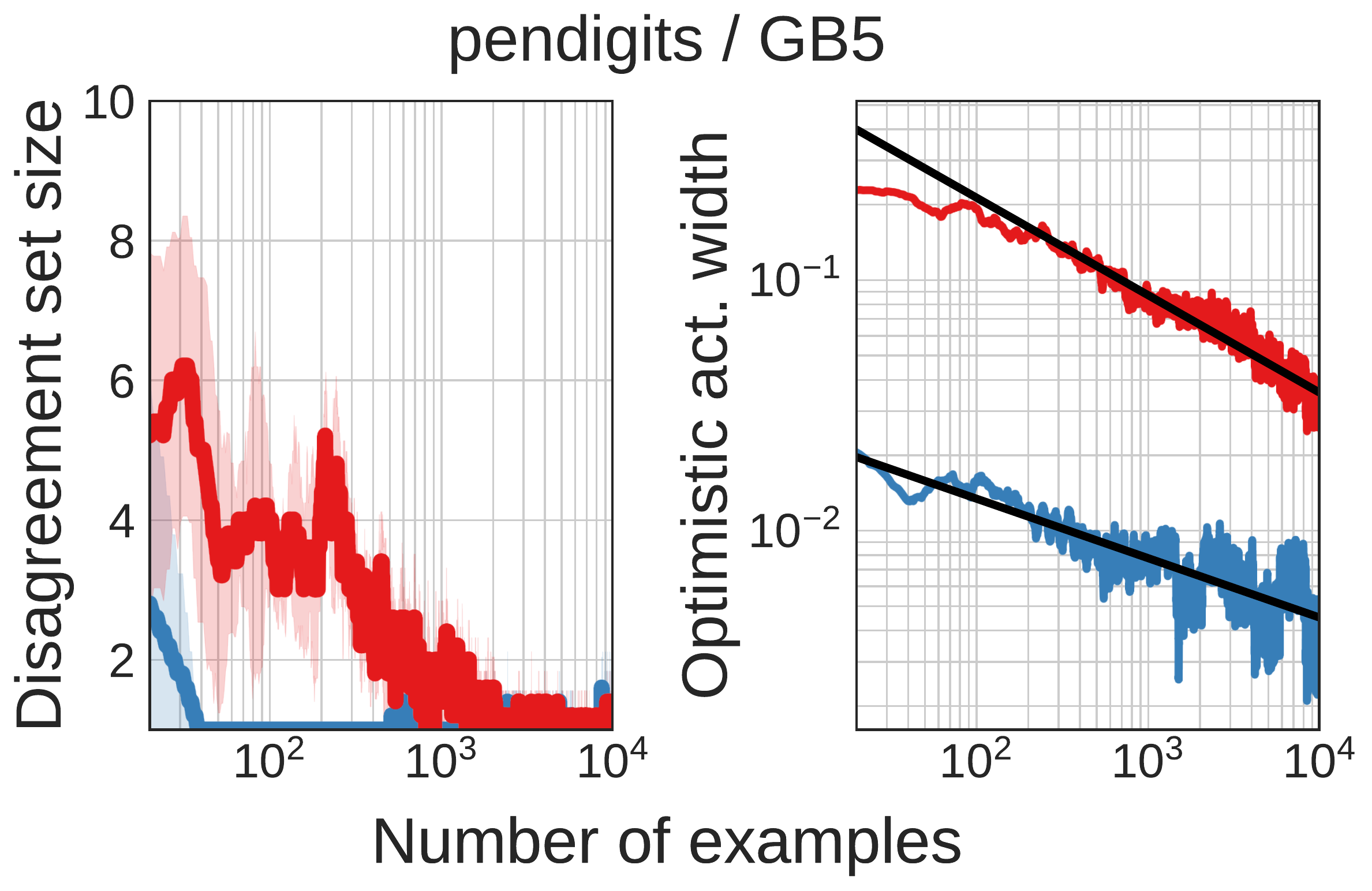}\hfill
\includegraphics[width=0.33\textwidth]{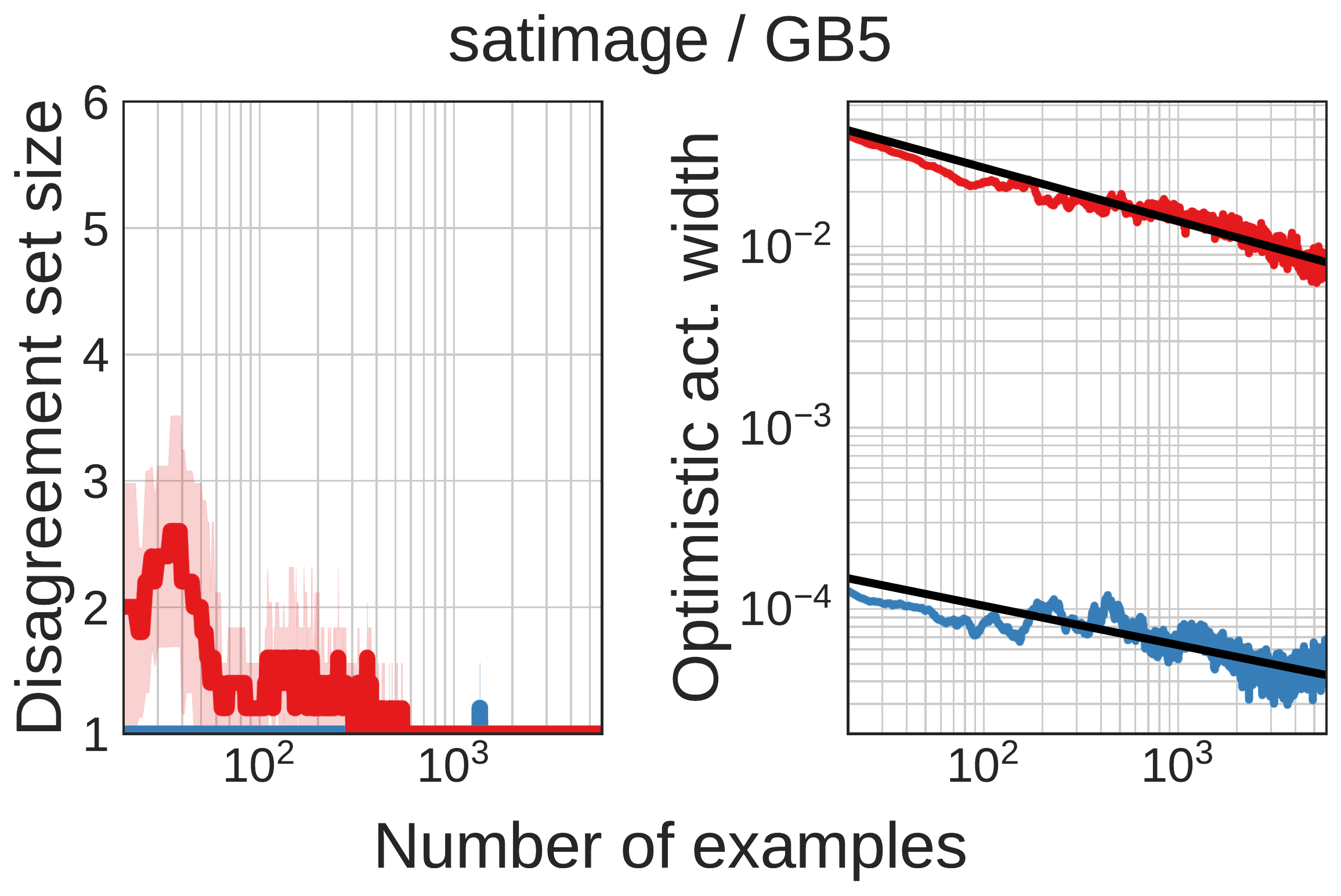}

\includegraphics[width=0.33\textwidth]{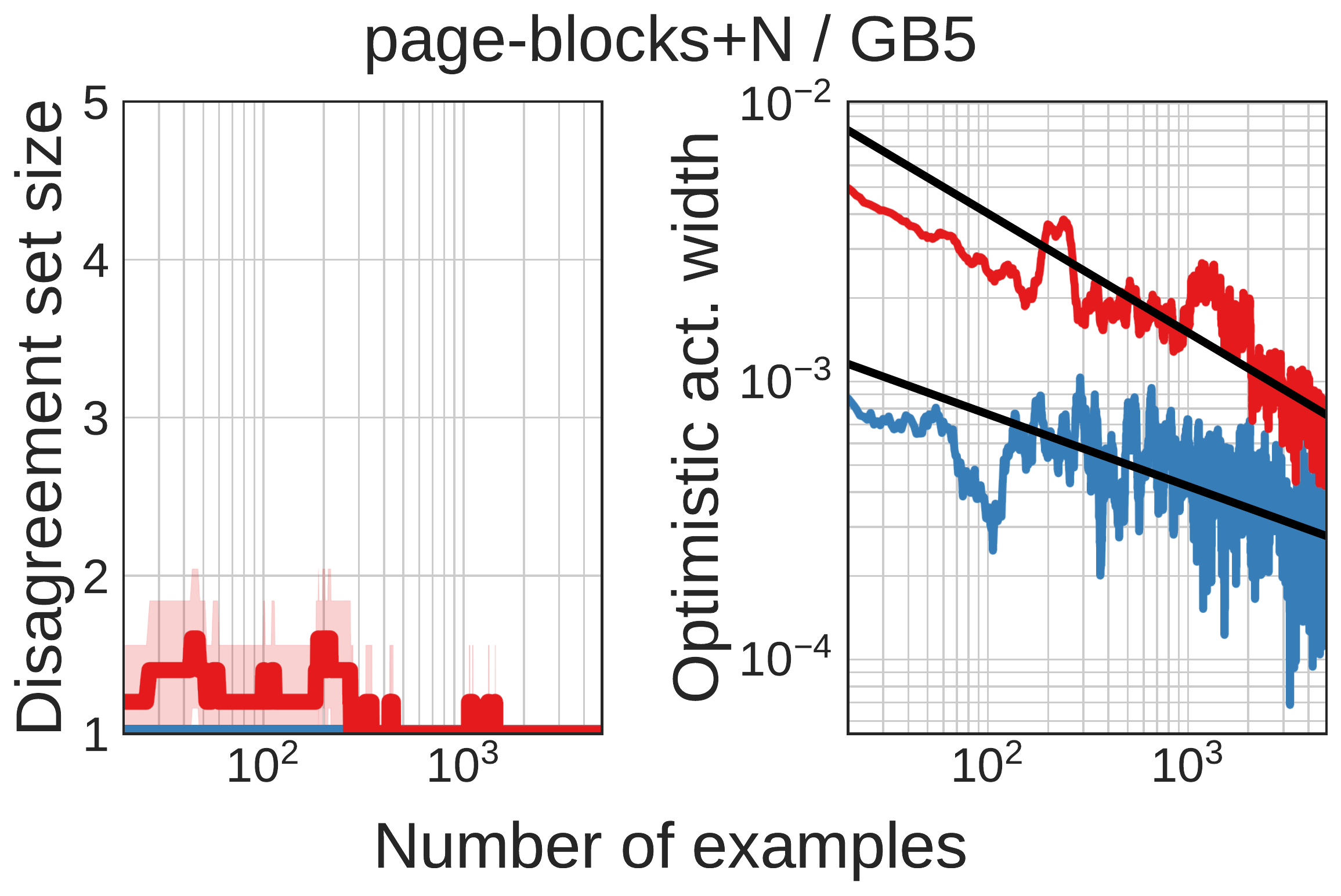}\hfill
\includegraphics[width=0.33\textwidth]{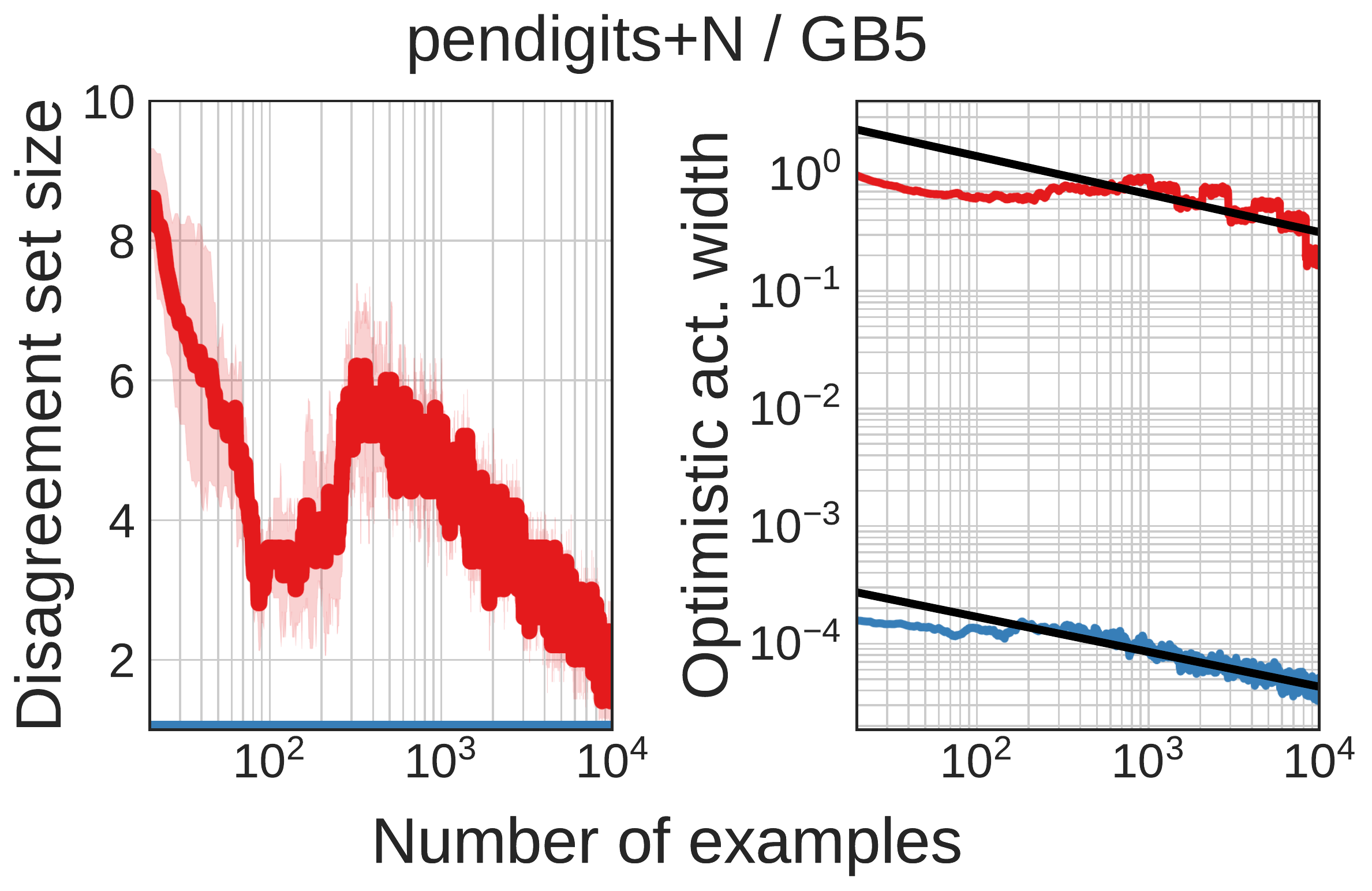}\hfill
\includegraphics[width=0.33\textwidth]{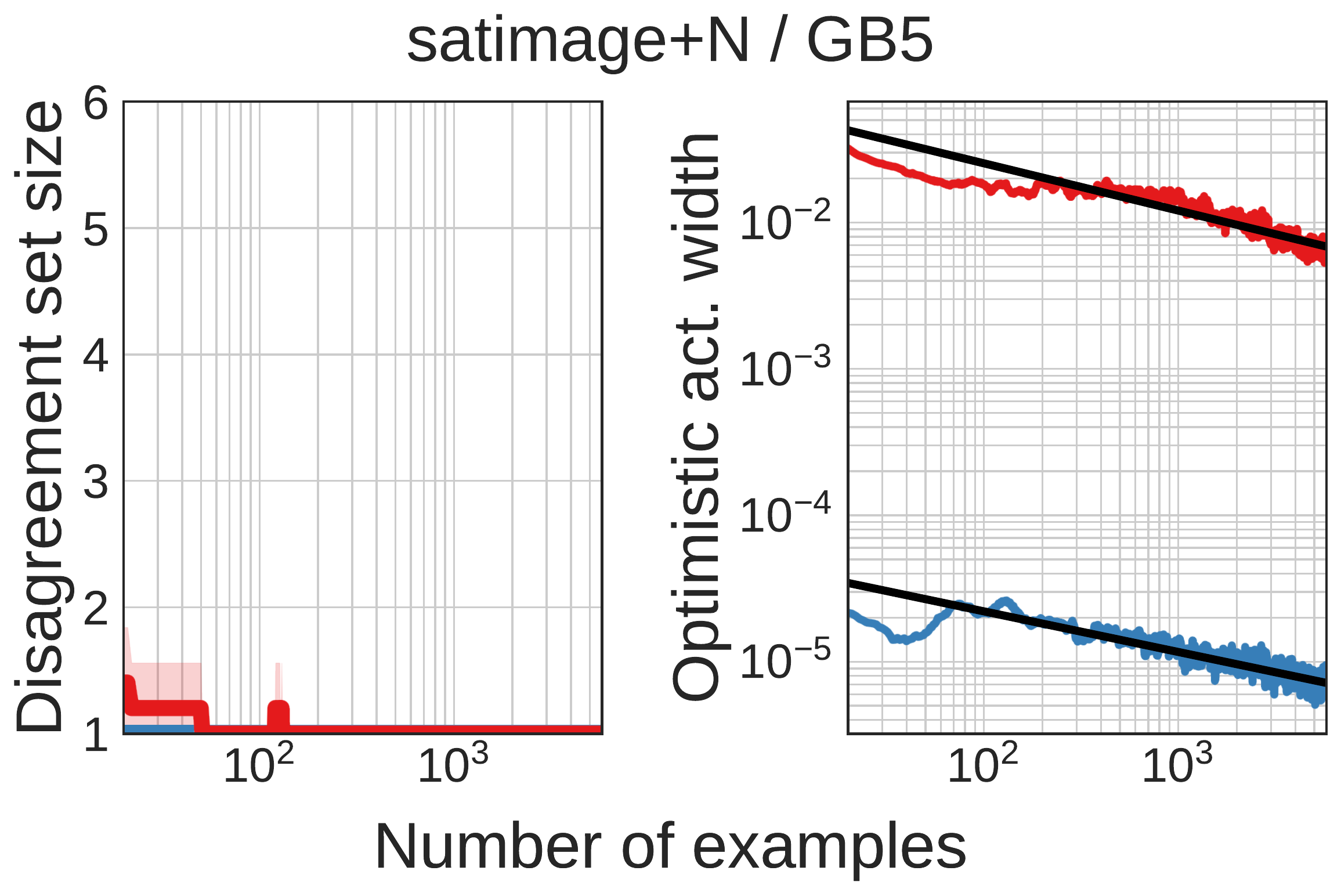}

\includegraphics[width=0.33\textwidth]{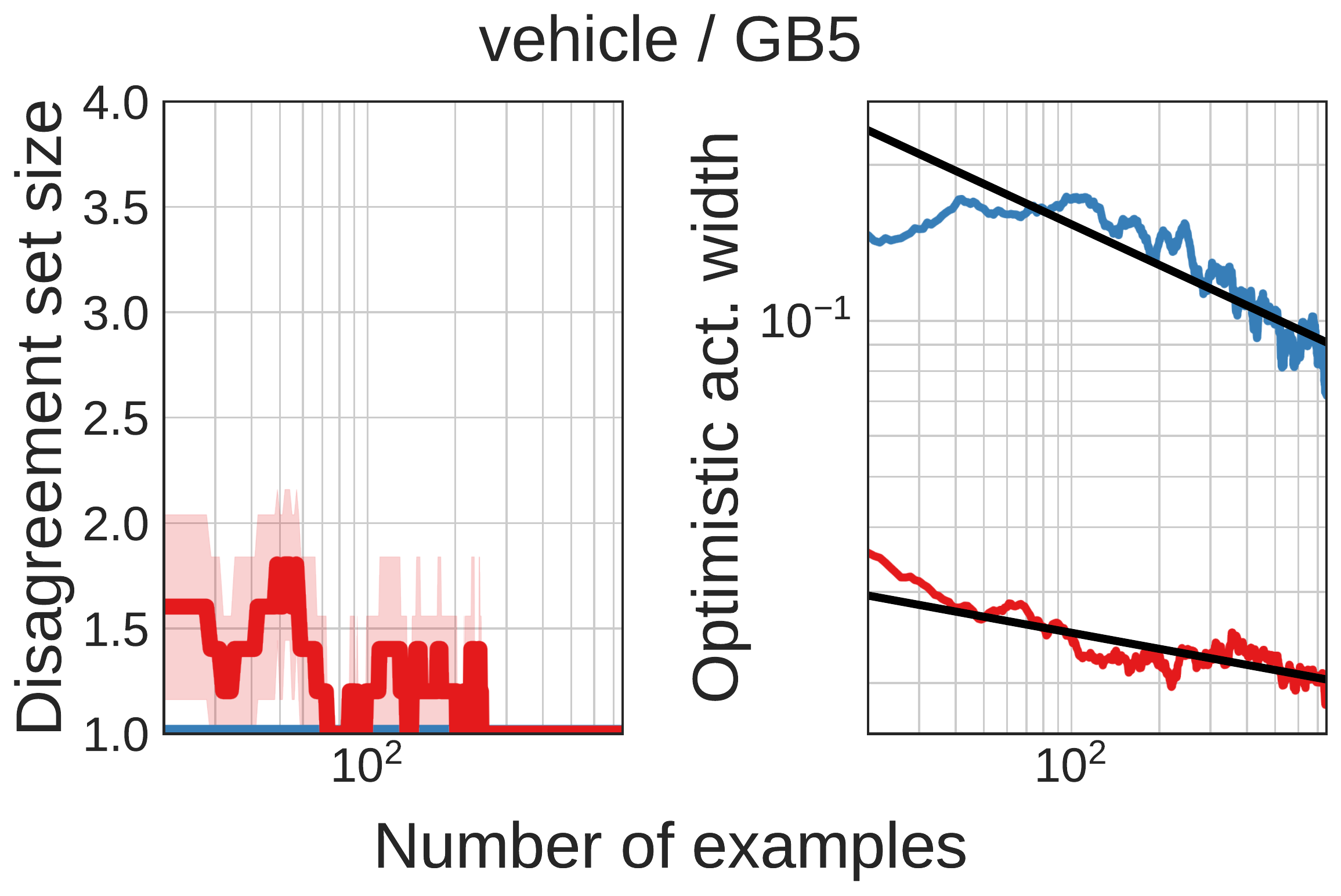}\hfill
\includegraphics[width=0.33\textwidth]{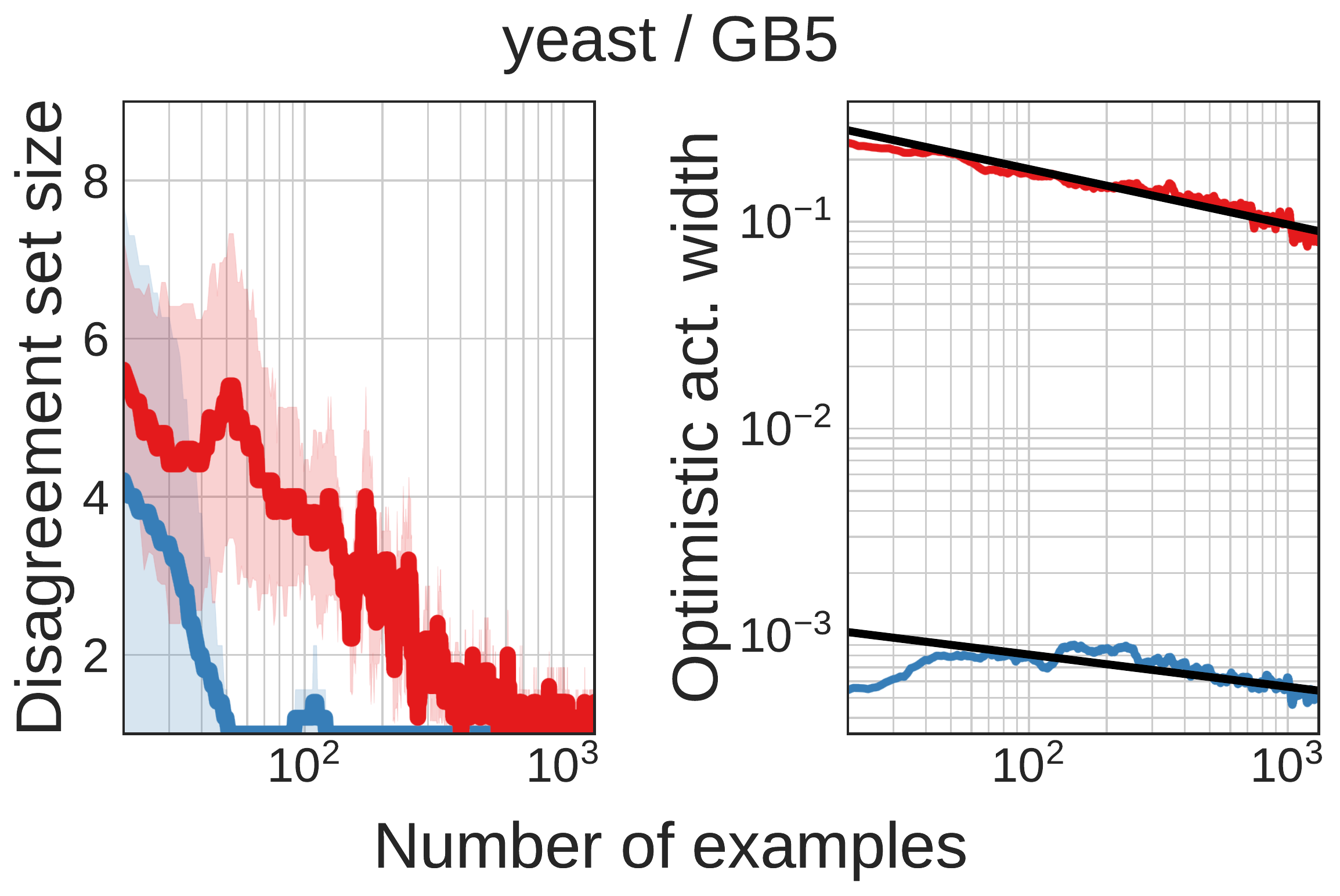}\hfill
\includegraphics[width=0.33\textwidth]{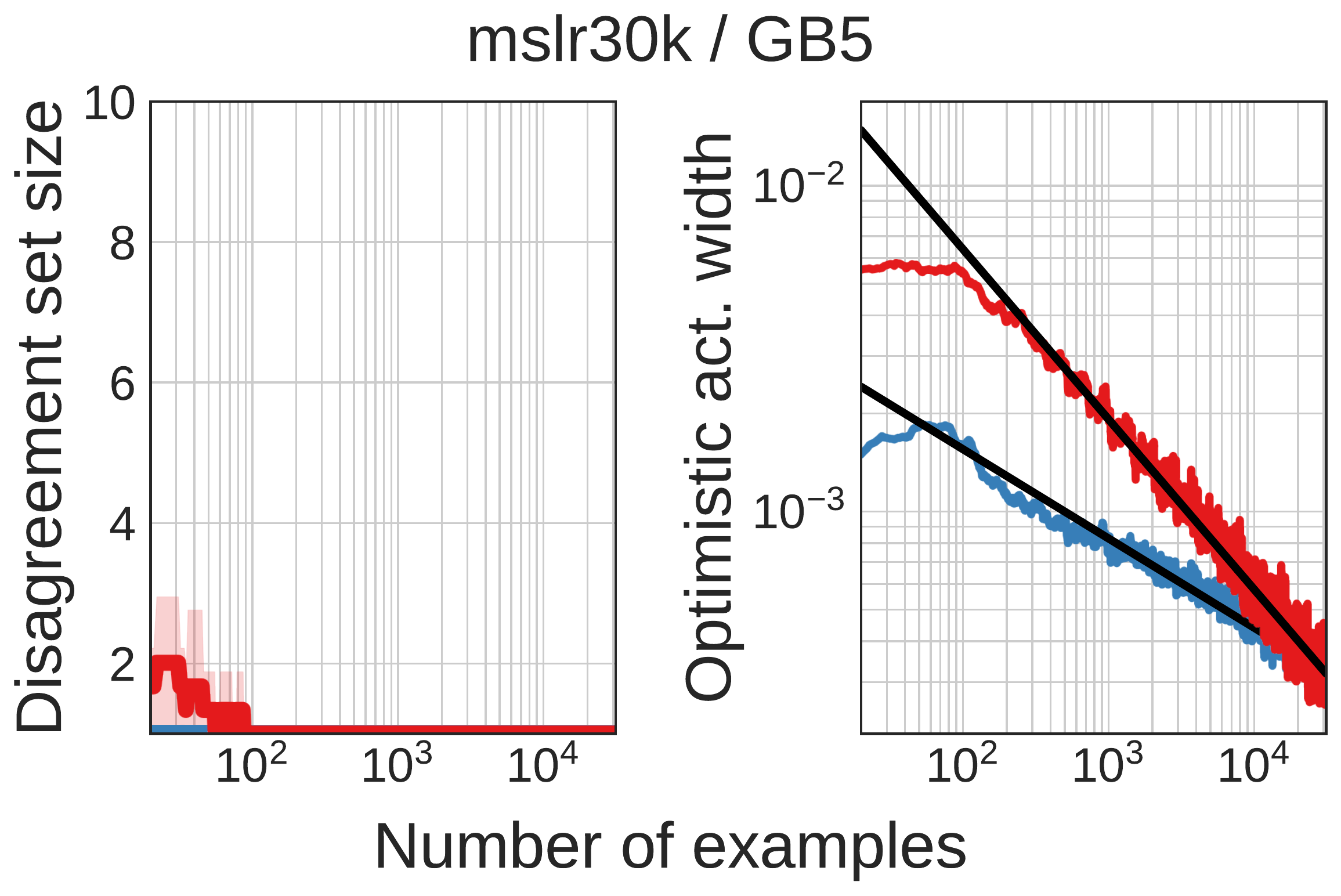}

\includegraphics[width=0.33\textwidth]{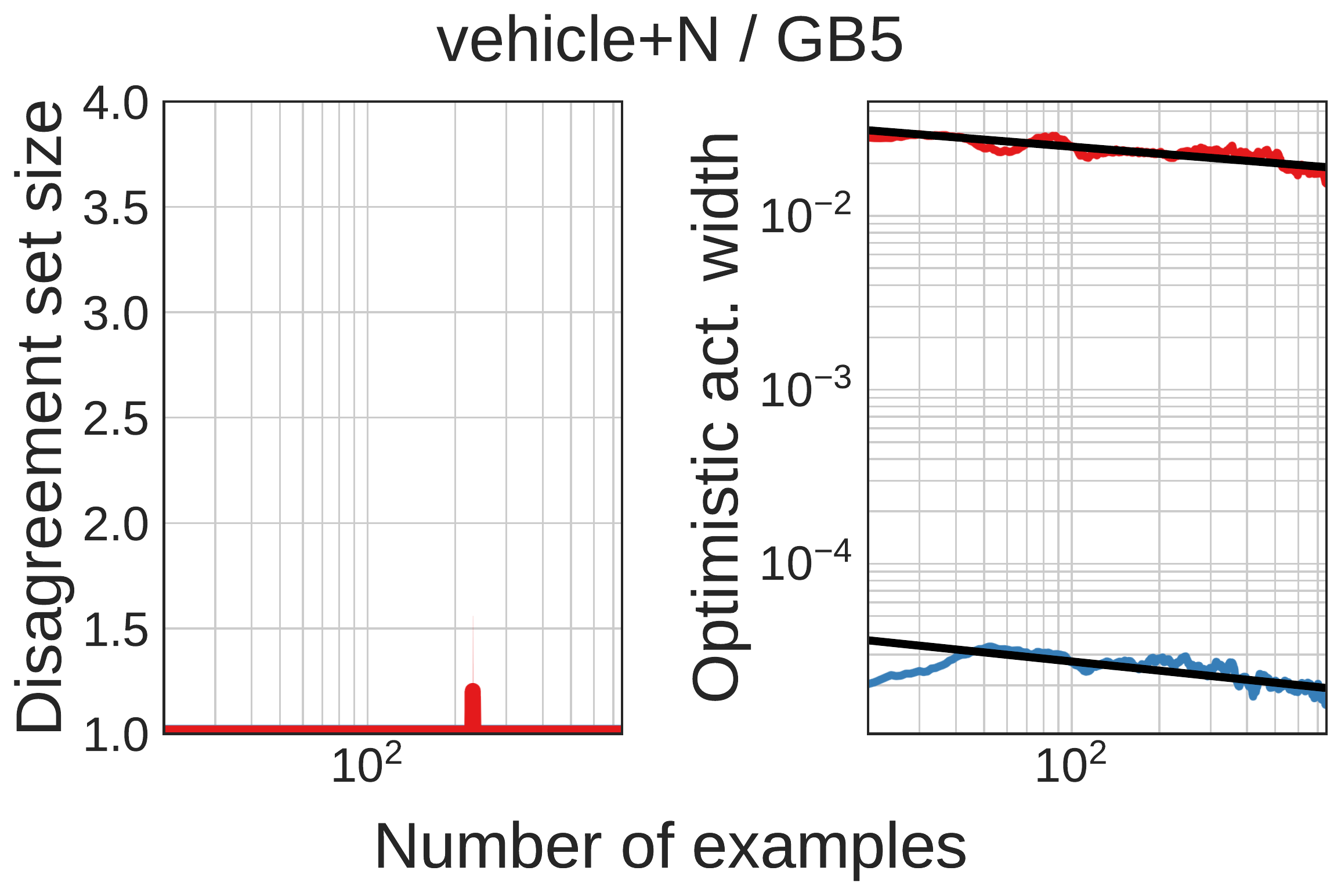}\hfill
\includegraphics[width=0.33\textwidth]{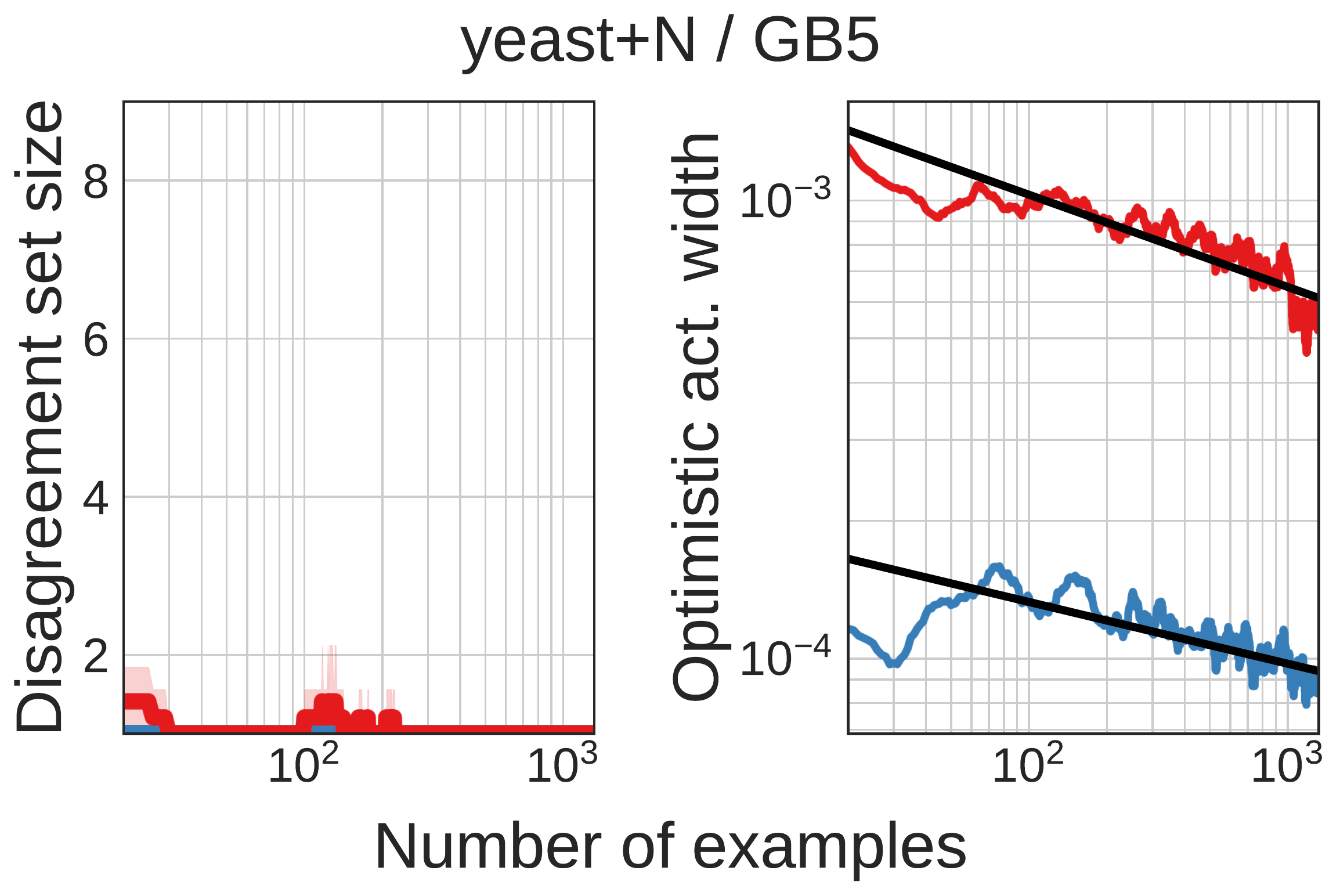}\hfill
\includegraphics[width=0.33\textwidth]{paper_plots/width/plot_disagreement_yahoo.pdf}

\end{centering}
\caption{Size of disagreement set and confidence width.}
\end{figure*}

\newpage
\subsection{UCI Reward Matrices}
\label{sec:reward}

{\footnotesize
\begin{table}[ht!]\begin{center}\caption{Mean reward matrix: yeast}\scalebox{0.7}{\begin{tabular}{ | c | c | c | c | c | c | c | c | c | }
 \hline       1.00 &       0.02  &       0.26  &       0.89  &       0.20  &       0.31  &       0.97  &       0.34  &       0.39 \\ 
\hline       0.25 &       1.00  &       0.65  &       0.29  &       0.03  &       0.52  &       0.30  &       0.10  &       0.09 \\ 
\hline       0.49 &       0.92  &       1.00  &       0.61  &       0.33  &       0.84  &       0.22  &       0.37  &       0.62 \\ 
\hline       0.21 &       0.67  &       0.13  &       1.00  &       0.27  &       0.69  &       0.39  &       0.97  &       0.15 \\ 
\hline       0.26 &       0.63  &       0.11  &       0.39  &       1.00  &       0.25  &       0.34  &       0.17  &       0.69 \\ 
\hline       0.78 &       0.11  &       0.22  &       0.16  &       0.22  &       1.00  &       0.68  &       0.39  &       0.55 \\ 
\hline       0.40 &       0.71  &       0.40  &       0.45  &       0.23  &       0.48  &       1.00  &       0.99  &       0.43 \\ 
\hline       0.41 &       0.48  &       0.85  &       0.87  &       0.22  &       0.46  &       0.33  &       1.00  &       0.96 \\ 
\hline       0.07 &       0.76  &       0.49  &       0.74  &       0.44  &       0.85  &       0.09  &       0.23  &       1.00 \\ 
\hline \end{tabular}}\end{center}\end{table}\begin{table}[ht!]\begin{center}\caption{Mean reward matrix: letter}\scalebox{0.7}{\begin{tabular}{ | c | c | c | c | c | c | c | c | c | c | c | c | c | c | c | c | c | c | c | c | c | c | c | c | c | c | }
 \hline       1.00 &       0.33  &       0.82  &       0.04  &       0.11  &       0.60  &       0.53  &       0.42  &       0.34  &       0.62  &       0.44  &       0.74  &       0.52  &       0.58  &       0.65  &       0.99  &       0.82  &       0.41  &       0.88  &       0.82  &       0.05  &       0.72  &       0.80  &       0.74  &       0.71  &       0.54 \\ 
\hline       0.12 &       1.00  &       0.40  &       0.22  &       0.72  &       0.99  &       0.26  &       0.67  &       0.60  &       0.72  &       0.94  &       0.35  &       0.25  &       0.40  &       0.75  &       0.72  &       0.41  &       0.99  &       0.45  &       0.37  &       0.71  &       0.08  &       0.40  &       0.77  &       0.76  &       0.28 \\ 
\hline       0.19 &       0.47  &       1.00  &       0.73  &       0.19  &       0.33  &       0.84  &       0.62  &       0.89  &       0.98  &       0.84  &       0.18  &       0.62  &       0.48  &       0.40  &       0.74  &       0.83  &       0.68  &       0.14  &       0.70  &       0.06  &       0.19  &       0.92  &       0.41  &       0.15  &       0.68 \\ 
\hline       0.16 &       0.65  &       0.25  &       1.00  &       0.96  &       0.07  &       0.51  &       0.34  &       0.66  &       0.84  &       0.60  &       0.59  &       0.12  &       0.71  &       0.20  &       0.49  &       0.04  &       0.32  &       0.86  &       0.56  &       0.55  &       0.37  &       0.83  &       0.28  &       0.13  &       0.56 \\ 
\hline       0.27 &       0.78  &       0.18  &       0.78  &       1.00  &       0.04  &       0.56  &       0.67  &       0.94  &       0.79  &       0.75  &       0.50  &       0.04  &       0.82  &       0.01  &       0.55  &       0.57  &       0.11  &       0.06  &       0.57  &       0.49  &       0.30  &       0.04  &       0.63  &       0.12  &       0.01 \\ 
\hline       0.28 &       0.30  &       0.18  &       0.07  &       0.78  &       1.00  &       0.25  &       0.52  &       0.25  &       0.85  &       0.48  &       0.62  &       0.97  &       0.35  &       0.22  &       0.98  &       0.59  &       0.98  &       0.97  &       0.71  &       0.02  &       0.61  &       0.25  &       0.13  &       0.37  &       0.20 \\ 
\hline       0.77 &       0.93  &       0.03  &       0.26  &       0.27  &       0.14  &       1.00  &       0.25  &       0.36  &       0.05  &       0.24  &       0.88  &       0.96  &       0.66  &       0.30  &       0.06  &       0.86  &       0.16  &       0.27  &       0.55  &       0.25  &       0.84  &       0.50  &       0.48  &       0.91  &       0.92 \\ 
\hline       0.24 &       0.02  &       0.67  &       0.27  &       0.01  &       0.10  &       0.42  &       1.00  &       0.21  &       0.75  &       0.46  &       0.11  &       0.22  &       0.93  &       0.01  &       0.64  &       0.64  &       0.68  &       0.58  &       0.78  &       0.82  &       0.65  &       0.18  &       0.73  &       0.28  &       0.84 \\ 
\hline       0.57 &       0.09  &       0.91  &       0.46  &       0.94  &       0.04  &       0.11  &       0.76  &       1.00  &       0.45  &       0.82  &       0.42  &       0.19  &       0.84  &       0.11  &       0.29  &       0.22  &       0.46  &       0.32  &       0.91  &       0.79  &       0.71  &       0.14  &       0.61  &       0.85  &       0.92 \\ 
\hline       0.66 &       0.26  &       0.28  &       0.64  &       0.72  &       0.31  &       0.68  &       0.51  &       0.83  &       1.00  &       0.91  &       0.12  &       0.84  &       0.95  &       0.57  &       0.00  &       0.03  &       0.41  &       0.46  &       0.48  &       0.68  &       0.75  &       0.82  &       0.35  &       0.61  &       0.39 \\ 
\hline       0.73 &       0.56  &       0.59  &       0.39  &       0.63  &       0.87  &       0.65  &       0.13  &       0.09  &       0.68  &       1.00  &       0.31  &       0.89  &       0.86  &       0.81  &       0.36  &       0.64  &       0.60  &       0.24  &       0.59  &       1.00  &       0.05  &       0.24  &       0.33  &       0.80  &       0.44 \\ 
\hline       0.06 &       0.32  &       0.83  &       0.74  &       0.28  &       0.73  &       0.32  &       0.15  &       0.98  &       0.26  &       0.61  &       1.00  &       0.64  &       0.43  &       0.40  &       0.05  &       0.08  &       0.45  &       0.92  &       0.23  &       0.87  &       0.81  &       0.17  &       0.31  &       0.43  &       0.86 \\ 
\hline       0.63 &       0.82  &       0.50  &       0.58  &       0.45  &       0.26  &       0.62  &       0.58  &       0.87  &       0.92  &       0.57  &       0.69  &       1.00  &       0.68  &       1.00  &       0.94  &       0.14  &       0.94  &       0.04  &       0.03  &       0.18  &       0.31  &       0.98  &       0.94  &       0.76  &       0.62 \\ 
\hline       0.97 &       0.57  &       0.21  &       0.13  &       0.76  &       0.53  &       0.82  &       0.79  &       0.67  &       0.78  &       0.69  &       0.43  &       0.83  &       1.00  &       0.78  &       0.09  &       0.95  &       0.48  &       0.89  &       0.08  &       0.94  &       0.31  &       0.42  &       0.69  &       0.09  &       0.21 \\ 
\hline       0.58 &       0.39  &       0.11  &       0.01  &       0.90  &       0.67  &       0.32  &       0.89  &       0.97  &       0.08  &       0.26  &       0.53  &       0.92  &       0.23  &       1.00  &       0.90  &       0.34  &       0.23  &       0.18  &       0.05  &       0.96  &       0.15  &       0.96  &       0.34  &       0.06  &       0.82 \\ 
\hline       0.80 &       0.46  &       0.77  &       0.75  &       0.45  &       0.28  &       0.14  &       0.91  &       0.08  &       0.73  &       0.08  &       0.67  &       0.06  &       0.11  &       0.48  &       1.00  &       0.03  &       0.64  &       0.90  &       0.48  &       0.84  &       0.71  &       0.93  &       0.97  &       0.59  &       0.95 \\ 
\hline       0.71 &       0.46  &       0.92  &       0.58  &       0.24  &       0.39  &       0.42  &       0.16  &       0.02  &       0.05  &       0.68  &       0.25  &       0.15  &       0.20  &       0.82  &       0.89  &       1.00  &       0.74  &       0.58  &       0.49  &       0.64  &       0.95  &       0.80  &       0.41  &       0.25  &       0.00 \\ 
\hline       0.29 &       0.98  &       0.42  &       0.54  &       0.06  &       0.14  &       0.99  &       0.54  &       0.22  &       0.64  &       0.73  &       0.50  &       0.33  &       0.72  &       0.13  &       0.72  &       0.45  &       1.00  &       0.63  &       0.86  &       0.32  &       0.70  &       0.12  &       0.44  &       0.72  &       0.89 \\ 
\hline       0.56 &       0.63  &       0.53  &       0.35  &       0.85  &       0.57  &       0.26  &       0.80  &       0.83  &       0.45  &       0.68  &       0.09  &       0.72  &       0.34  &       0.02  &       0.71  &       0.55  &       0.83  &       1.00  &       0.99  &       0.33  &       0.13  &       0.04  &       0.32  &       0.21  &       0.57 \\ 
\hline       0.96 &       0.22  &       0.33  &       0.27  &       0.27  &       0.69  &       0.89  &       0.58  &       0.40  &       0.43  &       0.55  &       0.31  &       0.26  &       0.91  &       0.51  &       0.12  &       0.57  &       0.25  &       0.01  &       1.00  &       0.36  &       0.68  &       0.61  &       0.17  &       0.30  &       0.72 \\ 
\hline       0.43 &       0.13  &       0.17  &       0.73  &       0.62  &       0.56  &       0.06  &       0.39  &       0.45  &       0.58  &       0.70  &       0.72  &       0.59  &       0.27  &       0.41  &       0.78  &       0.47  &       0.40  &       0.85  &       1.00  &       1.00  &       0.63  &       0.91  &       0.15  &       0.29  &       0.65 \\ 
\hline       0.18 &       0.28  &       0.94  &       0.31  &       0.10  &       0.50  &       0.08  &       0.25  &       0.96  &       0.84  &       0.15  &       0.25  &       0.05  &       0.20  &       0.81  &       0.91  &       0.62  &       0.09  &       0.50  &       0.67  &       0.11  &       1.00  &       0.76  &       0.39  &       0.83  &       0.17 \\ 
\hline       0.26 &       0.80  &       0.68  &       0.78  &       0.18  &       0.95  &       0.18  &       0.70  &       0.31  &       0.51  &       0.91  &       0.78  &       0.75  &       0.11  &       0.91  &       0.90  &       0.98  &       0.11  &       0.38  &       0.27  &       0.85  &       0.90  &       1.00  &       0.22  &       0.05  &       0.88 \\ 
\hline       0.95 &       0.75  &       0.82  &       0.31  &       0.13  &       0.10  &       0.67  &       0.14  &       0.92  &       0.24  &       0.75  &       0.61  &       0.34  &       0.63  &       0.02  &       0.76  &       0.17  &       0.61  &       0.12  &       0.57  &       0.73  &       0.80  &       0.14  &       1.00  &       0.41  &       0.40 \\ 
\hline       0.83 &       0.19  &       0.76  &       0.74  &       0.42  &       0.14  &       0.70  &       0.88  &       0.18  &       0.12  &       0.21  &       0.44  &       0.46  &       0.76  &       0.16  &       0.90  &       0.52  &       0.28  &       0.02  &       0.59  &       0.20  &       0.44  &       0.96  &       0.20  &       1.00  &       0.84 \\ 
\hline       0.03 &       0.67  &       0.47  &       0.34  &       0.50  &       0.43  &       0.56  &       0.11  &       0.36  &       0.93  &       0.50  &       0.64  &       0.47  &       0.97  &       0.12  &       0.35  &       0.68  &       0.79  &       0.40  &       0.74  &       0.37  &       0.10  &       0.02  &       0.14  &       0.99  &       1.00 \\ 
\hline \end{tabular}}\end{center}\end{table}\begin{table}[ht!]\begin{center}\caption{Mean reward matrix: optdigits}\scalebox{0.7}{\begin{tabular}{ | c | c | c | c | c | c | c | c | c | c | }
 \hline       1.00 &       0.56  &       0.12  &       0.40  &       0.78  &       0.51  &       0.18  &       0.85  &       0.96  &       0.98 \\ 
\hline       0.19 &       1.00  &       0.23  &       0.03  &       0.95  &       0.92  &       0.29  &       0.17  &       0.40  &       0.51 \\ 
\hline       0.31 &       0.43  &       1.00  &       0.56  &       0.83  &       1.00  &       0.33  &       0.09  &       0.77  &       0.15 \\ 
\hline       0.73 &       0.96  &       0.07  &       1.00  &       0.84  &       0.15  &       0.77  &       0.78  &       0.68  &       0.13 \\ 
\hline       0.04 &       0.66  &       0.25  &       0.99  &       1.00  &       0.06  &       0.70  &       0.63  &       0.90  &       0.16 \\ 
\hline       0.61 &       0.32  &       0.76  &       0.16  &       0.93  &       1.00  &       0.83  &       0.23  &       0.11  &       0.67 \\ 
\hline       0.58 &       0.88  &       1.00  &       0.28  &       0.74  &       0.28  &       1.00  &       0.49  &       0.87  &       0.16 \\ 
\hline       0.97 &       0.05  &       0.70  &       0.65  &       0.05  &       0.20  &       0.33  &       1.00  &       0.37  &       0.53 \\ 
\hline       0.35 &       0.51  &       0.26  &       0.85  &       0.62  &       0.30  &       0.78  &       0.90  &       1.00  &       0.86 \\ 
\hline       0.82 &       0.87  &       0.38  &       0.61  &       0.42  &       0.24  &       0.06  &       0.82  &       0.38  &       1.00 \\ 
\hline \end{tabular}}\end{center}\end{table}\begin{table}[ht!]\begin{center}\caption{Mean reward matrix: page-blocks}\scalebox{0.7}{\begin{tabular}{ | c | c | c | c | c | }
 \hline       1.00 &       0.38  &       0.66  &       0.16  &       0.96 \\ 
\hline       0.35 &       1.00  &       0.24  &       0.59  &       0.41 \\ 
\hline       0.14 &       0.54  &       1.00  &       0.77  &       0.93 \\ 
\hline       0.09 &       0.20  &       0.99  &       1.00  &       0.24 \\ 
\hline       0.63 &       0.73  &       0.69  &       0.03  &       1.00 \\ 
\hline \end{tabular}}\end{center}\end{table}\begin{table}[ht!]\begin{center}\caption{Mean reward matrix: pendigits}\scalebox{0.7}{\begin{tabular}{ | c | c | c | c | c | c | c | c | c | c | }
 \hline       1.00 &       0.37  &       0.56  &       0.96  &       0.74  &       0.82  &       0.10  &       0.93  &       0.61  &       0.60 \\ 
\hline       0.09 &       1.00  &       0.66  &       0.44  &       0.55  &       0.70  &       0.59  &       0.05  &       0.56  &       0.77 \\ 
\hline       0.91 &       0.09  &       1.00  &       0.46  &       0.45  &       1.00  &       0.16  &       0.71  &       0.16  &       0.81 \\ 
\hline       0.04 &       0.53  &       0.17  &       1.00  &       0.05  &       0.24  &       0.67  &       0.78  &       0.70  &       0.33 \\ 
\hline       0.49 &       0.52  &       0.30  &       0.46  &       1.00  &       0.50  &       0.40  &       0.73  &       0.86  &       0.03 \\ 
\hline       0.29 &       0.79  &       0.46  &       0.01  &       0.42  &       1.00  &       0.60  &       0.32  &       0.98  &       0.59 \\ 
\hline       0.13 &       0.52  &       0.36  &       0.01  &       0.10  &       0.78  &       1.00  &       0.20  &       0.62  &       0.64 \\ 
\hline       0.27 &       0.13  &       0.47  &       0.39  &       0.41  &       0.38  &       0.29  &       1.00  &       0.43  &       0.78 \\ 
\hline       0.70 &       0.78  &       0.29  &       0.21  &       0.50  &       0.13  &       0.17  &       0.25  &       1.00  &       0.23 \\ 
\hline       0.63 &       0.63  &       0.53  &       0.74  &       0.82  &       0.37  &       0.80  &       0.88  &       0.59  &       1.00 \\ 
\hline \end{tabular}}\end{center}\end{table}\begin{table}[ht!]\begin{center}\caption{Mean reward matrix: satimage}\scalebox{0.7}{\begin{tabular}{ | c | c | c | c | c | c | }
 \hline       1.00 &       0.06  &       0.12  &       0.79  &       0.98  &       0.27 \\ 
\hline       0.87 &       1.00  &       0.64  &       0.78  &       0.63  &       0.13 \\ 
\hline       1.00 &       0.63  &       1.00  &       0.62  &       0.34  &       0.76 \\ 
\hline       0.11 &       0.52  &       0.63  &       1.00  &       0.11  &       0.29 \\ 
\hline       0.07 &       0.67  &       0.23  &       0.52  &       1.00  &       0.45 \\ 
\hline       0.73 &       0.97  &       0.20  &       0.72  &       0.79  &       1.00 \\ 
\hline \end{tabular}}\end{center}\end{table}\begin{table}[ht!]\begin{center}\caption{Mean reward matrix: vehicle}\scalebox{0.7}{\begin{tabular}{ | c | c | c | c | }
 \hline       1.00 &       0.36  &       0.18  &       0.52 \\ 
\hline       0.01 &       1.00  &       0.80  &       0.76 \\ 
\hline       0.67 &       0.03  &       1.00  &       0.40 \\ 
\hline       0.19 &       0.77  &       0.62  &       1.00 \\ 
\hline \end{tabular}}\end{center}\end{table}\begin{table}[ht!]\begin{center}\caption{Mean reward matrix: adult}\scalebox{0.7}{\begin{tabular}{ | c | c | }
 \hline       1.00 &       0.61 \\ 
\hline       0.66 &       1.00 \\ 
\hline \end{tabular}}\end{center}\end{table}
}

\end{document}